\documentclass{article}

 \usepackage[main, final]{neurips_2026}

\usepackage[utf8]{inputenc} 
\usepackage[T1]{fontenc}    
\usepackage{url}            
\usepackage{booktabs}       
\usepackage{amsfonts, amsmath, amssymb, amsthm, mathtools}       
\usepackage{enumitem}
\usepackage{graphicx} 
\usepackage{subcaption}
\usepackage{nicefrac}       
\usepackage{microtype}      
\usepackage{xcolor}         
\usepackage{multirow}
\usepackage{float}
\usepackage[textsize=tiny]{todonotes}
\usepackage[normalem]{ulem}

\usepackage{xcolor}

\definecolor{LinkBlue}{HTML}{1A4E8A}   
\definecolor{CiteGreen}{HTML}{2A6F55}  
\definecolor{UrlTeal}{HTML}{0B6F7A}    

\usepackage[
    colorlinks=true,
    linkcolor=LinkBlue,
    citecolor=CiteGreen,
    urlcolor=UrlTeal,
    filecolor=LinkBlue,
    breaklinks=true
]{hyperref}
\usepackage[capitalize,noabbrev]{cleveref}

\usepackage[page,header]{appendix}
\usepackage{titletoc}

\usepackage{tcolorbox}
\tcbuselibrary{listings,breakable}
\newtcblisting{promptbox}{
  breakable,
  colback=gray!5,
  colframe=gray!80,
  boxrule=0.5pt,
  arc=4pt,
  outer arc=4pt,
  left=6pt,right=6pt,top=6pt,bottom=6pt,
  listing only,
  listing options={style=promptstyle},
}
\usepackage{listings}
\lstdefinestyle{promptstyle}{
  basicstyle=\ttfamily\scriptsize,
  columns=fullflexible,
  breaklines=true,
  breakatwhitespace=true,
  breakautoindent=false,
  breakindent=0pt,
  showstringspaces=false,
  keepspaces=true,
  upquote=true,
  mathescape=false,
  literate={_}{{\_}}1
}

\theoremstyle{plain}
\newtheorem{theorem}{Theorem}
\newtheorem{proposition}{Proposition}[section]
\newtheorem{lemma}[proposition]{Lemma}
\newtheorem{corollary}[proposition]{Corollary}
\newtheorem{remark}[proposition]{Remark}
\theoremstyle{definition}
\newtheorem{definition}[proposition]{Definition}
\newtheorem{assumption}[proposition]{Assumption}

\newcommand{\cX}{\mathcal{X}}
\newcommand{\cY}{\mathcal{Y}}
\newcommand{\cT}{\mathcal{T}}
\newcommand{\cE}{\mathcal{E}}
\newcommand{\cP}{\mathcal{P}}
\newcommand{\PP}{\mathbb{P}}
\newcommand{\EE}{\mathbb{E}}
\newcommand{\Ind}{\mathbf{1}}
\newcommand{\bT}{\mathbf{T}}
\newcommand{\bN}{\mathbf{N}}
\newcommand{\hatl}{\hat{\lambda}}
\newcommand{\abstain}{\bot}
\newcommand{\RR}{\mathbb{R}}
\newcommand{\eps}{\epsilon}

\title{Pause and Reflect: Conformal Aggregation for Chain-of-Thought Reasoning}

\author{%
  \textbf{Yu Gu}\textsuperscript{*}\\
  \normalfont Department of Mathematics and Statistics\\
  \normalfont McGill University\\
  \normalfont Montr\'eal, QC, Canada\\
  \texttt{yu.gu4@mail.mcgill.ca}
  \And
  \textbf{Zijun Yu}\textsuperscript{*}\\
  \normalfont Department of Mathematics and Statistics\\
  \normalfont McGill University\\
  \normalfont Montr\'eal, QC, Canada\\
  \texttt{zijun.yu@mail.mcgill.ca}
  \AND
  \textbf{Vahid Partovi Nia}\\
  \normalfont Department of Mathematics and\\
  \normalfont Industrial Engineering\\
  \normalfont Polytechnique de Montr\'eal\\
  \normalfont Montr\'eal, QC, Canada\\
  \texttt{vahid.partovinia@polymtl.ca}
  \And
  \textbf{Masoud Asgharian}\\
  \normalfont Department of Mathematics and Statistics\\
  \normalfont McGill University\\
  \normalfont Montr\'eal, QC, Canada\\
  \texttt{masoud.asgharian2@mcgill.ca}
}

\begin{document}

\maketitle
\begingroup
\renewcommand{\thefootnote}{\fnsymbol{footnote}}
\footnotetext[1]{Equal contribution.}
\endgroup

\begin{abstract}

Chain-of-thought (CoT) reasoning with self-consistency improves performance by aggregating multiple sampled reasoning paths.
In this setting, correctness is no longer tied to a single reasoning trace but to the aggregation rule over a pool of candidate paths, making aggregation uncertainty the central challenge.
This issue is critical where confidently incorrect answers are far more costly than abstentions.
We introduce a conformal procedure for CoT reasoning that directly addresses aggregation uncertainty. Our approach replaces majority voting with weighted score aggregation over reasoning paths and calibrates an abstention rule using conformal risk control. This approach leads to finite-sample guarantees on the confident-error rate--the  probability that the system answers and is wrong. We further identify score separability as the key condition under which abstention provably improves selective accuracy, and derive closed-form expressions that predict accuracy gains from calibration data alone. The method is fully inference-time, and requires no retraining.
Across four benchmarks, four open-source models, and three score classes, realized confident-error rates are consistent with the prescribed targets up to calibration-split and test-set variability.
Our method achieves $90.1\%$ selective accuracy on GSM8K by abstaining on fewer than $5\%$ of problems, compared with $82.3\%$ accuracy under majority voting and $87.4\%$ under score-weighted voting without abstention, while conformal risk control keeps the expected rate of confidently wrong answers below $10\%$.

\end{abstract}



\section{Introduction}

Chain-of-thought (CoT) prompting \citep{wei2022chain} has become a standard way to improve reasoning in large language models. In practice, CoT-based reasoning takes several forms at inference time. Some methods improve a single reasoning trajectory by branching or revising intermediate thoughts, as in Tree-of-Thoughts \citep{yao2023tree}; others, such as self-consistency CoT, sample multiple complete reasoning traces and aggregate their final answers \citep{wang2023selfconsistency}. Recent developments include confidence-weighted self-consistency \citep{taubenfeld2025confidence} which shows that scoring paths further improves aggregation. All of these methods show that generating multiple reasoning paths improves final performance on difficult reasoning problems, especially when the task benefits from exploration, search, or diverse valid solution paths.
It is therefore desirable to know, given a pool of sampled paths, whether the aggregated answer is reliable. Conformal prediction (CP) is an uncertainty quantification framework that provides finite-sample reliability guarantees. CP methods developed for large language models do not, however, address aggregation-level uncertainty. Prior work focuses on either final output tokens or single reasoning traces. \citet{rubin-toles2025conformal} advance reasoning-aware conformal methods by studying \emph{intra}-trace uncertainty: calibrating the internal claim-dependency structure of a single generated reasoning path to ensure factuality. \citet{yadkori2024mitigating} develop conformal abstention for LLMs, but applies it at the \emph{final-answer level}, without considering reasoning paths. Neither approach targets the aggregation uncertainty that arises when pooling multiple complete CoT paths with quality scores.

We study precisely this setting: the uncertainty in self-consistency reasoning  where the system must decide from a pool of complete CoT paths. The  self-consistency paradigm transforms reasoning into an aggregation problem. Once multiple complete paths are sampled, ``correctness'' is no longer a property of any single path alone, but rather of the rule used to combine them. The main failure mode is therefore not merely an incoherent individual path or an unreliable final response, but an unreliable consensus. Several wrong paths may dominate the path pool even when correct paths are present (Figure~\ref{fig:framework_comparison}). As Mark Twain observed, ``whenever you find yourself on the side
of the majority, it is time to {\bf pause and reflect}''\footnote{Mark Twain (Notebook, 1904).}, highlighting the importance of questioning consensus --- and this is
precisely the failure mode we target: the majority vote can be wrong even when correct alternatives exist.

The above observation raises two natural questions. First, \textbf{can we do better than a simple majority vote?} If each reasoning path can be assigned a quality score, from the sequence perplexity to classifier models, then majority voting should be generalized to \emph{score-weighted voting}, where high-quality paths exert greater influence on the final answer. The resulting weighted vote share of the winning answer then provides an inherent \emph{confidence measure} for the aggregated prediction, quantifying how strongly the paths support a single answer after accounting for path quality. This then leads to the second question: \textbf{given this confidence measure, when is it safe to answer and when should the system abstain?} We address this second question using the conformal risk control (CRC) \citep{angelopoulos2024conformal} by tuning a threshold on the aggregation confidence that controls the \emph{confident-error rate} with finite-sample theoretical guarantees. CRC is central to our framework, which adapts conformal prediction to the chain-of-thought setting.
\begin{figure*}[t]
    \centering
    \includegraphics[width=0.95\textwidth]{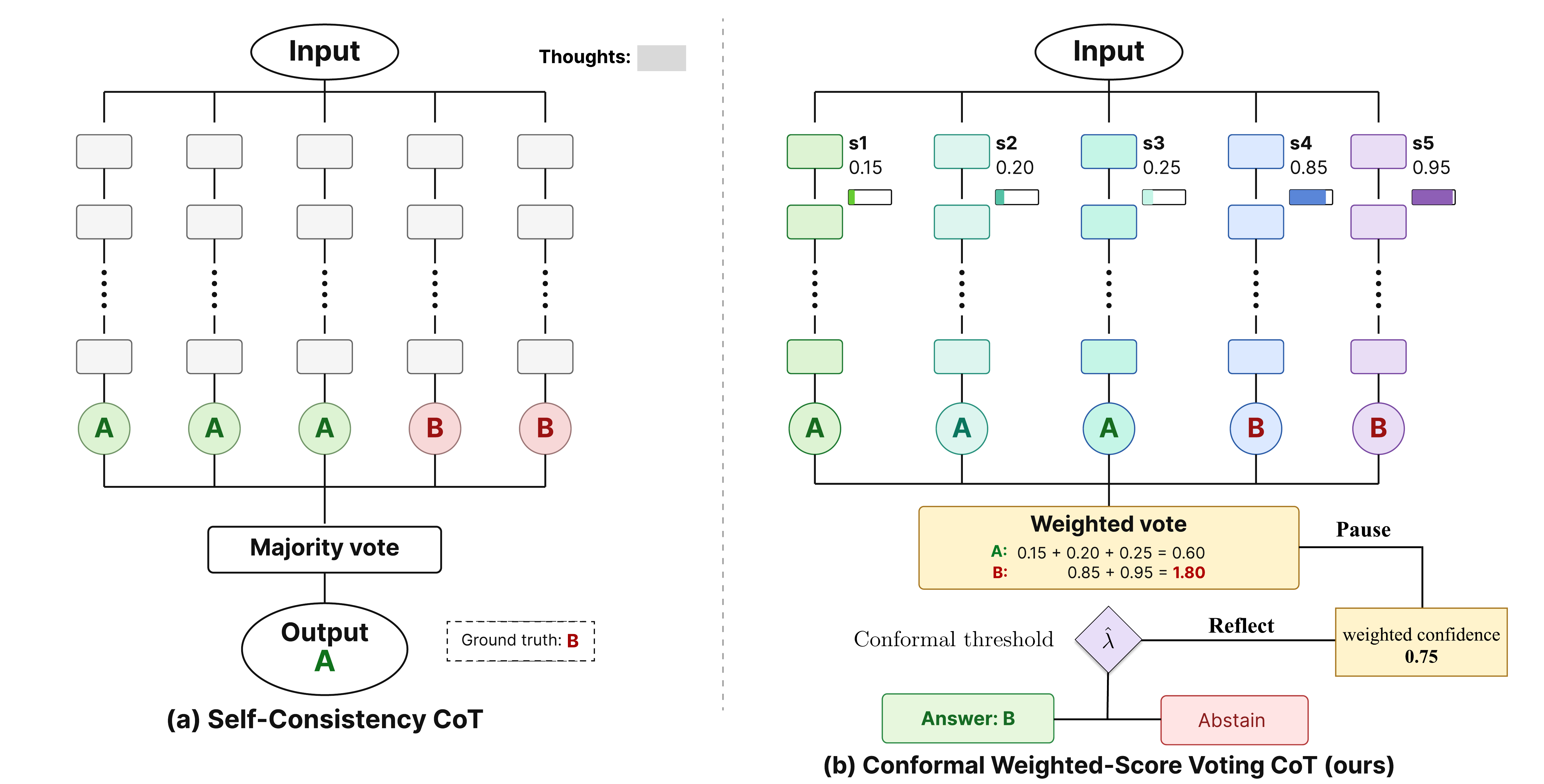}
    \caption{\textbf{Pause and Reflect.} \emph{Left:} The system commits to the majority answer even when it is wrong and right answer presents. \emph{Right (Ours):} Rather than committing unconditionally, the system \emph{pauses}: it aggregates paths by weighted voting to form a confidence measure, then \emph{reflects} on whether that confidence clears a calibrated threshold~$\hat\lambda$. If it does, the system answers; if not, it abstains.}
    \vspace{-1.5em}
    \label{fig:framework_comparison}
\end{figure*}
Our contributions are as follows:
\begin{itemize}
    \item \textbf{A conformal framework for multi-path reasoning}  (\Cref{sec:method}): We propose a purely inference-time framework that generalizes majority voting to score-weighted aggregation over sampled CoT paths and calibrates an abstention threshold, providing finite-sample control of the confident-error rate (Theorem 2). The framework is agnostic to the choice of path-quality score, requires no parameter updates, and applies to any pretrained model capable of generating multiple reasoning paths.
    \item \textbf{Score separability and deployment-time guarantees} (\Cref{sec:separability}--\ref{sec:calibration}): We introduce score separability as a necessary and sufficient condition for abstention to improve selective accuracy, and derive a closed-form expression for accuracy gain (\Cref{thm:ca-gain}), with a converse establishing that non-separable scores cannot benefit from abstention at any threshold (\Cref{prop:nonsep}). 
    \item \textbf{Empirical validation on mathematical and question-answering benchmarks} (\Cref{sec:experiments}): Across four benchmarks and four open-source models, we confirm that the abstention threshold delivers the guaranteed confident-error rate control, that the closed-form predictor closely tracks observed accuracy, and that score choice is the dominant design decision, consistent with the theory.
\end{itemize}

\section{Related Work}

\paragraph{Conformal methods for language model outputs.}
Conformal prediction has been applied to language models primarily at the level of \emph{final outputs}. In closed-form or short-answer settings, conformal prediction constructs prediction sets over final answers~\citep{quach2024conformal}. \citet{vishwakarma2025prune} apply conformal prediction to multiple-choice settings by pruning candidate answer choices. In open-ended generation, related methods quantify uncertainty over the final response itself: filtering unsupported subclaims~\citep{mohri2024language, cherian2024large}. More recent work moves beyond final outputs to treat reasoning as the object of conformal control. \citet{rubin-toles2025conformal} introduce \emph{coherent factuality} and apply split conformal prediction to the internal dependency structure of a single generated reasoning trace, ensuring that retained claims form a logically substantiated chain; their target is intra-trace coherence, not multi-path consensus under aggregation. \citet{yadkori2024mitigating} apply CRC at the level of the final response token distribution, without access to intermediate reasoning paths.

Our setting is categorically different from all of the above. Under self-consistency, the system samples a \emph{pool} of complete reasoning paths and aggregates their answers. The central uncertainty is not the correctness of a single response or the coherence of a single trace, but whether the \emph{consensus induced by the sampled path pool is reliable}. Neither response-level nor intra-trace conformal methods address this failure mode. We instead calibrate an abstention threshold on the aggregated vote confidence to control the confident-error rate with finite-sample guarantees, operating on the pool as a whole rather than on any individual output or trace. To our knowledge, this is the first conformal framework targeting \emph{aggregation uncertainty} in multi-path reasoning.

\vspace{-1em}
\paragraph{Self-consistency, repeated sampling, and inference-time selection.}
Chain-of-thought prompting and zero-shot reasoning established reasoning traces as a useful inference-time scaffold \citep{wei2022chain, kojima2022large}. Self-consistency showed that sampling multiple reasoning paths and aggregating their answers substantially improves accuracy \citep{wang2023selfconsistency}, and confidence-weighted self-consistency further demonstrated that scoring paths improves aggregation quality~\citep{taubenfeld2025confidence, du2025crew}. This idea connects to a broader literature on repeated sampling, verifier-based selection, and inference-time scaling, including best-of-$N$  selection, process supervision, and test-time compute scaling \citep{cobbe2021training, lightman2023lets, naik2024diversity, levi2025a, snell2025scaling, brown2025large, kim2026scaling}.
Other methods improve reasoning by exploring or selecting among multiple candidate paths through search or speculative decoding \citep{yao2023tree, leviathan2023fast, shi2025speccot, liao2025reward, wang2025efficient}.
The idea of abstaining when confidence is low has a long history in selective classification \citep{geifman2017selective}, where abstention improves accuracy by rejecting low-confidence predictions; our contribution is to formalize this condition under the reasoning paths aggregation setting and derive closed-form accuracy gains. A key question left open by this line of work is when the aggregated consensus warrants a committed answer, and when the system should instead abstain.
\section{Methodology}
\label{sec:method}

Our CoT aggregation framework is designed around two principles. It should be \emph{safe}: the marginal probability of emitting a confidently wrong answer should be controlled at a user-specified level. It should be \emph{useful}: the system should answer a nontrivial fraction of inputs, and abstention should be informative so that the answered subset has higher accuracy than the unfiltered baseline. We capture these desiderata through two metrics: \emph{selective accuracy}, the accuracy on the answered subset, and the \emph{confident-error rate}, the marginal probability of committing to an incorrect answer, the safety target a reliability layer must control.

The remainder of this section develops the framework as follows. \Cref{sec:setup} sets up the score-weighted aggregation that produces the system's prediction and an associated confidence. \Cref{sec:separability} identifies \emph{score separability} as the condition under which abstention provably improves selective accuracy, and \Cref{sec:frontier} characterizes the resulting accuracy–yield trade-off. \Cref{sec:calibration} calibrates the abstention threshold to deliver the safety guarantee on the confident-error rate. Proofs are deferred to Appendix~\ref{app:proofs}.

\subsection{Setup: Multi-Path Reasoning as Score-Weighted Aggregation}
\label{sec:setup}

We consider a reasoning task with prompt space $\cX$ and answer space $\cY$. At inference time, a pretrained LLM samples $m \geq 1$ independent CoT paths $\bT = (T^{(1)}, \ldots, T^{(m)})$ from a prompt $X \in \cX$, with no modification to the underlying model. An extraction function $a: \cT \to \cY$ matches each free-form path to a candidate answer.

\paragraph{Path-quality score.} We equip the framework with a path-quality score $q: \cX \times \cT \to \mathbb{R}$ that assigns higher values to paths the system considers more reliable. The score is computed from the prompt and the path alone; it does not access the true answer. The \emph{safety} guarantee of our framework is agnostic to this choice. Whether a particular $q$ is also \emph{useful} is a diagnostic property we characterize in \Cref{sec:separability}. Concrete instantiations range from cheap signals (perplexity, self-consistency agreement) to learned signals such as process or outcome reward models and judge LLMs.

\paragraph{Score-weighted vote.} The weight function $w: \mathbb{R} \to \mathbb{R}^+$ converts a path-quality score $q$ into a vote weight. Setting $w \equiv 1$ recovers majority voting; a non-trivial monotone $w$ allows higher-quality paths to exert greater influence. For candidate answer $y \in \cY$, define the \emph{weighted vote}
\[
    V_y(X,\bT) \;\coloneqq  \; \sum_{j=1}^m w[q(X, T^{(j)})]\,\Ind\!\bigl[a(T^{(j)}) = y\bigr],
\]
from which we obtain the aggregated prediction and its corresponding \emph{vote confidence}, with ties broken by an arbitrary fixed rule:
\[
    \hat{y}(X,\bT) \;\coloneqq\; \arg\max_{y \in \cY} V_y(X,\bT),
    \qquad
    \nu(X,\bT) \;\coloneqq\; \frac{V_{\hat{y}}(X,\bT)}{\sum_{y' \in \cY} V_{y'}(X,\bT)} \in [0,1].
\]
This normalized weighted vote share of the winning answer is the central statistic of our framework: every subsequent abstention and calibration decision is made on the basis of $\nu$.
We write $p_v \coloneqq \PP(\hat{y}(X,\bT) = Y)$ for the \emph{vote accuracy}, i.e. the accuracy of unfiltered score-weighted aggregation, and adopt $p_v$ as the natural baseline against which any abstention policy is evaluated.

\paragraph{Abstention policy.} Given a threshold $\lambda \in [0,1]$, the framework returns $\hat{y}$ when confident and abstains otherwise:
\[
    \pi_\lambda(X,\bT) \;=\;
    \begin{cases}
        \hat{y}(X,\bT) & \text{if } \nu(X,\bT) > \lambda, \\
        \abstain & \text{otherwise.}
    \end{cases}
\]
A \emph{confident error} event occurs when $\pi_\lambda \notin \{\abstain, Y\}$, where $Y$ is the ground truth and $\abstain$ is abstention. In other words, a confident error appears when the system commits to an incorrect answer.
Confident errors carry asymmetrically higher cost than abstentions in deployment settings where an abstention can be routed to a human reviewer or a fallback policy. Thus \emph{confident error rate} is the appropriate target for a reliability layer to control.

\subsection{Score Separability and the Selective Accuracy Gain}
\label{sec:separability}

Raising $\lambda$ filters out low-confidence predictions, but whether this filtering actually improves accuracy on the retained subset depends on whether $\nu$ assigns \emph{systematically} higher values to correct answers. This is precisely the binary-discrimination setting for which classical tools apply: the Youden $J$-statistic~\citep{youden1950index} measures the vertical distance between a binary classifier's ROC curve and the chance diagonal, with $J(\lambda) > 0$ indicating that the classifier retains true positives at a strictly higher rate than false positives.

Although the discrimination is performed on $\nu$, the discriminative power of $\nu$ is inherited from the path-quality score $q$ and the way paths are aggregated. Three factors govern how the discriminative power of $q$ carries through to $\nu$. \emph{Answer concentration} is the signal majority voting already exploits: when one answer dominates the pool and is correct, $\nu$ is naturally high on correct events; when the dominant answer is wrong, $\nu$ will be misleadingly high. \emph{Score discrimination} is the path-level signal: a $q$ that assigns distinguishably higher values to correct paths improves $\nu$'s discriminativity regardless of how the answer pool happens to be distributed. \emph{Weight translation} carries this path-level signal through to the aggregation. This translation is consequential when an incorrect answer holds the plurality and only a sufficiently sensitive $w$ lets a minority of high-quality paths overturn the wrong majority. A quantitative decomposition of all three factors is given in Appendix~\ref{app:sep-factors}; for the calibration analysis that follows, their joint effect is summarized by a single statistic.

\begin{definition}[Score separability]
\label{def:separability}
Assume $p_v \in (0,1)$, so that both correct and incorrect predictions occur with positive probability. Define the correct and error conditional survival functions
\[
    S_{\text{cor}}(\lambda) \;\coloneqq\; \PP\bigl(\nu > \lambda \mid \hat{y} = Y\bigr),
    \qquad
    S_{\text{err}}(\lambda) \;\coloneqq\; \PP\bigl(\nu > \lambda \mid \hat{y} \neq Y\bigr),
\]
and the \emph{separability gap} $\Delta(\lambda) \coloneqq S_{\text{cor}}(\lambda) - S_{\text{err}}(\lambda)$.
We call the score \emph{separable at $\lambda$} if $\Delta(\lambda) > 0$.
\end{definition}

The next theorem converts this condition into an exact, closed-form accuracy statement.

\begin{theorem}[Accuracy Gain]
\label{thm:ca-gain}
Let $A_c(\lambda) \coloneqq \PP(\hat{y} = Y \mid \nu > \lambda)$ denote the \emph{selective accuracy} of the non-abstained predictions. If $\nu$ is separable at $\lambda$, then
\begin{equation}
\label{eq:ca-gain}
    A_c(\lambda) - p_v \;=\; \frac{p_v(1 - p_v)\,\Delta(\lambda)}{S_{\text{cor}}(\lambda) - (1 - p_v)\,\Delta(\lambda)} \;>\; 0.
\end{equation}
\end{theorem}

Equation~\eqref{eq:ca-gain} is informative in three ways. First, the gain is \emph{multiplicatively} driven by the separability gap $\Delta(\lambda)$: a score with $\Delta(\lambda) = 0$ produces zero accuracy gain regardless of how many paths are sampled or how aggressively $\lambda$ is set. Second, the $p_v(1-p_v)$ prefactor means the potential gain peaks at intermediate task difficulty and vanishes at both saturation ($p_v \to 1$) and at chance ($p_v \to 0$). Third, every quantity on the right-hand side of~\eqref{eq:ca-gain} is estimable from calibration data without access to test labels; we formalize this as an operational tool.

\begin{corollary}[Closed-form accuracy predictor]
\label{cor:predictor}
Let $\hat S_{\text{cor}}(\lambda)$, $\hat S_{\text{err}}(\lambda)$, $\hat p_v$ denote the empirical analogues of $S_{\text{cor}}, S_{\text{err}}, p_v$ computed on a labeled calibration sample of size $n$, and define the plug-in predictor
\begin{equation}
\label{eq:ac-hat}
    \widehat A_c(\lambda) \;\coloneqq\; \hat p_v + \frac{\hat p_v(1 - \hat p_v)\,\hat\Delta(\lambda)}{\hat S_{\text{cor}}(\lambda) - (1 - \hat p_v)\,\hat\Delta(\lambda)},
    \qquad \hat\Delta(\lambda) \coloneqq \hat S_{\text{cor}}(\lambda) - \hat S_{\text{err}}(\lambda).
\end{equation}
Then $\widehat A_c(\lambda) \xrightarrow{\;\mathrm{a.s.}\;} A_c(\lambda)$ as $n \to \infty$, uniformly over any compact interval of $\lambda$ on which $\Delta(\lambda) \geq \Delta_0 > 0$.
\end{corollary}

Corollary~\ref{cor:predictor} elevates~\eqref{eq:ca-gain} from a descriptive identity to a \emph{planning tool}: a practitioner can sweep $\lambda$ on calibration data and read off the predicted selective accuracy at every operating point, without running additional test-time evaluation. The uniform convergence rests on Glivenko--Cantelli applied to $\hat S_{\text{cor}}, \hat S_{\text{err}}$ and continuity of the right-hand side of~\eqref{eq:ca-gain} away from $\Delta = 0$, and is therefore distribution-free. A finite-sample concentration bound requires additional assumptions, stated and justified in Appendix~\ref{app:ca-bound}.

A converse result confirms that separability is also \emph{necessary} for abstention to help.

\begin{proposition}[Non-separability lower bound]
\label{prop:nonsep}
If $\Delta(\lambda) \leq 0$, then $A_c(\lambda) \leq p_v$, with equality iff $\Delta(\lambda) = 0$. In particular, a score with $\Delta(\lambda) \leq 0$ at every $\lambda > 0$ cannot improve selective accuracy via abstention at any threshold.
\end{proposition}


Theorem~\ref{thm:ca-gain} and Proposition~\ref{prop:nonsep} together identify the separability gap as the single summary statistic governing whether abstention is beneficial. When $\Delta(\lambda) = 0$, abstention cannot improve accuracy beyond the unfiltered vote $p_v$, so withholding yields no gain. This is precisely the failure mode that a discriminative path-quality score addresses: by assigning higher weight to correct paths, $q$ breaks the symmetry between correct and incorrect high-confidence predictions and opens a nonzero separability gap.

\subsection{Accuracy--Yield Frontier}
\label{sec:frontier}

Abstention improves selective accuracy but at a cost: every abstained input is a question the system declines to answer. We call the answered fraction \emph{yield} $Y(\lambda) \coloneqq \PP(\nu > \lambda)$, and the accuracy--yield trade-off is the central operational question for deployment. Separability guarantees that $A_c(\lambda) > p_v$ whenever $\Delta(\lambda) > 0$, but it does not say whether accuracy increases as yield decreases. That requires understanding the local behavior of the confidence distribution at each margin level $\lambda$, which we characterize through the hazard rate. Among retained predictions at margin $\lambda$, the hazard rate measures how quickly an additional increase in $\lambda$ thins them out:

$$h_{\mathrm{cor}}(\lambda) \;\coloneqq\; \lim_{d\lambda \to 0^+} \frac{1}{d\lambda} \, P\!\left(\nu \in [\lambda, \lambda + d\lambda] \,\big|\, \nu > \lambda,\ \hat y = Y\right).$$

with $h_{\mathrm{err}}(\lambda)$ defined analogously on incorrect predictions for continuous variable $\nu$; see Appendix~\ref{app:app-hazard} for discrete case and further details. This motivates a stricter local condition: are incorrect predictions filtered out faster than correct ones at margin $\lambda$?

\begin{definition}[Strict score separability]
\label{def:strict-separability}
The score is \emph{strictly separable} at $\lambda$ if the hazard gap $\delta(\lambda) \coloneqq h_{\mathrm{err}}(\lambda) - h_{\mathrm{cor}}(\lambda)$ is strictly positive.
\end{definition}

Strict separability is the local refinement: it requires that errors thin out faster than correct predictions at each margin level, not merely that the upper tail above $\lambda$ contains more correct mass overall.

\begin{proposition} [Accuracy--yield trade-off] 
\label{prop:pareto}
Suppose $\Delta(\lambda) > 0$ for all $\lambda \in (0, \lambda_{\max})$. Then $A_c(\lambda) > p_v$ and $Y(\lambda)$ is non-increasing on $(0, \lambda_{\max})$.

If additionally $\delta(\lambda) > 0$ for all $\lambda \in (0, \lambda_{\max})$, then $A_c(\lambda)$ is also non-decreasing.
\end{proposition}
Proposition \ref{prop:pareto} implies that under strict separability, the map $\lambda \mapsto (Y(\lambda), A_c(\lambda))$ has an interesting property which is aligned with the notion of Pareto efficiency. In other words, no operating point on the frontier dominates another: raising $\lambda$ increases accuracy and decreases yield.

Let $R(\lambda) \coloneqq \PP(\pi_\lambda \notin \{\abstain, Y\}) = (1-p_v)\,S_{\text{err}}(\lambda)$ denote the confident-error rate.

\begin{proposition}[Yield at fixed error]
\label{prop:yield}
For every $\lambda$, $Y(\lambda) = R(\lambda)/(1-p_v) + p_v\,\Delta(\lambda)$, and $A_c(\lambda) = 1 - R(\lambda)/Y(\lambda)$ whenever $Y(\lambda) > 0$.
\end{proposition}

Proposition~\ref{prop:yield} further characterizes the role of score
separability in the accuracy--yield trade-off. At a fixed confident-error
rate $\alpha$ and base voting accuracy $p_v$, yield increases linearly
with $\Delta(\lambda)$, with slope $p_v$.

\subsection{Calibrating the Confidence Threshold}
\label{sec:calibration}

The vote confidence $\nu$ supplies the right statistic for abstention, and any $\lambda$ with $\Delta(\lambda) > 0$ yields a gain given by~\eqref{eq:ca-gain}. It remains to pick a specific operating point from finite calibration data. We use conformal risk control (CRC)~\citep{angelopoulos2024conformal}: a principled calibration tool that bounds the confident-error rate at a user-specified level and fits naturally with exchangeable calibration data. In our setup, CRC goes beyond the usual error-rate bound: it selects a threshold that lands on the frontier of \Cref{prop:pareto}, navigating the reliability guarantee and the accuracy--yield trade-off through a single knob.

\begin{assumption}[Calibration exchangeability]
\label{assum:exch}
The pairs $(X_i, Y_i)_{i=1}^{n+1}$ are exchangeable, and the generation of path pools $\bT_i$ and scores $q(X_i, T_i^{(j)})$ depends only on $X_i$, never on $Y_i$.
\end{assumption}

The second clause holds automatically for any autoregressive language model. Exchangeability is the minimal distributional condition for conformal guarantees; it does not require the calibration prompts to be i.i.d.

\paragraph{CRC-calibrated threshold.} Given a calibration set $\{(X_i, Y_i, \bT_i)\}_{i=1}^n$ and a target confident-error rate $\alpha \in (0,1)$, define the per-example loss $L_i(\lambda) \coloneqq \Ind[\pi_\lambda(X_i, \bT_i) \notin \{\abstain, Y_i\}]$ and its empirical risk $\hat{R}(\lambda) \coloneqq \tfrac{1}{n} \sum_{i=1}^n L_i(\lambda)$. Because raising $\lambda$ only converts predictions into abstentions, $L_i(\lambda)$ is non-increasing in $\lambda$. The CRC threshold is
\begin{equation}
\label{eq:crc-threshold}
    \hatl \;=\; \inf\Bigl\{\lambda \in [0,1] \;|\; \hat R(\lambda) \leq \alpha - \frac{1-\alpha}{n} \Bigr\}.
\end{equation}

\begin{theorem}[Confident-error rate control]
\label{thm:crc}
Suppose Assumption~\ref{assum:exch} holds. Then the policy $\pi_{\hatl}$ with the threshold $\hatl$ defined by ~\eqref{eq:crc-threshold} satisfies
\[
\PP\left(\pi_{\hatl}(X_{n+1}, \bT_{n+1}) \notin \{\abstain, Y_{n+1}\} \right) \leq \alpha.
\]
\end{theorem}
Under Assumption~\ref{assum:exch}, Theorem~\ref{thm:crc}
controls the marginal probability of a confident error at
level $\alpha$, averaging over the calibration sample and
a fresh test example. This is a validity guarantee, rather than utility: the degenerate policy that always abstains trivially satisfies the bound. What makes $\hatl$ operationally meaningful is separability. When $\Delta(\hatl) > 0$, the threshold preferentially filters incorrect predictions, so the answered subset is strictly more accurate than the unfiltered vote, quantified by Theorem~\ref{thm:ca-gain}. Together, the two theorems deliver the framework's core promise: a certified upper bound on confident errors, achieved at an operating point where abstentions are informative rather than conservative.

\section{Experiments}
\label{sec:experiments}

Our experiments are designed to answer four questions. \textbf{(Q1)} Does the CRC-calibrated threshold $\hatl$ deliver the claimed confident-error rate control (Theorem~\ref{thm:crc})? \textbf{(Q2)} Does the closed-form accuracy gain of Theorem~\ref{thm:ca-gain} match observed selective accuracy? \textbf{(Q3)} How does the path-score choice affect separability and, consequently, the achievable accuracy--yield frontier? \textbf{(Q4)} How does conformal aggregation compare with simpler alternatives, and does abstention remain useful when the base model is already accurate?

\subsection{Experimental Setup}
\label{sec:exp-setup}

We evaluate four open-source reasoning models spanning two size regimes: DeepSeek-R1-Distill-Qwen-1.5B~\citep{guo2025deepseek}, Qwen3-1.7B~\citep{qwen3technicalreport}, Gemma-4-E4B~\citep{Gemma2024}, and Qwen3-8B, each generating $m=16$ independent CoT paths per prompt. We study four benchmarks: GSM8K~\citep{cobbe2021training} (grade-school math), MATH~\citep{hendrycksmath2021} (competition math), MATH-Hard (the difficulty-5 subset), and HotpotQA~\citep{yang2018hotpotqa} (multi-hop QA, distractor setting), yielding ten dataset--model configurations (Table~\ref{tab:exp-grid} in Appendix~\ref{app:exp-grid}). As an additional experiment testing whether abstention remains useful at high base accuracy, we run Qwen3-14B on the same MATH and GSM8K prompts, with the same $m$ and $n_\mathrm{cal}$, and fix the target grid $\alpha\in\{0.10, 0.05, 0.03, 0.02\}$ before generation (Appendix~\ref{app:cr-qwen14}). For each configuration, $n_\mathrm{cal}=200$ labeled prompts are held out for calibration and the remainder serve as test data. We evaluate three path-quality score families: \emph{Reward} (a reward model score), \emph{Judge} (an LLM judge), and \emph{Perplexity} (average per-token log-likelihood under the generating model) (Appendix~\ref{app:score-details}). Scoring adds little overhead, typically a small fraction of path-generation time (Appendix~\ref{app:runtimes}).


We compare against majority voting (MV), score-weighted voting without abstention, whose accuracy is the $p_v$ baseline of Theorem~\ref{thm:ca-gain}, and best-of-$m$.  To separate the contribution of weighting and calibrated abstention, we also report MV with CRC abstention on its vote fraction (MV + CRC), and the weighted vote with a \emph{naive} plug-in threshold, which selects the smallest $\lambda$ with $\hat R(\lambda) \le \alpha$, i.e.\ \eqref{eq:crc-threshold} without the finite-sample correction. All methods of a configuration are evaluated on the same calibration--test splits.


\subsection{Confident-Error Rate Control (Q1)}
\label{sec:exp-q1}

\begin{figure}[t]
    \centering
    \includegraphics[width=\textwidth]{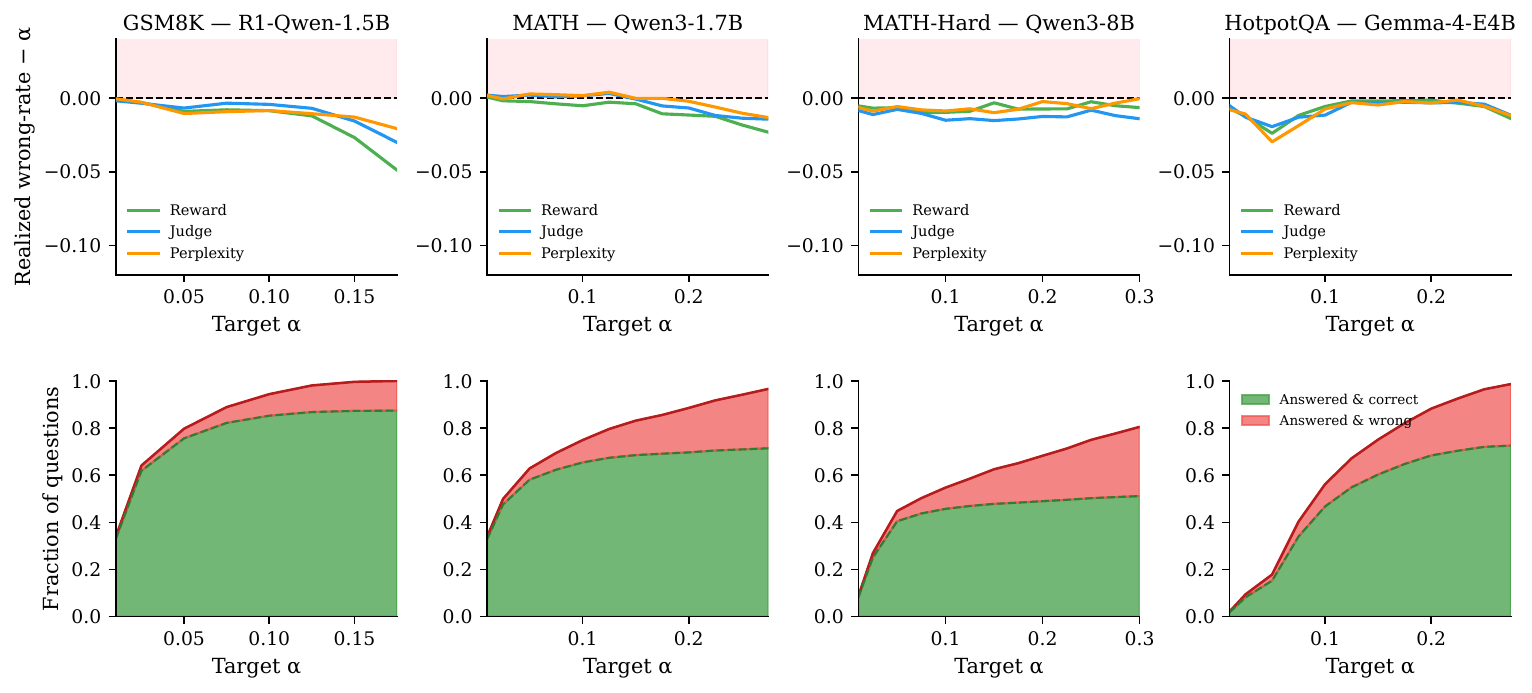}
    \caption{\textbf{Confident-error rate control (\Cref{thm:crc}).} \emph{Top:} Gap between the realized confident-error rate and the target~$\alpha$ across four representative configurations and all three score families. \emph{Bottom:} Fraction of questions answered correctly (green) or incorrectly (red) as~$\alpha$ varies. On easier tasks yield remains high even at strict~$\alpha$; on harder tasks the system abstains more aggressively.}
    \label{fig:alpha-control}
    \vspace{-1.5em}
\end{figure}

Figure~\ref{fig:alpha-control} validates Theorem~\ref{thm:crc} across all configurations and score families. The top row indicates that the mean realized confident-error rate aligns with the nominal target~$\alpha$ within calibration and test variability, consistent with the guarantees of CRC under exchangeability (Appendix~\ref{app:cr-baselines}).
The bottom row shows the operational consequence: as~$\alpha$ tightens, the green region shrinks while the red region is suppressed first. The system preferentially eliminates wrong answers before sacrificing correct ones. This asymmetric filtering is the signature of positive separability: the threshold disproportionately removes low-confidence predictions that are incorrect, so the answered subset becomes progressively more accurate as~$\alpha$ decreases.


\subsection{Validating the Closed-Form Accuracy Predictor (Q2)}
\label{sec:exp-q2}

\begin{figure}[t]
    \centering
    \includegraphics[width=\textwidth]{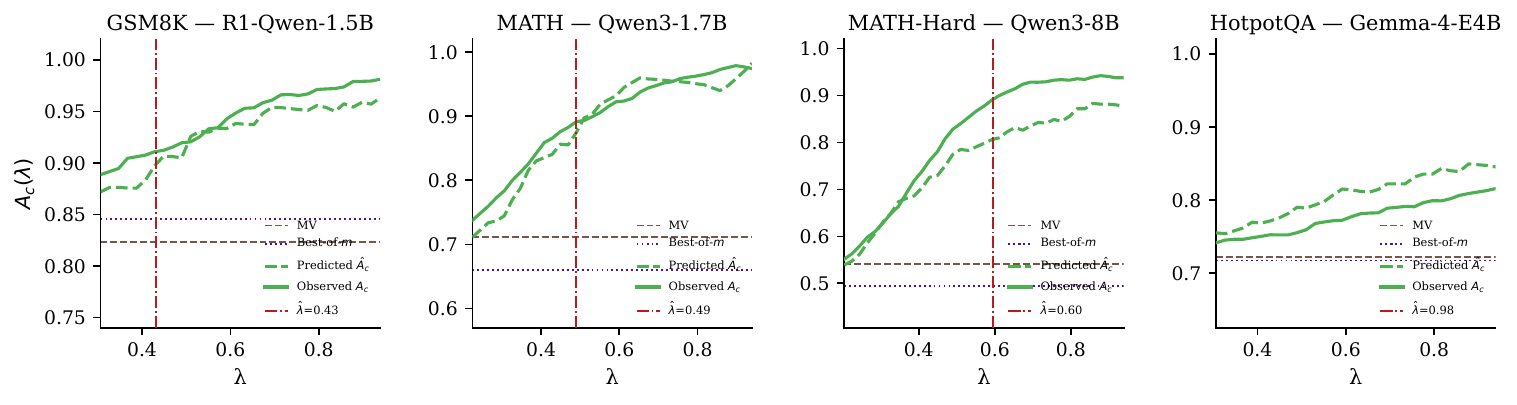}
    \caption{\textbf{Predicted versus observed selective accuracy.} Solid curves show the observed selective accuracy~$A_c$ on held-out test data; dashed curves show the calibration-set plug-in predictor~$\widehat A_c$ (Corollary~\ref{cor:predictor}). Vertical lines mark the CRC-calibrated~$\hat\lambda$ at $\alpha=0.10$. Horizontal baselines indicate majority-vote accuracy (dashed) and best-of-$m$ accuracy (dotted).}
    \label{fig:pred-vs-obs}
    \vspace{-1.5em}
\end{figure}
A distinctive feature of our framework is that selective accuracy at any operating point can be predicted from calibration data alone. Figure~\ref{fig:pred-vs-obs} overlays the plug-in predictor~$\widehat A_c$ against observed selective accuracy on held-out test data. Across all configurations, the two curves track closely, confirming that the closed-form expression of \Cref{thm:ca-gain} is an operational planning tool. The prediction is tightest on configurations with strong separability and high yield, and noisiest on the hardest settings where $\Delta$ is small, exactly as the theory prescribes. Consistent with the $p_v(1-p_v)$ interpretation, the largest absolute accuracy gains appear at intermediate task difficulty and diminish at both extremes. Section~\ref{sec:exp-q4} tests the upper extreme directly with a stronger model.

\subsection{Score Comparison and the Accuracy--Yield Frontier (Q3)}
\label{sec:exp-frontier}

\begin{figure}[t]
    \centering
    \includegraphics[width=\textwidth]{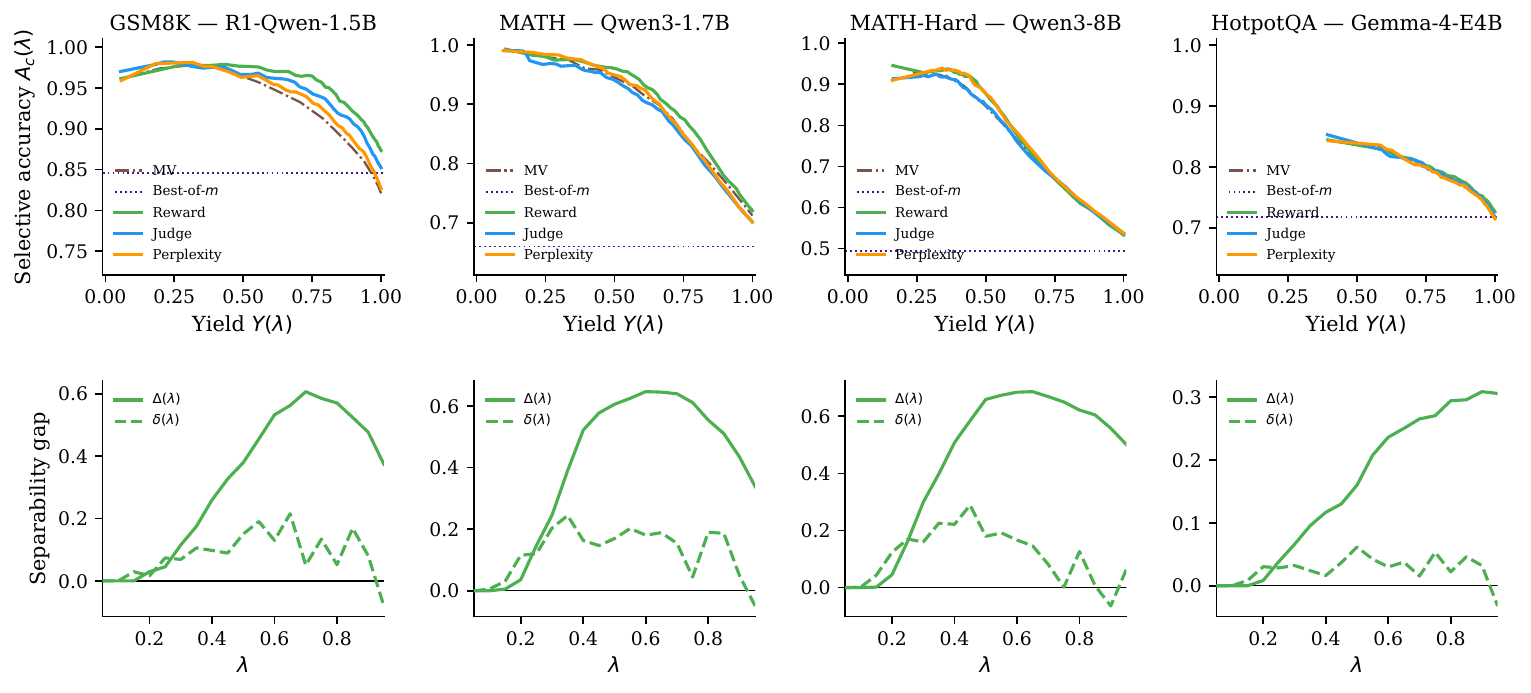}
    \caption{\textbf{Accuracy--yield frontiers and separability profiles.} \emph{Top:} Selective accuracy~$A_c$ versus yield for each score family. Higher and further right is better. \emph{Bottom:} Separability gap $\Delta(\lambda)$ (solid) and strict separability $\delta(\lambda)$ (dashed) for the Reward score.}
    \label{fig:frontiers}
    \vspace{-1.5em}
\end{figure}
Figure~\ref{fig:frontiers} compares all three score families on the yield--$A_c$ plane. Reward scores generally perform well, with frontiers that curve above the majority-vote baseline.
The separability profiles in the bottom row explain the shape of the frontiers. On HotpotQA, $\Delta(\lambda)$ is markedly smaller for all score families, and the frontiers correspondingly compress, offering little accuracy improvement over the unfiltered majority-vote baseline. Proposition~\ref{prop:nonsep} rules out a gain only when $\Delta \le 0$; for small positive $\Delta$, Theorem~\ref{thm:ca-gain} predicts a gain proportional to $\Delta$, which the HotpotQA frontiers show is small in practice.

The accuracy--yield frontier offers a concrete planning tool for practitioners. Tightening $\alpha$ moves the operating point along the frontier, trading yield for selective accuracy, and relaxing it recovers more answers.
The separability gap $\Delta(\lambda)$ is the operative quantity: where $\Delta(\lambda) > 0$, the frontier lies above the unfiltered vote, and it rises monotonically where strict separability also holds (Proposition~\ref{prop:pareto}). The closed-form predictor~$\widehat A_c$ traces the frontier from calibration data, so a practitioner can read off the achievable accuracy at any target yield without running additional evaluation. The frontier thus serves as a diagnostic: a frontier that stays close to $A_c = p_v$ signals weak separability and suggests that the path-quality score should be strengthened before tightening $\alpha$.


\subsection{Baselines and Higher Base Accuracy (Q4)}
\label{sec:exp-q4}

\begin{table}[t]
\centering
\footnotesize
\setlength{\tabcolsep}{4pt}
\caption{\textbf{Separating weighting, abstention, and calibration.} Selective accuracy $A_c$, yield $Y$, and realized confident-error rate (Err.) on test, averaged over 20 shared calibration--test splits (Reward score). Without abstention, $A_c$ is the accuracy of the vote. Bold: higher yield of the two CRC rules. Full results in Table~\ref{tab:ac-ladder-full}.}
\label{tab:ac-ladder}
\begin{tabular}{l cc cc c ccc}
\toprule
 & \multicolumn{2}{c}{No abstention} & \multicolumn{2}{c}{MV + CRC} & \begin{tabular}[b]{@{}c@{}}Weighted + naive\end{tabular} & \multicolumn{3}{c}{Weighted + CRC (ours)} \\
\cmidrule(lr){2-3}\cmidrule(lr){4-5}\cmidrule(lr){6-6}\cmidrule(lr){7-9}
Configuration & MV & Weighted & $A_c$ & $Y$ & Err. & Err. & $A_c$ & $Y$ \\
\midrule
\multicolumn{9}{l}{\emph{$\alpha = 0.10$}} \\
MATH-Hard / Qwen3-8B    & .536 & .534 & .844 & .50 & .098 & .094 & .832 & \textbf{.55} \\
MATH / Qwen3-1.7B       & .713 & .720 & .873 & .71 & .101 & .098 & .872 & \textbf{.75} \\
GSM8K / R1-Qwen-1.5B    & .823 & .874 & .896 & .84 & .099 & .094 & .901 & \textbf{.95} \\
MATH / Qwen3-14B        & .880 & .868 & .899 & .94 & .110 & .107 & .890 & \textbf{.96} \\
GSM8K / Qwen3-14B       & .892 & .883 & .914 & \textbf{.96} & .093 & .089 & .907 & \textbf{.96} \\
\midrule
\multicolumn{9}{l}{\emph{$\alpha = 0.05$}} \\
MATH-Hard / Qwen3-8B    & .536 & .534 & .901 & .39 & .049 & .046 & .902 & \textbf{.45} \\
MATH / Qwen3-1.7B       & .713 & .720 & .924 & .58 & .053 & .049 & .923 & \textbf{.63} \\
GSM8K / R1-Qwen-1.5B    & .823 & .874 & .942 & .65 & .048 & .042 & .949 & \textbf{.80} \\
MATH / Qwen3-14B        & .880 & .868 & .934 & .73 & .056 & .050 & .935 & \textbf{.77} \\
GSM8K / Qwen3-14B       & .892 & .883 & .955 & \textbf{.87} & .055 & .048 & .946 & .86 \\
\bottomrule
\end{tabular}
\vspace{-1em}
\end{table}

Table~\ref{tab:ac-ladder} separates the three ingredients of our policy on identical splits; full results and analyses appear in Appendix~\ref{app:cr}. \textbf{Weighting:} without abstention, it changes accuracy little except on GSM8K with R1-Qwen-1.5B; under CRC, it leaves $A_c$ essentially unchanged and raises yield on the main-grid configurations, also at matched realized error. \textbf{Abstention:} it drives the accuracy gain, lifting both CRC rules well above the unfiltered votes. \textbf{Calibration:} the naive threshold exceeds $\alpha$ more often, while CRC keeps the mean error at $\alpha$ up to split-to-split variability. \textbf{Higher base accuracy:} on Qwen3-14B, the gain from abstention shrinks as the $p_v(1-p_v)$ factor of Theorem~\ref{thm:ca-gain} predicts but persists, and weighting adds little; at stricter targets, the attainable error on MATH is bounded by answer grading (Appendix~\ref{app:cr-floor}).

\subsection{Ablations}
\label{sec:ablations}

We ablate three design choices on a representative setting (GSM8K, Qwen3-1.7B, $\alpha{=}0.10$); full tables appear in Appendix~\ref{app:ablation-results}. Figure~\ref{fig:ablation-summary} summarizes the findings. \textbf{Number of paths:} Increasing $m$ from 4 to 20 steadily improves yield while $A_c$ saturates early (panel~a); shaded bands show standard deviation across splits, confirming the trend is robust. Additional paths sharpen $\nu$ without altering separability, expanding the frontier rather than shifting it. \textbf{Calibration size:} $A_c$ is stable across $n_\mathrm{cal}\in\{50,\ldots,500\}$; larger calibration sets close the gap between realized and target $\alpha$ (panel~b). \textbf{Weight function:} Among $w(q)=\exp(\beta\, q)$, $\beta=1$ attains the best accuracy–yield trade-off (panel~c). $\beta=5$ collapses aggregation toward best-of-$m$, forfeiting pooling benefits.
\textbf{Score type:} Reward or Judge attains the highest Frontier AUC on six of the ten configurations and Perplexity on three (one tie), while SC is the strongest score on only one configuration and the weakest on six (panel~d).


\begin{figure}[t]
	\centering
	\includegraphics[width=\textwidth]{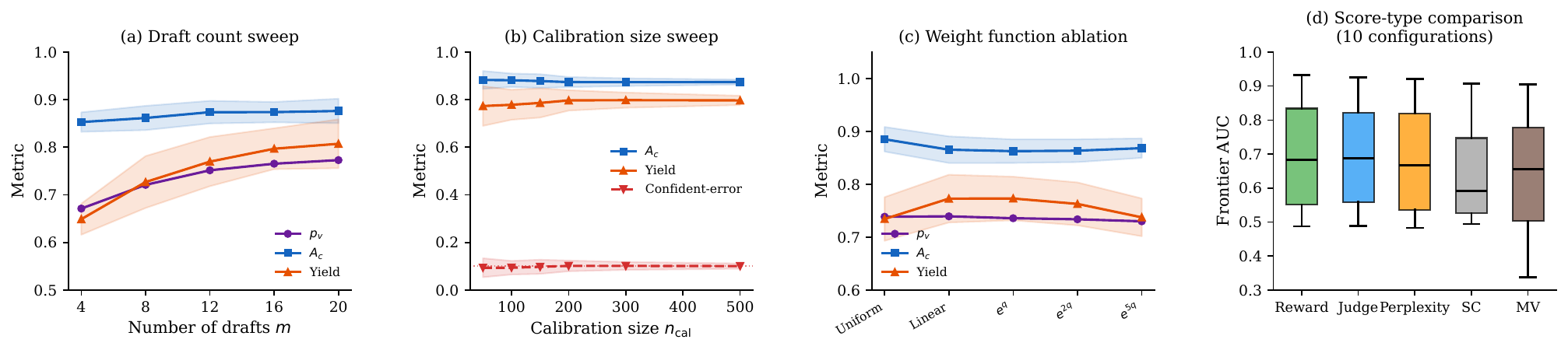}
    \caption{\textbf{Ablation summary}. \textbf{(a)}~Sweep on the number of sampled paths~$m$ and \textbf{(b)}~on calibration sizes $n_\mathrm{cal}\in\{50,\ldots,500\}$, with the Skywork-V2-Llama-3.2-3B reward scorer. \textbf{(c)}~Sweep on the weight choice, with the Gemma-26B judge scorer. \textbf{(d)}~Frontier AUC across all ten configurations.}
	\label{fig:ablation-summary}
    \vspace{-1em}
\end{figure}
\section{Conclusion}
\label{sec:conclusion}

We presented a conformal framework for chain-of-thought reasoning that directly addresses aggregation uncertainty in self-consistency. By replacing majority voting with score-weighted aggregation and calibrating an abstention threshold via conformal risk control, the framework provides finite-sample guarantees on the confident-error rate while remaining fully inference-time.
Reliability is therefore score-agnostic: the guarantee holds regardless of which path-quality score is used. Utility, however, is diagnostic: we identify score separability as the condition for abstention to improve selective accuracy, and derive a closed-form predictor that quantifies expected gains from calibration data alone. Empirically, the framework consistently delivers guaranteed error-rate control across four benchmarks and four models, and the closed-form predictor closely tracks observed accuracy. These results highlight that reliability in multi-path reasoning is fundamentally governed by aggregation quality rather than individual reasoning traces.

\subsection*{Limitations}
The confident-error rate guarantee is marginal over
the calibration sample and a fresh test example,
rather than a per-input or per-calibration-split
guarantee. It also relies on exchangeability and does not
extend directly to distribution shift. Our cross-dataset
experiments quantify the resulting changes in error
rates and the effect of target-domain recalibration
(Appendix~\ref{app:cr-shift}). Calibration size and
confidence granularity also affect the attainable yield
(Appendix~\ref{app:cr-ncal} and
\ref{app:cr-granularity}).
Our evaluation is restricted to tasks with extractable
answers and binary correctness. Extending empirical
validation to open-ended generation requires
task-specific definitions of correctness and
corresponding evaluation protocols. Although the
theoretical framework accommodates bounded graded
losses (Appendix~\ref{app:graded-loss}), its practical
effectiveness in these settings remains to be evaluated.

\clearpage

\begin{ack}
    We thank GERAD for providing GPU computing resources and IT support. Masoud Asgharian acknowledges support from the Natural Sciences and Engineering Research Council of Canada (NSERC) through Discovery Grant RGPIN-2024-05640. Zijun Yu acknowledges partial support from an ISM Graduate Scholarship. We thank Yuanhao Yu, Xinyu Wang, and Peng Lu for helpful discussions. Vahid Partovi Nia thanks Irina Kezele for her early technical feedback. We acknowledge Noah's Ark Lab, Montr\'eal Research Centre, for discussions that informed the early stages of this work.
\end{ack}

\bibliographystyle{apalike}
\bibliography{references}

\clearpage
\appendix

\appendixpage

{\small
\startcontents[sections]
\printcontents[sections]{l}{1}{\setcounter{tocdepth}{2}}
}



\section{Notations}
\label{app:notation}

Table~\ref{tab:notation} summarizes the principal symbols used in the paper.

\begin{table}[h]
\centering
\caption{Summary of notation.}
\label{tab:notation}
\small
\renewcommand{\arraystretch}{1.3}
\begin{tabular}{p{3cm} p{10cm}}
\toprule
\textbf{Symbol} & \textbf{Description} \\
\midrule
$\cX,\; \cY,\; \cT$ & Prompt, answer, and reasoning-path spaces \\
$X,\; Y$ & Input prompt and ground-truth answer \\
$m$ & Number of sampled CoT paths per prompt \\
$\bT = (T^{(1)}, \ldots, T^{(m)})$ & Pool of sampled reasoning paths \\
$a : \cT \to \cY$ & Extraction function mapping a path to a candidate answer \\
$q : \cX \times \cT \to \mathbb R$ & Path-quality score \\
$w : \mathbb R \to \mathbb R^+$ & Non-decreasing weight function\\
\midrule
$V_y(X, \bT)$ & Score-weighted vote tally for candidate answer $y$ \\
$\hat{y}$ & Aggregated prediction: $\arg\max_y V_y$ \\
$\nu \in [0,1]$ & Vote confidence: normalized vote share of $\hat{y}$ \\
$p_v$ & Vote accuracy: $\PP(\hat{y} = Y)$ \\
\midrule
$\lambda$ & Abstention threshold on $\nu$ \\
$\pi_\lambda$ & Abstention policy: returns $\hat{y}$ if $\nu > \lambda$, else $\bot$ \\
$\alpha$ & Target confident-error rate control \\
$\hat{R}(\lambda)$ & Empirical confident-error rate on the calibration set \\
$\hatl$ & CRC-calibrated threshold (Eq.~\eqref{eq:crc-threshold}) \\
$n$ & Calibration set size \\
\midrule
$S_{\text{cor}}(\lambda),\; S_{\text{err}}(\lambda)$ & Conditional survival functions: $\PP(\nu > \lambda \mid \hat{y} = Y)$ and $\PP(\nu > \lambda \mid \hat{y} \neq Y)$ \\
$\Delta(\lambda)$ & Separability gap: $S_{\text{cor}}(\lambda) - S_{\text{err}}(\lambda)$ \\
$A_c(\lambda)$ & Conditional accuracy: $\PP(\hat{y} = Y \mid \nu > \lambda)$ \\
$Y(\lambda)$ & Yield: $\PP(\nu > \lambda)$ \\
$\widehat{A}_c(\lambda)$ & Plug-in accuracy predictor (Eq.~\eqref{eq:ac-hat}) \\
\bottomrule
\end{tabular}
\end{table}


\section{Proofs}
\label{app:proofs}

\subsection{Proof of Accuracy Gain}
\label{app:proof-acgain}
\begin{proof}[Proof of \Cref{thm:ca-gain}]
The joint probabilities of $\{\nu > \lambda\}$ with each correctness event are, by definition of conditional probability,
\[
	\PP(\nu > \lambda,\, \hat y = Y) = S_{\text{cor}}(\lambda)\, p_v,
	\qquad
	\PP(\nu > \lambda,\, \hat y \neq Y) = S_{\text{err}}(\lambda)\,(1-p_v).
\]
Bayes' rule on the event $\{\nu > \lambda\}$ gives
\begin{equation}
	\label{eq:ac-bayes}
	A_c(\lambda) \;=\; \frac{S_{\text{cor}}(\lambda)\,p_v}{S_{\text{cor}}(\lambda)\,p_v + S_{\text{err}}(\lambda)\,(1-p_v)}.
\end{equation}
Let $D \coloneqq S_{\text{cor}}(\lambda)\,p_v + S_{\text{err}}(\lambda)\,(1-p_v)$ denote the denominator of~\eqref{eq:ac-bayes}. Subtracting $p_v$,
\[
	A_c(\lambda) - p_v
	\;=\; \frac{S_{\text{cor}}(\lambda)\,p_v - p_v\, D}{D}
	\;=\; \frac{p_v\bigl[S_{\text{cor}}(\lambda) - S_{\text{cor}}(\lambda)\,p_v - S_{\text{err}}(\lambda)\,(1-p_v)\bigr]}{D}.
\]
Using $1 = p_v + (1-p_v)$, the bracketed numerator collapses:
\[
	S_{\text{cor}}(\lambda)\bigl[1 - p_v\bigr] - S_{\text{err}}(\lambda)\,(1-p_v)
	\;=\; (1-p_v)\bigl[S_{\text{cor}}(\lambda) - S_{\text{err}}(\lambda)\bigr]
	\;=\; (1-p_v)\,\Delta(\lambda).
\]
For the denominator, substituting $S_{\text{err}}(\lambda) = S_{\text{cor}}(\lambda) - \Delta(\lambda)$ yields
\[
	D \;=\; S_{\text{cor}}(\lambda)\bigl[p_v + (1-p_v)\bigr] - (1-p_v)\,\Delta(\lambda)
	\;=\; S_{\text{cor}}(\lambda) - (1-p_v)\,\Delta(\lambda).
\]
Combining numerator and denominator gives~\eqref{eq:ca-gain}. For positivity: $p_v, (1-p_v) \in (0,1)$ by assumption; $\Delta(\lambda) > 0$ by separability; and $D = S_{\text{cor}}(\lambda) - (1-p_v)\Delta(\lambda) \geq S_{\text{cor}}(\lambda) - \Delta(\lambda) = S_{\text{err}}(\lambda) \geq 0$, with strict positivity unless $S_{\text{err}}(\lambda) = 0$ and $p_v = 0$ simultaneously, a degenerate case excluded by $p_v \in (0,1)$.
\end{proof}

\subsection{Proof of the Closed-form Accuracy Predictor}
\label{app:proof-corcf}

The proof rests on a small algebraic identity that recasts $A_c$ as a ratio of two non-negative survival functions, after which Glivenko-Cantelli and continuous mapping handle the rest.

\paragraph{An equivalent ratio form.}
For $\lambda \in [0,1]$, define
\begin{align}
	G(\lambda) &\coloneqq \PP\bigl(\nu > \lambda,\ \hat y = Y\bigr) \;=\; p_v\,S_{\mathrm{cor}}(\lambda), \\
	H(\lambda) &\coloneqq \PP(\nu > \lambda) \;=\; p_v\,S_{\mathrm{cor}}(\lambda) + (1-p_v)\,S_{\mathrm{err}}(\lambda) \;=\; Y(\lambda),
\end{align}
with empirical counterparts
\[
	\hat G(\lambda) = \frac{1}{n}\sum_{i=1}^n \Ind[\nu_i > \lambda,\ \hat y_i = Y_i],
	\qquad
	\hat H(\lambda) = \frac{1}{n}\sum_{i=1}^n \Ind[\nu_i > \lambda].
\]
Note that $G$ and $H$ are non-increasing in $\lambda$, $H \geq G$ pointwise, and
\begin{equation}
	\label{eq:ac-as-ratio}
	A_c(\lambda) \;=\; \frac{G(\lambda)}{H(\lambda)} \quad \text{whenever } H(\lambda) > 0.
\end{equation}
The same identity holds for the plug-in predictor:
\begin{equation}
	\label{eq:hatac-as-ratio}
	\hat A_c(\lambda) \;=\; \hat p_v + \frac{\hat p_v(1-\hat p_v)\,\hat\Delta(\lambda)}{\hat S_{\mathrm{cor}}(\lambda) - (1-\hat p_v)\,\hat\Delta(\lambda)} \;=\; \frac{\hat G(\lambda)}{\hat H(\lambda)},
\end{equation}
which follows from the same algebra as in the proof of \Cref{thm:ca-gain}, applied to the empirical analogues:
\[
	\hat S_{\mathrm{cor}}(\lambda) - (1-\hat p_v)\,\hat\Delta(\lambda)
	\;=\; \hat p_v\,\hat S_{\mathrm{cor}}(\lambda) + (1-\hat p_v)\,\hat S_{\mathrm{err}}(\lambda)
	\;=\; \hat H(\lambda),
\]
together with $\hat p_v\,\hat S_{\mathrm{cor}}(\lambda) = \hat G(\lambda)$. Working with $(\hat G, \hat H)$ rather than $(\hat S_{\mathrm{cor}}, \hat S_{\mathrm{err}}, \hat p_v)$ folds the stratum sizes $|\{i : \hat y_i = Y_i\}|$ and $|\{i : \hat y_i \neq Y_i\}|$ into the joint empirical, removing the need for separate concentration on those binomial counts.

\begin{proof}[Proof of Corollary~\ref{cor:predictor}]
We work under the iid restriction of Assumption~\ref{assum:exch}, i.e.\ the calibration tuples $(X_i, Y_i, \bT_i)_{i=1}^n$ are iid. The extension to infinite exchangeable calibration sequences is discussed in \Cref{app:ca-bound}.

\emph{Step 1 (uniform convergence of $\hat G$ and $\hat H$).}
Set $Z_i \coloneqq \bigl(\nu_i, \Ind[\hat y_i = Y_i]\bigr) \in [0,1]\times\{0,1\}$. Under iid calibration data, $(Z_i)_{i=1}^n$ is iid. Both $\hat G$ and $\hat H$ are empirical means of indicators of half-line events parametrised by $\lambda$, a totally ordered class. The Glivenko-Cantelli theorem~\cite{berti1997glivenko} therefore gives
\begin{equation}
	\label{eq:gc-uniform}
	\sup_{\lambda \in [0,1]} \bigl|\hat G(\lambda) - G(\lambda)\bigr| \xrightarrow{\mathrm{a.s.}} 0,
	\qquad
	\sup_{\lambda \in [0,1]} \bigl|\hat H(\lambda) - H(\lambda)\bigr| \xrightarrow{\mathrm{a.s.}} 0.
\end{equation}

\emph{Step 2 (the denominator is bounded below on $\Lambda_0$).}
Fix $\Delta_0 > 0$ and let $\Lambda_0 \subset [0,1]$ be a compact set on which $\Delta(\lambda) \geq \Delta_0$. From $H(\lambda) = p_v\,S_{\mathrm{cor}}(\lambda) + (1-p_v)\,S_{\mathrm{err}}(\lambda) \geq p_v\,S_{\mathrm{cor}}(\lambda) \geq p_v\,\Delta(\lambda)$,
\begin{equation}
	\label{eq:H-lower-bound}
	s_0 \;\coloneqq\; \inf_{\lambda \in \Lambda_0} H(\lambda) \;\geq\; p_v\,\Delta_0 \;>\; 0.
\end{equation}
We refer to $s_0$ as the \emph{answering margin}: it is the smallest yield attained on the operating range $\Lambda_0$.

\emph{Step 3 (continuous mapping).}
On the full-probability event in~\eqref{eq:gc-uniform}, set $\eta_n \coloneqq \sup_\lambda |\hat G - G| + \sup_\lambda |\hat H - H| \to 0$ a.s. For $\lambda \in \Lambda_0$ and $n$ large enough that $\eta_n < s_0/2$, $\hat H(\lambda) \geq H(\lambda) - \eta_n \geq s_0/2 > 0$. The quotient identity
\begin{equation}
	\label{eq:quotient-identity}
	\frac{\hat G}{\hat H} - \frac{G}{H}
	\;=\; \frac{\hat G - G}{\hat H} \;-\; A_c \cdot \frac{\hat H - H}{\hat H},
\end{equation}
combined with $|A_c| \leq 1$, gives, for every such $\lambda$,
\[
	\bigl|\hat A_c(\lambda) - A_c(\lambda)\bigr|
	\;\leq\; \frac{|\hat G(\lambda) - G(\lambda)| + |\hat H(\lambda) - H(\lambda)|}{\hat H(\lambda)}
	\;\leq\; \frac{\,\eta_n}{s_0/2}
	\;=\; \frac{2\,\eta_n}{s_0}.
\]
Since $\eta_n \to 0$ a.s., the right-hand side tends to zero, uniformly in $\lambda \in \Lambda_0$:
\[
	\sup_{\lambda \in \Lambda_0} \bigl|\hat A_c(\lambda) - A_c(\lambda)\bigr| \xrightarrow{\mathrm{a.s.}} 0,
\]
proving Corollary~\ref{cor:predictor}. The same argument with the DKW-Massart inequality in place of Glivenko-Cantelli yields a finite-sample rate, given as \Cref{thm:ac-concentration} in \Cref{app:ca-bound}.
\end{proof}

\subsection{Proof of the Accuracy-Yield Trade-off}
\label{app:proof-ay}

\begin{proof}[Proof of \Cref{prop:pareto}]
\emph{Part 1 (separability $\Rightarrow$ pointwise improvement and monotone yield).} Under $\Delta(\lambda) > 0$ for all $\lambda \in (0, \lambda_{\max})$, \Cref{thm:ca-gain} gives $A_c(\lambda) > p_v$ at every such $\lambda$. The yield $Y(\lambda) = \PP(\nu > \lambda)$ is a survival function of $\nu$ and is therefore non-increasing in $\lambda$ by definition; this requires no separability assumption.

\emph{Part 2 (strict separability $\Rightarrow$ $A_c$ non-decreasing).} Assume additionally that the score is strictly separable $\delta(\lambda) > 0$ at every $\lambda \in (0, \lambda_{\max})$. Inverting the Bayes form~\eqref{eq:ac-bayes},
\begin{equation}
	\label{eq:ac-recip}
	\frac{1}{A_c(\lambda)} \;=\; 1 + \frac{1-p_v}{p_v} \cdot \frac{S_{\mathrm{err}}(\lambda)}{S_{\mathrm{cor}}(\lambda)},
\end{equation}
so $A_c(\lambda)$ is non-decreasing in $\lambda$ if and only if the survival ratio
\(
	\rho(\lambda) \coloneqq S_{\mathrm{err}}(\lambda)/S_{\mathrm{cor}}(\lambda)
\)
is non-increasing in $\lambda$. Under the absolute-continuity condition stated in \Cref{app:app-hazard} (the conditional laws of $\nu$ given $\{\hat y = Y\}$ and $\{\hat y \neq Y\}$ admit densities on $(0, \lambda_{\max})$), $S_{\mathrm{cor}}, S_{\mathrm{err}}$ are absolutely continuous and strictly positive on $(0, \lambda_{\max})$ with hazard recovery
\[
	S_{\mathrm{cor}}(\lambda) = \exp\Bigl(-\!\int_0^\lambda h_{\mathrm{cor}}(u)\,du\Bigr),
	\qquad
	S_{\mathrm{err}}(\lambda) = \exp\Bigl(-\!\int_0^\lambda h_{\mathrm{err}}(u)\,du\Bigr).
\]
Hence
\begin{equation}
	\label{eq:rho-as-integral}
	\log \rho(\lambda) \;=\; -\!\int_0^\lambda \bigl[h_{\mathrm{err}}(u) - h_{\mathrm{cor}}(u)\bigr]\,du \;=\; -\!\int_0^\lambda \delta(u)\,du.
\end{equation}
Since $\delta > 0$ on $(0, \lambda_{\max})$ by assumption, $\log \rho$ is non-increasing, $\rho$ is non-increasing, and $A_c$ is non-decreasing on $(0, \lambda_{\max})$.

\emph{Pareto-frontier interpretation.} Combining the two parts: under strict separability, $\lambda \mapsto Y(\lambda)$ is non-increasing and $\lambda \mapsto A_c(\lambda)$ is non-decreasing. Equivalently, there are no $\lambda_1 < \lambda_2$ in $(0, \lambda_{\max})$ with $Y(\lambda_2) \geq Y(\lambda_1)$ and $A_c(\lambda_2) \leq A_c(\lambda_1)$ both holding and at least one strict. The image of $(0, \lambda_{\max})$ under $\lambda \mapsto (Y(\lambda), A_c(\lambda))$ is therefore a Pareto frontier in the yield-accuracy plane.
\end{proof}

\subsection{Proof of Confident-error Rate Control}
This proof follows from the proof of Conformal Risk Control~\citep{angelopoulos2024conformal}.

\begin{proof}[Proof of \Cref{thm:crc}]
We use the following right-continuous threshold convention throughout this proof:
\[
\pi_\lambda(X,T)=
\begin{cases}
\hat y(X,T), & \nu(X,T)>\lambda,\\
\perp, & \nu(X,T)\le \lambda .
\end{cases}
\]
Equivalently, the system answers only when its vote confidence strictly exceeds the threshold.
Let
\[
Z_i=(X_i,Y_i,T_i),\qquad i=1,\ldots,n+1 .
\]
We assume that the augmented examples \(Z_1,\ldots,Z_{n+1}\) are exchangeable. This follows, for
example, if \((X_i,Y_i)_{i=1}^{n+1}\) are exchangeable and, conditionally on \(X_i\), the path pool
\(T_i\) is generated from the same randomized inference procedure independently across examples.
The aggregation rule, tie-breaking rule, score computation, and weight function are fixed in advance
and do not access \(Y_i\) except through the final loss evaluation.

For each threshold \(\lambda\in[0,1]\), define the confident-error loss
\[
L_i(\lambda)
=
\mathbf 1\!\left\{
\pi_\lambda(X_i,T_i)\notin\{\perp,Y_i\}
\right\}
=
\mathbf 1\!\left\{
\nu(X_i,T_i)>\lambda,\ \hat y(X_i,T_i)\neq Y_i
\right\}.
\]
For a dataset \(D_k=(Z_1,\ldots,Z_k)\), write
\[
\widehat R_k(\lambda;D_k)
=
\frac1k\sum_{i=1}^k L_i(\lambda).
\]
The calibration threshold is
\[
\widehat\lambda
=
\inf\left\{
\lambda\in[0,1]:
\widehat R_n(\lambda;D_n)
\le
\alpha-\frac{1-\alpha}{n}
\right\},
\]
whenever the set is nonempty. If
\(\alpha-(1-\alpha)/n<0\), we set \(\widehat\lambda=1\); then the policy always abstains and the
claim is trivial. Hence, in the rest of the proof, assume
\[
\beta_n:=\alpha-\frac{1-\alpha}{n}\ge 0.
\]

First note that, for every fixed example \(i\), the map
\[
\lambda\mapsto L_i(\lambda)
\]
takes values in \([0,1]\), is monotone non-increasing, and is right-continuous. Indeed,
\(L_i(\lambda)=\mathbf 1\{\nu_i>\lambda,\hat y_i\neq Y_i\}\), where
\(\nu_i=\nu(X_i,T_i)\) and \(\hat y_i=\hat y(X_i,T_i)\). Moreover,
\[
L_i(1)=0
\]
because \(\nu_i\le 1\). Therefore \(\widehat R_k(\lambda;D_k)\) is also monotone non-increasing and
right-continuous in \(\lambda\). In particular,
\[
\widehat R_n(\widehat\lambda;D_n)\le \beta_n .
\]

We now introduce the oracle threshold that would be chosen if the test point were also available:
\[
\lambda^\star
=
\inf\left\{
\lambda\in[0,1]:
\widehat R_{n+1}(\lambda;D_{n+1})\le \alpha
\right\}.
\]
By right-continuity,
\[
\widehat R_{n+1}(\lambda^\star;D_{n+1})\le \alpha .
\]
Since \(\lambda^\star\) is a symmetric function of the exchangeable sample
\(Z_1,\ldots,Z_{n+1}\), the random variables
\[
L_1(\lambda^\star),\ldots,L_{n+1}(\lambda^\star)
\]
have the same expectation. Hence
\[
\mathbb E\!\left[L_{n+1}(\lambda^\star)\right]
=
\frac{1}{n+1}\sum_{i=1}^{n+1}
\mathbb E\!\left[L_i(\lambda^\star)\right]
=
\mathbb E\!\left[
\widehat R_{n+1}(\lambda^\star;D_{n+1})
\right]
\le \alpha .
\]

It remains to compare \(\widehat\lambda\) with \(\lambda^\star\). For any fixed \(\lambda\in[0,1]\),
because \(0\le L_{n+1}(\lambda)\le 1\),
\[
\widehat R_{n+1}(\lambda;D_{n+1})
=
\frac{n}{n+1}\widehat R_n(\lambda;D_n)
+
\frac{1}{n+1}L_{n+1}(\lambda)
\le
\frac{n}{n+1}\widehat R_n(\lambda;D_n)
+
\frac{1}{n+1}.
\]
Therefore, if \(\widehat R_n(\lambda;D_n)\le \beta_n\), then
\[
\widehat R_{n+1}(\lambda;D_{n+1})
\le
\frac{n}{n+1}
\left(
\alpha-\frac{1-\alpha}{n}
\right)
+
\frac{1}{n+1}
=
\alpha .
\]
Thus every threshold feasible for the calibration problem is feasible for the oracle problem. Consequently,
\[
\lambda^\star\le \widehat\lambda .
\]
Since \(L_{n+1}(\lambda)\) is monotone non-increasing in \(\lambda\), we have
\[
L_{n+1}(\widehat\lambda)
\le
L_{n+1}(\lambda^\star).
\]
Taking expectations gives
\[
\mathbb E\!\left[L_{n+1}(\widehat\lambda)\right]
\le
\mathbb E\!\left[L_{n+1}(\lambda^\star)\right]
\le
\alpha .
\]
Finally, because \(L_{n+1}(\widehat\lambda)\) is exactly the indicator of a confident error,
\[
\mathbb P\!\left(
\pi_{\widehat\lambda}(X_{n+1},T_{n+1})
\notin
\{\perp,Y_{n+1}\}
\right)
=
\mathbb E\!\left[L_{n+1}(\widehat\lambda)\right]
\le
\alpha .
\]
\end{proof}
This proves the confident-error rate control guarantee.

\paragraph{Proof of \Cref{prop:yield}.}
$Y = p_v S_{\text{cor}} + (1-p_v) S_{\text{err}} = S_{\text{err}} + p_v\Delta = R/(1-p_v) + p_v\Delta$, and $\PP(\hat y = Y,\ \nu > \lambda) = Y - R$ gives $A_c = 1 - R/Y$. The same computation shows that the denominator in~\eqref{eq:ca-gain} equals $Y(\lambda)$. \qed

\section{Additional Theoretical Details}
\label{app:additional}

\subsection{Conditional decomposition of the separability gap}
\label{app:sep-factors}

This appendix gives the path-level structure behind the separability gap. We work under Assumption~\ref{assum:exch}, which implies that conditional on the prompt $X$, the sampled paths $T^{(1)}, \ldots, T^{(m)}$ are iid from the LM's path distribution. The aim here is interpretive: to identify which sources of signal the framework can exploit, and to explain why $\Delta$ does not admit a clean additive decomposition into per-factor contributions.

\paragraph{A two-step factorisation of the vote.} For each path $j$, let
\[
	A_j \coloneqq a(T^{(j)}) \in \cY,
	\qquad
	W_j \coloneqq w\!\bigl(q(X, T^{(j)})\bigr) \in \RR^+
\]
denote the path's predicted answer and its vote weight. The pairs $(A_j, W_j)_{j=1}^m$ are conditionally iid given $X$, and their joint conditional law is captured by two ingredients:
\begin{itemize}
	\item the \emph{answer distribution} $\pi_y(X) \coloneqq \PP(A_j = y \mid X)$ on $\cY$;
	\item the \emph{conditional weight law} $G_y(\cdot \mid X)$, i.e.\ the law of $W_j$ given $\{A_j = y, X\}$.
\end{itemize}
Letting $\bN \coloneqq (N_y)_{y \in \cY}$ with $N_y \coloneqq \sum_j \Ind[A_j = y]$, the vote is generated in two steps:
\begin{enumerate}
	\item[\textbf{(i)}] $\bN \mid X \sim \mathrm{Multinomial}\bigl(m;\,\pi(X)\bigr)$;
	\item[\textbf{(ii)}] given $(\bN, X)$, the cluster weights $V_y \coloneqq \sum_{j: A_j = y} W_j$ are independent across $y$, with $V_y$ a sum of $N_y$ iid draws from $G_y(\cdot \mid X)$.
\end{enumerate}
The vote confidence is $\nu = V_{\hat y}/\sum_{y'} V_{y'}$ with $\hat y = \arg\max_{y'} V_{y'}$. This factorisation cleanly separates the design surface for separability into three sources:
\begin{itemize}
	\item \textbf{Factor A (answer concentration).} The shape of $\pi(X)$, entering only through step (i). When the LM is sharply peaked on the correct answer on most prompts but diffuse on most wrong-answer prompts, step (i) alone produces separability.
	\item \textbf{Factor B (score-weight discriminativity).} The relationship between $G_Y(\cdot \mid X)$ and $\{G_y(\cdot \mid X)\}_{y \neq Y}$, entering only through step (ii). When weights on correct paths tend to be larger than weights on wrong paths, step (ii) tilts $\nu$ in favour of the correct cluster.
	\item \textbf{Factor C (weight translation).} The mapping from the underlying score $q$ to vote weights $W = w(q)$. The choice of $w$ controls how much path-level discriminativity in $q$ converts to vote-level mass shift.
\end{itemize}

\paragraph{Under majority voting, only Factor A is active.} The cleanest case is $w \equiv 1$: weights carry no information beyond the answer multiset, and the entire framework reduces to the multinomial structure of step (i).

\begin{lemma}[MV recovers Factor A]
\label{lem:mv-factor-a}
Under majority voting ($w \equiv 1$, so $V_y = N_y$ and $\nu = N_{\hat y}/m$), the conditional laws of $\nu$ given $\{\hat y = Y\}$ and $\{\hat y \neq Y\}$ are determined by $\pi(\cdot)$ alone. In particular, $S^{\mathrm{MV}}_{\mathrm{cor}}, S^{\mathrm{MV}}_{\mathrm{err}}$, and the resulting $\Delta^{\mathrm{MV}}, \delta^{\mathrm{MV}}$ depend on the LM only through the answer distribution $\pi$.
\end{lemma}

\begin{proof}
With $w \equiv 1$, step (ii) collapses: $V_y = N_y$ deterministically, so $\nu = N_{\hat y}/m$ is a function of $\bN$ alone. Since $\bN \mid X$ is multinomial with parameter $\pi(X)$, the conditional law of $\nu$ given $X$ is determined by $\pi(X)$. Marginalising over $X$ gives $S^{\mathrm{MV}}_{\mathrm{cor}}, S^{\mathrm{MV}}_{\mathrm{err}}$ as functionals of the prompt distribution and $\pi$ alone.
\end{proof}


\paragraph{Why no clean additive decomposition exists.}
A natural ambition would be to write $\Delta(\lambda) = \Delta^{(\mathrm A)}(\lambda) + \Delta^{(\mathrm B)}(\lambda) + \Delta^{(\mathrm C)}(\lambda)$ as an additive split, isolating the contribution of each factor. This is not possible in general, because Factors A and B interact through $\bN$: the vote-level effect of a discriminative score $w \circ q$ depends on the cluster sizes the answer distribution produces on a given prompt. On a prompt where $\pi(X)$ concentrates on one or two answers, the score has very few wrong clusters to discriminate; on a prompt with diffuse wrong-mass, the same score has many small wrong clusters whose weights it can drive down. The same score therefore contributes \emph{different} amounts to $\Delta(\lambda)$ on different prompts, depending on a multinomial draw whose distribution depends on Factor A.

What \emph{is} rigorous is the conditional reading of the two-step structure:
\begin{itemize}
	\item \emph{Conditional on $\bN$ and $X$,} step (i) is fixed and the cluster weights $V_y$ depend only on the conditional weight laws $\{G_y\}_y$ (Factor B) and the choice of $w$ (Factor C). When weights on correct paths tend to exceed weights on wrong paths, $V_Y$ tends to dominate any $V_{y'}$ at matched cluster sizes, tilting $\nu$ in favour of the correct cluster.
	\item \emph{Marginally,} no such clean separation holds: the marginal distribution of $\nu$ on $\{\hat y = Y\}$ averages this conditional dominance against the multinomial geometry of $\bN$.
\end{itemize}

The three-factor reading is therefore best understood as identifying \emph{which channels each design choice opens or closes}, not as a numerical decomposition. Choosing the score is the act of recruiting Factor B; choosing the weight function is the act of converting that recruitment into vote-level mass shift through C; the LM and prompt distribution determine the Factor-A baseline that any aggregation rule inherits.

\paragraph{Two limits of Factor C.}
For the parametric family $w_\beta(q) = e^{\beta\, q}$, the role of $\beta$ is bracketed by two limits:
\begin{itemize}
	\item \textbf{$\beta = 0$:} $w_0 \equiv 1$ recovers majority voting; Lemma~\ref{lem:mv-factor-a} applies and Factor B is shut off entirely, regardless of how discriminative $q$ is on path-correctness.
	\item \textbf{$\beta \to \infty$:} for any fixed multiset $\mathbf n$, $V_y = \sum_{j:\,A_j = y} e^{\beta q_j} = e^{\beta\,\max_{j:\,A_j = y} q_j}\bigl(1 + o(1)\bigr)$, so $\hat y \to A_{j^\star}$ where $j^\star = \arg\max_j q_j$. The aggregator collapses to single-path argmax (best-of-$m$ selection), and the variance reduction of aggregating multiple paths within each cluster is forfeited.
\end{itemize}
The empirical observation in \S\ref{sec:experiments} that $\beta = 1$ outperforms $\beta = 5$ corresponds to this saturation: $\beta = 1$ is the moderate regime where Factor~B is amplified while $V_y$ still aggregates over multiple paths within each cluster, retaining Factor~A's variance reduction.

\subsection{Finite-sample concentration of $\hat A_c$}
\label{app:ca-bound}

This subsection upgrades the almost-sure statement of Corollary~\ref{cor:predictor} to a finite-sample bound, and is careful about the scope of the bound: a clean DKW-based statement requires \emph{iid} calibration data, while infinite-exchangeable calibration sequences yield a partial extension via de Finetti's theorem.

\paragraph{Setup.}
Recall from~\eqref{eq:ac-as-ratio} that $A_c(\lambda) = G(\lambda)/H(\lambda)$ and $\hat A_c(\lambda) = \hat G(\lambda)/\hat H(\lambda)$, with
\[
	G(\lambda) = \PP(\nu > \lambda,\,\hat y = Y),
	\qquad
	H(\lambda) = \PP(\nu > \lambda).
\]
Fix $\Delta_0 > 0$ and let $\Lambda_0 \subset [0,1]$ be a measurable set on which $\Delta(\lambda) \geq \Delta_0$. The \emph{answering margin} on $\Lambda_0$ is
\begin{equation}
	\label{eq:s0-defn}
	s_0 \;\coloneqq\; \inf_{\lambda \in \Lambda_0} H(\lambda),
\end{equation}
the smallest yield attained on the operating range. By~\eqref{eq:H-lower-bound}, $s_0 \geq p_v\,\Delta_0 > 0$.

\subsubsection{Iid calibration data: a finite-sample DKW-based bound}

The cleanest finite-sample statement assumes iid calibration tuples. This is strictly stronger than Assumption~\ref{assum:exch} (which only assumes exchangeability), but is the standard setting for finite-sample empirical-process bounds and matches our experimental protocol, where calibration prompts are obtained by uniform random subsampling of a benchmark.

\begin{theorem}[Finite-sample concentration of $\hat A_c$, iid case]
\label{thm:ac-concentration}
Suppose the calibration tuples $(X_i, Y_i, \bT_i)_{i=1}^n$ are iid. Fix $\delta \in (0,1)$ and let
\[
	\eps_n(\delta) \;\coloneqq\; \sqrt{\log(4/\delta)/(2n)}.
\]
If $n$ is large enough that $\eps_n(\delta) \leq s_0/2$ -- equivalently, $n \geq 2 \log(4/\delta)/s_0^2$ -- then with probability at least $1-\delta$,
\begin{equation}
	\label{eq:ac-finite-sample}
	\sup_{\lambda \in \Lambda_0} \bigl|\hat A_c(\lambda) - A_c(\lambda)\bigr| \;\leq\; \frac{4}{s_0}\,\eps_n(\delta) \;=\; \frac{4}{s_0}\sqrt{\frac{\log(4/\delta)}{2n}}.
\end{equation}
\end{theorem}

\begin{proof}
Set $\eps = \eps_n(\delta)$. Under iid calibration data, the DKW-Massart inequality~\cite{massart1990tight} gives
\[
	\PP\!\Bigl(\sup_{\lambda} |\hat G(\lambda) - G(\lambda)| > \eps\Bigr) \;\leq\; 2 e^{-2 n \eps^2},
	\qquad
	\PP\!\Bigl(\sup_{\lambda} |\hat H(\lambda) - H(\lambda)| > \eps\Bigr) \;\leq\; 2 e^{-2 n \eps^2}.
\]
By a union bound, the event
\[
	\cE \;\coloneqq\; \Bigl\{\sup_\lambda |\hat G - G| \leq \eps\Bigr\} \,\cap\, \Bigl\{\sup_\lambda |\hat H - H| \leq \eps\Bigr\}
\]
occurs with probability at least $1 - 4 e^{-2n\eps^2} = 1 - \delta$. On $\cE$, for $\lambda \in \Lambda_0$:
\[
	\hat H(\lambda) \;\geq\; H(\lambda) - \eps \;\geq\; s_0 - \eps \;\geq\; s_0/2,
\]
using $\eps \leq s_0/2$. The quotient identity~\eqref{eq:quotient-identity}, combined with $|A_c| \leq 1$, then yields
\[
	|\hat A_c(\lambda) - A_c(\lambda)|
	\;\leq\; \frac{|\hat G(\lambda) - G(\lambda)| + |\hat H(\lambda) - H(\lambda)|}{\hat H(\lambda)}
	\;\leq\; \frac{2 \eps}{s_0/2}
	\;=\; \frac{4 \eps}{s_0}.
\]
Taking the supremum over $\lambda \in \Lambda_0$ gives~\eqref{eq:ac-finite-sample}.
\end{proof}

\begin{remark}[Connection to the $\Delta_0$ formulation of Corollary~\ref{cor:predictor}]
Combining \Cref{thm:ac-concentration} with the lower bound $s_0 \geq p_v\,\Delta_0$ from~\eqref{eq:H-lower-bound} gives: under iid and the conditions of Corollary~\ref{cor:predictor}, with probability at least $1-\delta$,
\[
	\sup_{\lambda:\,\Delta(\lambda) \geq \Delta_0} |\hat A_c(\lambda) - A_c(\lambda)| \;\leq\; \frac{4}{p_v\,\Delta_0}\sqrt{\frac{\log(4/\delta)}{2n}} \;=\; O\bigl(n^{-1/2}\bigr),
\]
upgrading the almost-sure convergence in Corollary~\ref{cor:predictor} to an explicit $O(n^{-1/2})$ rate.
\end{remark}

\subsubsection{Extension to infinite exchangeable calibration sequences}

\Cref{thm:ac-concentration} assumes iid calibration data. The CRC validity of \Cref{thm:crc} only requires exchangeability, and it is natural to ask whether the concentration result extends to that setting. The answer is partial: under \emph{infinite} exchangeability, an analogous bound holds for a conditional target derived from de Finetti's theorem, but the marginal target picks up an additional non-vanishing term. We do not pursue finite-sample bounds for finite exchangeable sequences; the iid case above is the operative one for our experiments.

\paragraph{Conditional iid via de Finetti.}
Suppose $(X_i, Y_i, \bT_i)_{i \geq 1}$ is an infinite exchangeable sequence. By de Finetti's theorem~\cite{barber2024finetti}, there exists a random probability measure $\cP$ -- the \emph{directing measure} -- such that, conditional on $\cP$, the tuples $(X_i, Y_i, \bT_i)_{i \geq 1}$ are iid with law $\cP$. (In the iid case, $\cP$ is degenerate, equal almost surely to the data-generating distribution.) Define the conditional analogues
\[
	G_{\cP}(\lambda) \coloneqq \PP_{\cP}(\nu > \lambda,\,\hat y = Y),
	\quad
	H_{\cP}(\lambda) \coloneqq \PP_{\cP}(\nu > \lambda),
	\quad
	A_c^{\cP}(\lambda) \coloneqq \frac{G_{\cP}(\lambda)}{H_{\cP}(\lambda)}.
\]
The marginal probabilities satisfy $G(\lambda) = \EE[G_{\cP}(\lambda)]$ and $H(\lambda) = \EE[H_{\cP}(\lambda)]$, but the marginal accuracy $A_c$ is \emph{not} in general equal to $\EE[A_c^{\cP}]$, since $A_c^{\cP}$ is a ratio of two random functions.

\begin{remark}[De-Finetti-conditional certificate]
\label{rem:definetti-conditional}
Let $s_0^{\cP} \coloneqq \inf_{\lambda \in \Lambda_0} H_{\cP}(\lambda)$. Conditional on $\cP$, the tuples are iid, so the proof of \Cref{thm:ac-concentration} applies verbatim conditional on $\cP$: for any $\delta \in (0,1)$, on the event $\{s_0^{\cP} > 0,\ n \geq 2\log(4/\delta)/(s_0^{\cP})^2\}$, with conditional probability at least $1-\delta$ given $\cP$,
\[
	\sup_{\lambda \in \Lambda_0} \bigl|\hat A_c(\lambda) - A_c^{\cP}(\lambda)\bigr| \;\leq\; \frac{4}{s_0^{\cP}}\sqrt{\frac{\log(4/\delta)}{2n}}.
\]
This is a \emph{de-Finetti-conditional certificate}: the plug-in predictor $\hat A_c$ tracks the conditional accuracy $A_c^{\cP}$. Since $A_c^{\cP}$ is itself random, this is a different statement from convergence to the marginal $A_c$.
\end{remark}

The conditioning trick used here -- exchangeable calibration data become conditionally iid given the directing measure, after which standard concentration applies -- mirrors the strategy used by Vovk in the analysis of conditional validity for inductive conformal predictors~\cite{vovk2012conditional}.

\begin{remark}[The marginal target acquires a heterogeneity radius]
\label{rem:heterogeneity}
For the marginal target $A_c$, an additional term enters that does not shrink in $n$. Writing
\[
	|\hat G(\lambda) - G(\lambda)| \;\leq\; |\hat G(\lambda) - G_{\cP}(\lambda)| \;+\; |G_{\cP}(\lambda) - G(\lambda)|,
\]
the first term is the conditional-iid term controlled by DKW (Remark~\ref{rem:definetti-conditional}); the second is the heterogeneity gap of the directing measure. $G_{\cP}(\lambda)$ is a $[0,1]$-valued random variable with mean $G(\lambda)$, but boundedness alone does not make $|G_{\cP}(\lambda) - G(\lambda)|$ shrink with $n$. Defining
\[
	r_G(\lambda, \delta') \;\coloneqq\; \text{the } (1-\delta')\text{-quantile of } |G_{\cP}(\lambda) - G(\lambda)|,
\]
and analogously $r_H$, a union bound gives, with marginal probability at least $1-\delta$,
\[
	|\hat G(\lambda) - G(\lambda)| \;\leq\; \sqrt{\log(8/\delta)/(2n)} \;+\; r_G(\lambda, \delta/4),
\]
and similarly for $H$, so the resulting marginal bound on $|\hat A_c - A_c|$ contains $r_G + r_H$ as an additional, non-vanishing term. In the iid case, the directing measure is degenerate and $r_G \equiv r_H \equiv 0$, recovering \Cref{thm:ac-concentration} verbatim. Under more general infinite exchangeability the heterogeneity radii need not vanish, so the safe finite-sample claim is the de-Finetti-conditional certificate of \Cref{rem:definetti-conditional} rather than a marginal bound on $|\hat A_c - A_c|$.
\end{remark}

\subsection{Hazard rate and survival recovery}
\label{app:app-hazard}

The hazard rate admits a definition for both discrete and continuous conditional distributions of $\nu$, and the discrete case is the operative one for majority voting, where $\nu = k/m$ takes values in the finite grid $\{1/m, 2/m, . . . , 1\}$.

\paragraph{Discrete case} When $\nu$ is supported on a countable set, the conditional hazard at $\lambda$ is defined as

$$h_{\mathrm{cor}}(\lambda) \coloneqq \frac{P(\nu = \lambda \mid \nu \geq \lambda,\, \hat{y} = Y)}{1}, \qquad h_{\mathrm{err}}(\lambda) \coloneqq \frac{P(\nu = \lambda \mid \nu \geq \lambda,\, \hat{y} \neq Y)}{1},$$

that is, the conditional probability of exiting the retained set precisely at $\lambda$ given survival to $\lambda$. The separability gap $\delta(\lambda) = h_{\mathrm{err}}(\lambda) - h_{\mathrm{cor}}(\lambda)$ and the survival recovery identity

$$S_{\mathrm{cor}}(\lambda) = \prod_{\lambda' < \lambda} \bigl(1 - h_{\mathrm{cor}}(\lambda')\bigr), \qquad S_{\mathrm{err}}(\lambda) = \prod_{\lambda' < \lambda} \bigl(1 - h_{\mathrm{err}}(\lambda')\bigr),$$

hold without any continuity assumption, with products taken over the support points of $\nu$ below $\lambda$. In particular, both $\delta(\lambda)$ and the ratio $S_{\mathrm{err}}(\lambda)/S_{\mathrm{cor}}(\lambda)$ are directly estimable from calibration data in the discrete case. This makes strict separability a practically verifiable condition even under majority voting.

\paragraph{Continuous case} When the conditional distributions of $\nu$ given $\{\hat y = Y\}$ and $\{\hat y \neq Y\}$ admit densities $f_{\mathrm{cor}}$ and $f_{\mathrm{err}}$ on $(0, 1)$, the hazard rates reduce to the standard continuous form

$$h_{\mathrm{cor}}(\lambda) = \frac{f_{\mathrm{cor}}(\lambda)}{S_{\mathrm{cor}}(\lambda)} = -\frac{d}{d\lambda}\log S_{\mathrm{cor}}(\lambda), \qquad h_{\mathrm{err}}(\lambda) = \frac{f_{\mathrm{err}}(\lambda)}{S_{\mathrm{err}}(\lambda)} = -\frac{d}{d\lambda}\log S_{\mathrm{err}}(\lambda),$$

and integrating recovers the survival functions. The ratio identity driving the proof of Proposition~\ref{prop:pareto},

$$\frac{S_{\mathrm{err}}(\lambda)}{S_{\mathrm{cor}}(\lambda)} = \exp\!\left(-\int_0^\lambda \delta(u)\, du\right),$$

holds in the continuous case and motivates the integral condition used in Appendix~\ref{app:proof-ay}. The discrete analogue replaces the integral with the product form above.

\emph{Strict separability is local; separability is global.} In both cases, pointwise strict separability $\delta > 0$ at each support point or Lebesgue-a.e. — implies $S_{\mathrm{err}}/S_{\mathrm{cor}} \leq 1$, hence $\Delta(\lambda) > 0$. The converse fails: $\Delta(\lambda) > 0$ is compatible with $\delta$ dipping to zero or below on a set of measure zero (continuous case) or at isolated support points (discrete case), provided the cumulative product or integral remains favorable.

\subsection{Extension to Graded Correctness}
\label{app:graded-loss}

The binary correctness formulation extends to
bounded evaluator-defined losses. Let
$\ell(\hat y,Y)\in[0,1]$ denote the loss of an
answered prediction, and define
\[
L_i^\ell(\lambda)
=
\ell(\hat y_i,Y_i)\,
\mathbf{1}\{\nu_i>\lambda\}.
\]
Since this loss is bounded and non-increasing
in $\lambda$, applying the same CRC calibration
rule to $L_i^\ell$ yields
\[
\mathbb{E}\bigl[
L_{n+1}^\ell(\hat\lambda_\ell)
\bigr]\leq\alpha
\]
under exchangeability.

The separability characterization also admits
a graded analogue. Let $C=1-\ell(\hat y,Y)$.
For any $\lambda$ with $Y(\lambda)>0$,
\[
\mathbb{E}[C\mid\nu>\lambda]-\mathbb{E}[C]
=
\frac{
\operatorname{Cov}
(C,\mathbf{1}\{\nu>\lambda\})
}{Y(\lambda)}.
\]
Thus, positive covariance is necessary and
sufficient for abstention to improve mean
correctness among retained predictions.
For binary correctness, the numerator reduces to
$p_v(1-p_v)\Delta(\lambda)$, recovering
Theorem~\ref{thm:ca-gain}.

Empirical evaluation of graded correctness
requires task-specific loss definitions and
evaluation protocols.


\section{Additional Experiment Details}
\label{app:exp-details}

\subsection*{Model and Dataset Licenses}

All datasets and models are publicly available under permissive licenses. Datasets: GSM8K (MIT), MATH (MIT), HotpotQA (CC-BY-SA-4.0). Base models evaluated include variants from the following families: Qwen3 family (Apache 2.0); DeepSeek-R1-Distill-Qwen family (MIT); Gemma-4 family (Apache 2.0). Reward model, including variants: Skywork-Reward-V2 (Llama 3.1 Community License). All licenses permit academic research, benchmarking, and publication of results. No proprietary models were used.

\subsection{Experiment Grid}
\label{app:exp-grid}

Table~\ref{tab:exp-grid} lists the ten dataset--model configurations evaluated in the paper.
Each cell is constructed by sampling $m=16$ Chain-of-Thought paths per prompt, temperature reported; per-cell we use $n_\text{total}=1300$ prompts, of which $n_\text{cal}=200$ are reserved for calibration and the remainder are held out as test data.
For tables and figures whose targets are marginal expectations, we average over 20 random calibration--test splits; for operating-curve plots (Figures~\ref{fig:pred-vs-obs} and~\ref{fig:frontiers}) we use a single fixed split, with fixed seed for reproducibility.
The pair (HotpotQA, DeepSeek-R1-Distill-Qwen-1.5B) is omitted because the model's tokenizer cannot correctly decode the reasoning paths generated for HotpotQA.
Temperatures are chosen based on each model’s recommended settings, while allowing minor adjustments to account for generation variability.

\begin{table}[h]
	\centering
	\caption{The ten dataset--model configurations used in the paper. MV is the unweighted majority-vote accuracy at $m=16$.}
	\label{tab:exp-grid}
	\small
	\begin{tabular}{rllcccc}
		\toprule
		\textbf{\#} & \textbf{Dataset} & \textbf{Model}                & \textbf{Temperature}  & \textbf{$n_\text{total}$} & \textbf{$m$} & \textbf{MV} \\
		\midrule
		1           & GSM8K            & Qwen3-1.7B                    & 0.8  & 1300                      & 16           & 0.730          \\
		2           & GSM8K            & DeepSeek-R1-Distill-Qwen-1.5B & 0.8  & 1300                      & 16           & 0.822          \\
		3           & GSM8K            & Gemma-4-E4B                   & 1.0  & 1300                      & 16           & 0.901          \\
		4           & MATH             & Qwen3-1.7B                    & 0.8  & 1300                      & 16           & 0.712          \\
		5           & MATH             & DeepSeek-R1-Distill-Qwen-1.5B & 0.8  & 1300                      & 16           & 0.603          \\
		6           & MATH             & Gemma-4-E4B                   & 1.0  & 1300                      & 16           & 0.583          \\
		7           & MATH-Hard        & Qwen3-8B                      & 0.8  & 1300                      & 16           & 0.537          \\
		8           & HotpotQA         & Qwen3-1.7B                    & 0.75  & 1300                      & 16           & 0.519          \\
		9           & HotpotQA         & Gemma-4-E4B                   & 1.0  & 1300                      & 16           & 0.724          \\
		10          & HotpotQA         & Qwen3-8B                      & 0.65  & 1300                      & 16           & 0.709          \\
		\bottomrule
	\end{tabular}
\end{table}

\subsection{Model Details}
\label{app:model-details}

\paragraph{Reasoning models (path generation).}
Model identifiers and parameter counts are listed in Table~\ref{tab:reasoning-models}.

\begin{table}[h]
	\centering
	\caption{Reasoning models used for path generation.}
	\label{tab:reasoning-models}
	\small
	\begin{tabular}{lll}
		\toprule
		\textbf{Short name} & \textbf{Identifier}                    & \textbf{Params}      \\
		\midrule
		Qwen3-1.7B          & \texttt{Qwen/Qwen3-1.7B}                           & 1.7B                 \\
		R1-Qwen-1.5B        & \texttt{deepseek-ai/DeepSeek-R1-Distill-Qwen-1.5B} & 1.5B                 \\
		Gemma-4-E4B         & \texttt{google/gemma-4-E4B-it}                     & $\sim$4B (effective) \\
		Qwen3-8B            & \texttt{Qwen/Qwen3-8B}                             & 8B                   \\
		\bottomrule
	\end{tabular}
\end{table}

\paragraph{Scorer models.}
Reward, Judge, and Perplexity scores are produced by a separate set of models. Reward scores use Skywork's V2 family; Judge scores use a frontier-class instruction-tuned model that classifies a candidate answer as correct or incorrect; Perplexity is computed under the same model that generated the path. We report two instantiations per family in the appendix tables to assess robustness to the specific scorer choice.

\begin{table}[h]
	\centering
	\caption{Scorer models, by family.}
	\label{tab:scorer-models}
	\small
	\begin{tabular}{lll}
		\toprule
		\textbf{Family} & \textbf{Name}          & \textbf{HuggingFace identifier}                                         \\
		\midrule
		\multirow{3}{*}{Reward}
		                & Skywork-Llama-8B~\citep{liu2025skywork}     & \texttt{Skywork/Skywork-Reward-V2-Llama-3.1-8B}                         \\
		                & Skywork-Qwen-8B (appendix)   & \texttt{Skywork/Skywork-Reward-V2-Qwen3-8B}                             \\
		                & Skywork-Llama-3B (appendix)  & \texttt{Skywork/Skywork-Reward-V2-Llama-3.2-3B}                         \\
		\midrule
		\multirow{2}{*}{Judge}
		                & Gemma-26B                    & \texttt{google/gemma-4-26B-A4B-it}                                      \\
		                & Qwen3-30B (appendix)         & \texttt{Qwen/Qwen3-30B-A3B-Instruct-2507}                               \\
		                & R1-Qwen-14B (math, appendix) & \texttt{deepseek-ai/DeepSeek-R1-Distill-Qwen-14B}                       \\
		                & Qwen3-8B (QA, appendix)      & \texttt{Qwen/Qwen3-8B}                                                  \\
		\midrule
		Perplexity      & Generating model             & (same as path-generation model in each row of Table~\ref{tab:exp-grid}) \\
		\bottomrule
	\end{tabular}
\end{table}

\subsection{Score Computation Details}
\label{app:score-details}

\paragraph{Reward.}
Each candidate path is evaluated with a Skywork reward model conditioned on the prompt and the full path text. The model returns a scalar quality score $r \in \mathbb{R}$ (higher is better).

\paragraph{Judge.}
A judge model is presented with the question and a candidate path and asked to assess correctness via a binary Yes/No response. The score is defined as $s = \log p(\texttt{Y})$, i.e., the log-probability assigned to a correct judgment. The full prompt template is provided in Appendix~\ref{app:prompts}.

\paragraph{Perplexity.}
Perplexity measures the average per-token uncertainty of the generating model over a sequence. For a path $T = (T_1, \dots, T_{|T|})$ conditioned on prompt $X$, we use its log form:
$$s = \frac{1}{|T|}\sum_{t=1}^{|T|} \log p(T_t \mid T_{<t}, X).$$
Higher values (lower perplexity) indicate greater model confidence in the generated path. A second variant, \emph{Perplexity-Std}, uses the standard deviation of token-level log-probabilities, $s = \mathrm{std}\{\log p(T_t \mid T_{<t}, X)\}$, and captures the variability of token-level confidence rather than its mean; it is reported in the ablation.

\paragraph{Self-Consistency (SC, ablation).}
Self-consistency scores a path by how similar it is to the other paths in the pool. We operationalize similarity via the Jaccard coefficient between the unigram token sets of two paths: $\mathrm{Jaccard}(A, B) = |A \cap B| / |A \cup B|$. The score of path $T^{(j)}$ is its mean pairwise similarity to the remaining $m-1$ paths,
$$s_j = \frac{1}{m-1}\sum_{k \neq j} \mathrm{Jaccard}(T^{(j)}, T^{(k)}),$$
rewarding paths that are representative of the pool.

\paragraph{Weighting.}
Path weights are $w_j = \exp(q_j)$, so higher-scoring paths receive greater influence. The exponential weighting scheme is justified by the ablation in Appendix~\ref{app:abla-weight}.

\paragraph{Per-cell scorer.}
Each row of Table~\ref{tab:exp-grid} uses one primary Reward scorer (Skywork-Llama-8B) and one primary Judge scorer (Gemma-26B) for the main-paper figures.
Additional instantiations are reported in Table~\ref{tab:scorer-models} to assess robustness to scorer choice.

\subsection{Answer Extraction and Matching}
\label{app:answer-extraction}

\paragraph{Mathematical answers.}
We extract the final answer from the last \verb|\boxed{...}| expression in the path, with fallbacks to phrases of the form ``The answer is'' or the trailing numeric expression when no boxed span is present. Correctness is determined via the evaluation harness’s equivalence check~\citep{yang2024qwen2, liao2024mario}, which handles common surface-form variations (e.g., $\tfrac{1}{2}$ vs.\ $0.5$, $\sqrt{2}$ vs.\ $2^{1/2}$).
For the additional Qwen3-14B experiment of Appendix~\ref{app:cr} we additionally report labels obtained after stripping a trailing \texttt{</think>} tag before extraction; Appendix~\ref{app:cr-extraction} quantifies the effect of this change on every math configuration.

\paragraph{Open-domain QA.}
For HotpotQA, we extract the answer from the last \verb|<answer>...</answer>| span in the path, with a fallback to the trailing sentence. Correctness is determined by exact match after standard normalization (lowercasing, removal of articles and punctuation, whitespace collapsing), yielding the binary correctness label required by our framework.

\subsection{Prompts Used}
\label{app:prompts}
All judge prompts present only the question and the candidate reasoning path, deliberately omitting any reference passages or background context. This design reflects a key distinction in what each scoring component is meant to evaluate: the judge is tasked with assessing the \emph{quality of reasoning}, whether the chain-of-thought is logically coherent and leads to a defensible answer, rather than verifying factual grounding against an external source. Including retrieved passages would shift the judge's role toward factuality and summarization assessment, which is orthogonal to our objective. Providing background context would also substantially inflate the judge's input length at inference time, introducing unnecessary computational overhead across the large number of scoring calls our framework requires. We find this separation of concerns to be both principled and practically efficient: factual grounding is implicitly encoded in the path itself during generation (where the context \emph{is} provided), while the judge focuses solely on the inferential validity of the resulting reasoning chain.

\paragraph{Path-generation prompts.}

\textit{Math (applies to GSM8K, MATH, MATH-Hard cells):}
\begin{promptbox}
[System]
Solve the problem using step-by-step reasoning.
Write the reasoning as a clean logical sequence.
Do not include self-corrections, backtracking, or phrases like
'wait', 'but', 'actually', or 'let me reconsider'.

[User]
{question}
\end{promptbox}

\textit{QA (applies to HotpotQA cells):}
\begin{promptbox}
[System]
You are a careful reading-comprehension assistant.
You will be given several reference passages followed by a question.
Reason step-by-step using only the information in the passages.
You must write down your thinking even it's short and simple.
Write your reasoning as a clean, forward-only logical chain.
After finishing your reasoning, state your final answer concisely.
Format your response exactly as:
<think>
[your step-by-step reasoning]
</think>
<answer>[your final answer]</answer>

[User]
Passages:
{context}

Question: {question}

Please reason through the passages and provide your answer.
\end{promptbox}

\paragraph{Judge prompt.}

\textit{Math judge:}
\begin{promptbox}
You will assess a candidate reasoning for a math word problem.
If the reasoning is likely correct and leads to the right final
answer, answer 'Y'. Otherwise answer 'N'.
Output only one letter.

Question: {question}

Candidate reasoning:
{draft}

Assessment (Y/N):
\end{promptbox}

\textit{QA judge:}
\begin{promptbox}
You are an expert evaluator for multi-hop question answering.
Read the question and the candidate reasoning below, then decide:
  * Answer 'Y' if the reasoning is logically sound and leads to a
    correct, well-supported answer.
  * Answer 'N' otherwise.
Output exactly one letter - nothing else.

Question: {question}

Candidate reasoning:
{draft}

Assessment (Y/N):
\end{promptbox}

\section{Additional Experiment Results}
\label{app:exp-results}

\subsection{Main Results Table}
\label{app:main-table}

Table~\ref{tab:ce-mathhard}-\ref{tab:main-hotpotqa} report the headline metrics --- confident-error rate, conditional accuracy $A_c$ and yield $Y$ --- for all ten configurations and four policies (greedy decoding, oracle accuracy --- the fraction of questions with at least one correct path, majority vote with CRC abstention on its vote fraction (MV + CRC), our calibrated policy at $\alpha\in\{0.05, 0.10\}$).
Numbers are mean $\pm$ std over twenty random calibration--test splits at $n_\text{cal}=200$, $m=16$.
Mean realized confident-error rates are close to the target and any above-target means are within one standard deviation, consistent with the marginal CRC guarantee and finite test-split variability.
Entries marked $\dagger$ in Tables \ref{tab:ce-math}-\ref{tab:ce-hotpotqa} correspond to MV + CRC, which cannot achieve the target confident-error rate at $\alpha=0.05$ on these harder configurations, underscoring that score-based weighting is necessary for reliable threshold calibration.

\begin{table}[h]
  \centering
  \caption{MATH-Hard empirical confident-error rate ($n_\text{cal}=200$, $m=16$, 20 splits; target $\alpha$ shown in column header). Values are mean\,$\pm$\,std. 
  }
  \label{tab:ce-mathhard}
  \small
  \begin{tabular}{llcc}
    \toprule
    \textbf{Model} & \textbf{Method} &
     $\boldsymbol{\alpha=0.10}$ &
     $\boldsymbol{\alpha=0.05}$ \\
    \midrule
    \multirow{8}{*}{Qwen3-8B} & \textit{MV + CRC} & $0.079{\scriptstyle\,\pm\,}0.022$ & $0.040{\scriptstyle\,\pm\,}0.012$ \\
    \cmidrule(lr){2-4}
     & \quad R-L8B (Skywork-Reward-V2-Llama-3.1-8B) & $0.090{\scriptstyle\,\pm\,}0.024$ & $0.044{\scriptstyle\,\pm\,}0.015$ \\
     & \quad R-Q8B (Skywork-Qwen-8B) & $0.092{\scriptstyle\,\pm\,}0.024$ & $0.043{\scriptstyle\,\pm\,}0.015$ \\
     & \quad J-G26B (Gemma-26B) & $0.085{\scriptstyle\,\pm\,}0.021$ & $0.042{\scriptstyle\,\pm\,}0.015$ \\
     & \quad J-Q30B (Qwen3-30B) & $0.088{\scriptstyle\,\pm\,}0.016$ & $0.039{\scriptstyle\,\pm\,}0.011$ \\
     & \quad J-R1Q14B (R1-Qwen-14B) & $0.086{\scriptstyle\,\pm\,}0.024$ & $0.044{\scriptstyle\,\pm\,}0.017$ \\
     & \quad Ppl & $0.091{\scriptstyle\,\pm\,}0.018$ & $0.044{\scriptstyle\,\pm\,}0.017$ \\
     & \quad Ppl-Std & $0.093{\scriptstyle\,\pm\,}0.023$ & $0.042{\scriptstyle\,\pm\,}0.015$ \\
    \bottomrule
  \end{tabular}
\end{table}

\begin{table}[p]
  \centering
  \caption{GSM8K empirical confident-error rate ($n_\text{cal}=200$, $m=16$, 20 splits; target $\alpha$ shown in column header). Values are mean\,$\pm$\,std. 
  }
  \label{tab:ce-gsm8k}
  \small
  \begin{tabular}{llcc}
    \toprule
    \textbf{Model} & \textbf{Method} &
     $\boldsymbol{\alpha=0.10}$ &
     $\boldsymbol{\alpha=0.05}$ \\
    \midrule
    \multirow{8}{*}{Qwen3-1.7B} & \textit{MV + CRC} & $0.080{\scriptstyle\,\pm\,}0.017$ & $0.045{\scriptstyle\,\pm\,}0.013$ \\
    \cmidrule(lr){2-4}
     & \quad R-L8B & $0.096{\scriptstyle\,\pm\,}0.018$ & $0.046{\scriptstyle\,\pm\,}0.013$ \\
     & \quad R-Q8B & $0.097{\scriptstyle\,\pm\,}0.025$ & $0.047{\scriptstyle\,\pm\,}0.016$ \\
     & \quad J-G26B & $0.098{\scriptstyle\,\pm\,}0.016$ & $0.047{\scriptstyle\,\pm\,}0.014$ \\
     & \quad J-Q30B & $0.091{\scriptstyle\,\pm\,}0.020$ & $0.048{\scriptstyle\,\pm\,}0.012$ \\
     & \quad J-R1Q14B & $0.097{\scriptstyle\,\pm\,}0.016$ & $0.048{\scriptstyle\,\pm\,}0.013$ \\
     & \quad Ppl & $0.103{\scriptstyle\,\pm\,}0.020$ & $0.048{\scriptstyle\,\pm\,}0.011$ \\
     & \quad Ppl-Std & $0.098{\scriptstyle\,\pm\,}0.019$ & $0.050{\scriptstyle\,\pm\,}0.010$ \\
    \midrule
    \multirow{8}{*}{R1-Qwen-1.5B} & \textit{MV + CRC} & $0.088{\scriptstyle\,\pm\,}0.026$ & $0.039{\scriptstyle\,\pm\,}0.017$ \\
    \cmidrule(lr){2-4}
     & \quad R-L8B & $0.091{\scriptstyle\,\pm\,}0.018$ & $0.041{\scriptstyle\,\pm\,}0.017$ \\
     & \quad R-Q8B & $0.087{\scriptstyle\,\pm\,}0.018$ & $0.041{\scriptstyle\,\pm\,}0.015$ \\
     & \quad J-G26B & $0.096{\scriptstyle\,\pm\,}0.023$ & $0.043{\scriptstyle\,\pm\,}0.014$ \\
     & \quad J-Q30B & $0.094{\scriptstyle\,\pm\,}0.022$ & $0.039{\scriptstyle\,\pm\,}0.016$ \\
     & \quad J-R1Q14B & $0.093{\scriptstyle\,\pm\,}0.020$ & $0.042{\scriptstyle\,\pm\,}0.016$ \\
     & \quad Ppl & $0.092{\scriptstyle\,\pm\,}0.018$ & $0.040{\scriptstyle\,\pm\,}0.013$ \\
     & \quad Ppl-Std & $0.090{\scriptstyle\,\pm\,}0.020$ & $0.041{\scriptstyle\,\pm\,}0.014$ \\
    \midrule
    \multirow{8}{*}{Gemma-4-E4B} & \textit{MV + CRC} & $0.084{\scriptstyle\,\pm\,}0.018$ & $0.040{\scriptstyle\,\pm\,}0.017$ \\
    \cmidrule(lr){2-4}
     & \quad R-L8B & $0.085{\scriptstyle\,\pm\,}0.014$ & $0.042{\scriptstyle\,\pm\,}0.019$ \\
     & \quad R-Q8B & $0.076{\scriptstyle\,\pm\,}0.014$ & $0.043{\scriptstyle\,\pm\,}0.019$ \\
     & \quad J-G26B & $0.088{\scriptstyle\,\pm\,}0.020$ & $0.040{\scriptstyle\,\pm\,}0.014$ \\
     & \quad J-Q30B & $0.086{\scriptstyle\,\pm\,}0.017$ & $0.043{\scriptstyle\,\pm\,}0.017$ \\
     & \quad J-R1Q14B & $0.093{\scriptstyle\,\pm\,}0.030$ & $0.044{\scriptstyle\,\pm\,}0.020$ \\
     & \quad Ppl & $0.091{\scriptstyle\,\pm\,}0.020$ & $0.045{\scriptstyle\,\pm\,}0.015$ \\
     & \quad Ppl-Std & $0.085{\scriptstyle\,\pm\,}0.023$ & $0.044{\scriptstyle\,\pm\,}0.020$ \\
    \bottomrule
  \end{tabular}
\end{table}

\begin{table}[p]
  \centering
  \caption{MATH empirical confident-error rate ($n_\text{cal}=200$, $m=16$, 20 splits; target $\alpha$ shown in column header). Values are mean\,$\pm$\,std. Entries marked $\dagger$ have yield${}<0.01$ and are omitted.}
  \label{tab:ce-math}
  \small
  \begin{tabular}{llcc}
    \toprule
    \textbf{Model} & \textbf{Method} &
     $\boldsymbol{\alpha=0.10}$ &
     $\boldsymbol{\alpha=0.05}$ \\
    \midrule
    \multirow{8}{*}{Qwen3-1.7B} & \textit{MV + CRC} & $0.092{\scriptstyle\,\pm\,}0.027$ & $0.045{\scriptstyle\,\pm\,}0.014$ \\
    \cmidrule(lr){2-4}
     & \quad R-L8B & $0.095{\scriptstyle\,\pm\,}0.021$ & $0.048{\scriptstyle\,\pm\,}0.012$ \\
     & \quad R-Q8B & $0.097{\scriptstyle\,\pm\,}0.019$ & $0.052{\scriptstyle\,\pm\,}0.016$ \\
     & \quad J-G26B & $0.102{\scriptstyle\,\pm\,}0.020$ & $0.052{\scriptstyle\,\pm\,}0.018$ \\
     & \quad J-Q30B & $0.100{\scriptstyle\,\pm\,}0.028$ & $0.050{\scriptstyle\,\pm\,}0.016$ \\
     & \quad J-R1Q14B & $0.104{\scriptstyle\,\pm\,}0.027$ & $0.051{\scriptstyle\,\pm\,}0.015$ \\
     & \quad Ppl & $0.102{\scriptstyle\,\pm\,}0.031$ & $0.053{\scriptstyle\,\pm\,}0.019$ \\
     & \quad Ppl-Std & $0.100{\scriptstyle\,\pm\,}0.024$ & $0.047{\scriptstyle\,\pm\,}0.017$ \\
    \midrule
    \multirow{8}{*}{R1-Qwen-1.5B} & \textit{MV + CRC} & $0.090{\scriptstyle\,\pm\,}0.019$ & $0.042{\scriptstyle\,\pm\,}0.015$ \\
    \cmidrule(lr){2-4}
     & \quad R-L8B & $0.097{\scriptstyle\,\pm\,}0.017$ & $0.046{\scriptstyle\,\pm\,}0.011$ \\
     & \quad R-Q8B & $0.099{\scriptstyle\,\pm\,}0.022$ & $0.048{\scriptstyle\,\pm\,}0.013$ \\
     & \quad J-G26B & $0.101{\scriptstyle\,\pm\,}0.020$ & $0.049{\scriptstyle\,\pm\,}0.016$ \\
     & \quad J-Q30B & $0.095{\scriptstyle\,\pm\,}0.015$ & $0.050{\scriptstyle\,\pm\,}0.013$ \\
     & \quad J-R1Q14B & $0.096{\scriptstyle\,\pm\,}0.015$ & $0.050{\scriptstyle\,\pm\,}0.017$ \\
     & \quad Ppl & $0.096{\scriptstyle\,\pm\,}0.018$ & $0.047{\scriptstyle\,\pm\,}0.011$ \\
     & \quad Ppl-Std & $0.096{\scriptstyle\,\pm\,}0.016$ & $0.049{\scriptstyle\,\pm\,}0.011$ \\
    \midrule
    \multirow{8}{*}{Gemma-4-E4B} & \textit{MV + CRC} & $0.066{\scriptstyle\,\pm\,}0.039$ & $\dagger$ \\
    \cmidrule(lr){2-4}
     & \quad R-L8B & $0.098{\scriptstyle\,\pm\,}0.020$ & $0.028{\scriptstyle\,\pm\,}0.026$ \\
     & \quad R-Q8B & $0.095{\scriptstyle\,\pm\,}0.019$ & $0.027{\scriptstyle\,\pm\,}0.019$ \\
     & \quad J-G26B & $0.093{\scriptstyle\,\pm\,}0.023$ & $0.020{\scriptstyle\,\pm\,}0.014$ \\
     & \quad J-Q30B & $0.090{\scriptstyle\,\pm\,}0.023$ & $0.008{\scriptstyle\,\pm\,}0.002$ \\
     & \quad J-R1Q14B & $0.095{\scriptstyle\,\pm\,}0.018$ & $0.017{\scriptstyle\,\pm\,}0.020$ \\
     & \quad Ppl & $0.094{\scriptstyle\,\pm\,}0.017$ & $0.012{\scriptstyle\,\pm\,}0.016$ \\
     & \quad Ppl-Std & $0.094{\scriptstyle\,\pm\,}0.017$ & $0.026{\scriptstyle\,\pm\,}0.024$ \\
    \bottomrule
  \end{tabular}
\end{table}

\begin{table}[p]
  \centering
  \caption{HotpotQA empirical confident-error rate ($n_\text{cal}=200$, $m=16$, 20 splits; target $\alpha$ shown in column header). Values are mean\,$\pm$\,std. Entries marked $\dagger$ have yield${}<0.01$ and are omitted.}
  \label{tab:ce-hotpotqa}
  \small
  \begin{tabular}{llcc}
    \toprule
    \textbf{Model} & \textbf{Method} &
     $\boldsymbol{\alpha=0.10}$ &
     $\boldsymbol{\alpha=0.05}$ \\
    \midrule
    \multirow{8}{*}{Qwen3-1.7B} & \textit{MV + CRC} & $0.078{\scriptstyle\,\pm\,}0.021$ & $0.036{\scriptstyle\,\pm\,}0.014$ \\
    \cmidrule(lr){2-4}
     & \quad R-L8B & $0.091{\scriptstyle\,\pm\,}0.024$ & $0.048{\scriptstyle\,\pm\,}0.018$ \\
     & \quad R-Q8B & $0.094{\scriptstyle\,\pm\,}0.018$ & $0.046{\scriptstyle\,\pm\,}0.016$ \\
     & \quad J-G26B & $0.091{\scriptstyle\,\pm\,}0.022$ & $0.046{\scriptstyle\,\pm\,}0.016$ \\
     & \quad J-Q30B & $0.087{\scriptstyle\,\pm\,}0.020$ & $0.042{\scriptstyle\,\pm\,}0.015$ \\
     & \quad J-Q8B (Qwen3-8B) & $0.084{\scriptstyle\,\pm\,}0.020$ & $0.039{\scriptstyle\,\pm\,}0.018$ \\
     & \quad Ppl & $0.098{\scriptstyle\,\pm\,}0.029$ & $0.043{\scriptstyle\,\pm\,}0.015$ \\
     & \quad Ppl-Std & $0.098{\scriptstyle\,\pm\,}0.029$ & $0.042{\scriptstyle\,\pm\,}0.017$ \\
    \midrule
    \multirow{8}{*}{Gemma-4-E4B} & \textit{MV + CRC} & $0.058{\scriptstyle\,\pm\,}0.055$ & $\dagger$ \\
    \cmidrule(lr){2-4}
     & \quad R-L8B & $0.094{\scriptstyle\,\pm\,}0.024$ & $0.026{\scriptstyle\,\pm\,}0.022$ \\
     & \quad R-Q8B & $0.092{\scriptstyle\,\pm\,}0.027$ & $0.024{\scriptstyle\,\pm\,}0.015$ \\
     & \quad J-G26B & $0.088{\scriptstyle\,\pm\,}0.027$ & $0.031{\scriptstyle\,\pm\,}0.022$ \\
     & \quad J-Q30B & $0.095{\scriptstyle\,\pm\,}0.022$ & $0.031{\scriptstyle\,\pm\,}0.023$ \\
     & \quad J-Q8B & $0.087{\scriptstyle\,\pm\,}0.032$ & $0.006{\scriptstyle\,\pm\,}0.001$ \\
     & \quad Ppl & $0.093{\scriptstyle\,\pm\,}0.024$ & $0.020{\scriptstyle\,\pm\,}0.016$ \\
     & \quad Ppl-Std & $0.093{\scriptstyle\,\pm\,}0.025$ & $0.027{\scriptstyle\,\pm\,}0.017$ \\
    \midrule
    \multirow{8}{*}{Qwen3-8B} & \textit{MV + CRC} & $0.068{\scriptstyle\,\pm\,}0.053$ & $\dagger$ \\
    \cmidrule(lr){2-4}
     & \quad R-L8B & $0.099{\scriptstyle\,\pm\,}0.023$ & $0.023{\scriptstyle\,\pm\,}0.020$ \\
     & \quad R-Q8B & $0.102{\scriptstyle\,\pm\,}0.025$ & $0.025{\scriptstyle\,\pm\,}0.012$ \\
     & \quad J-G26B & $0.102{\scriptstyle\,\pm\,}0.021$ & $0.015{\scriptstyle\,\pm\,}0.001$ \\
     & \quad J-Q30B & $0.094{\scriptstyle\,\pm\,}0.035$ & $0.005{\scriptstyle\,\pm\,}0.001$ \\
     & \quad J-Q8B & $0.075{\scriptstyle\,\pm\,}0.051$ & $0.001{\scriptstyle\,\pm\,}0.001$ \\
     & \quad Ppl & $0.099{\scriptstyle\,\pm\,}0.022$ & $0.018{\scriptstyle\,\pm\,}0.012$ \\
     & \quad Ppl-Std & $0.099{\scriptstyle\,\pm\,}0.028$ & $0.024{\scriptstyle\,\pm\,}0.018$ \\
    \bottomrule
  \end{tabular}
\end{table}

\begin{table}[p]
  \centering
  \caption{GSM8K results ($n_\text{cal}=200$, $m=16$, 20 splits). $p_v$: score-weighted MV accuracy; best: score-based best-of-m.}
  \label{tab:main-gsm8k}
  \small
  \begin{tabular}{llccccccc}
    \toprule
    \textbf{Model} & \textbf{Method} & \textbf{$p_v$} & \textbf{best} &
    \multicolumn{2}{c}{$\boldsymbol{\alpha=0.10}$} &
    \multicolumn{2}{c}{$\boldsymbol{\alpha=0.05}$} \\
    \cmidrule(lr){5-6}\cmidrule(lr){7-8}
    & & & & $A_c$ & $Y$ & $A_c$ & $Y$ \\
    \midrule
    \multirow{10}{*}{Qwen3-1.7B} & \textit{Greedy} & --- & --- & 0.583 & 1.000 & 0.583 & 1.000 \\
     & \textit{Oracle} & --- & 0.899 & 0.899 & 1.000 & 0.899 & 1.000 \\
     & \textit{MV + CRC} & 0.730 & --- & $0.888{\scriptstyle\,\pm\,}0.020$ & $0.713{\scriptstyle\,\pm\,}0.031$ & $0.929{\scriptstyle\,\pm\,}0.016$ & $0.617{\scriptstyle\,\pm\,}0.049$ \\
    \cmidrule(lr){2-8}
     & \quad R-L8B & 0.740 & 0.655 & $0.872{\scriptstyle\,\pm\,}0.020$ & $0.747{\scriptstyle\,\pm\,}0.031$ & $0.928{\scriptstyle\,\pm\,}0.017$ & $0.629{\scriptstyle\,\pm\,}0.036$ \\
     & \quad R-Q8B & 0.757 & 0.697 & $0.877{\scriptstyle\,\pm\,}0.024$ & $0.777{\scriptstyle\,\pm\,}0.048$ & $0.929{\scriptstyle\,\pm\,}0.020$ & $0.645{\scriptstyle\,\pm\,}0.044$ \\
     & \quad J-G26B & 0.723 & 0.599 & $0.869{\scriptstyle\,\pm\,}0.017$ & $0.741{\scriptstyle\,\pm\,}0.028$ & $0.925{\scriptstyle\,\pm\,}0.016$ & $0.621{\scriptstyle\,\pm\,}0.049$ \\
     & \quad J-Q30B & 0.736 & 0.591 & $0.877{\scriptstyle\,\pm\,}0.021$ & $0.735{\scriptstyle\,\pm\,}0.037$ & $0.924{\scriptstyle\,\pm\,}0.015$ & $0.625{\scriptstyle\,\pm\,}0.043$ \\
     & \quad J-R1Q14B & 0.737 & 0.612 & $0.872{\scriptstyle\,\pm\,}0.017$ & $0.748{\scriptstyle\,\pm\,}0.031$ & $0.923{\scriptstyle\,\pm\,}0.015$ & $0.622{\scriptstyle\,\pm\,}0.040$ \\
     & \quad Ppl & 0.733 & 0.649 & $0.866{\scriptstyle\,\pm\,}0.020$ & $0.758{\scriptstyle\,\pm\,}0.041$ & $0.923{\scriptstyle\,\pm\,}0.013$ & $0.617{\scriptstyle\,\pm\,}0.037$ \\
     & \quad Ppl-Std & 0.698 & 0.574 & $0.861{\scriptstyle\,\pm\,}0.020$ & $0.702{\scriptstyle\,\pm\,}0.036$ & $0.912{\scriptstyle\,\pm\,}0.013$ & $0.561{\scriptstyle\,\pm\,}0.036$ \\
    \midrule
    \multirow{10}{*}{R1-Qwen-1.5B} & \textit{Greedy} & --- & --- & 0.624 & 1.000 & 0.624 & 1.000 \\
     & \textit{Oracle} & --- & 0.956 & 0.956 & 1.000 & 0.956 & 1.000 \\
     & \textit{MV + CRC} & 0.822 & --- & $0.896{\scriptstyle\,\pm\,}0.023$ & $0.840{\scriptstyle\,\pm\,}0.051$ & $0.942{\scriptstyle\,\pm\,}0.018$ & $0.647{\scriptstyle\,\pm\,}0.083$ \\
    \cmidrule(lr){2-8}
     & \quad R-L8B & 0.873 & 0.842 & $0.904{\scriptstyle\,\pm\,}0.016$ & $0.944{\scriptstyle\,\pm\,}0.030$ & $0.950{\scriptstyle\,\pm\,}0.017$ & $0.797{\scriptstyle\,\pm\,}0.077$ \\
     & \quad R-Q8B & 0.883 & 0.883 & $0.909{\scriptstyle\,\pm\,}0.016$ & $0.954{\scriptstyle\,\pm\,}0.026$ & $0.951{\scriptstyle\,\pm\,}0.015$ & $0.818{\scriptstyle\,\pm\,}0.066$ \\
     & \quad J-G26B & 0.852 & 0.829 & $0.897{\scriptstyle\,\pm\,}0.020$ & $0.920{\scriptstyle\,\pm\,}0.037$ & $0.943{\scriptstyle\,\pm\,}0.014$ & $0.750{\scriptstyle\,\pm\,}0.065$ \\
     & \quad J-Q30B & 0.838 & 0.703 & $0.894{\scriptstyle\,\pm\,}0.019$ & $0.882{\scriptstyle\,\pm\,}0.042$ & $0.944{\scriptstyle\,\pm\,}0.017$ & $0.687{\scriptstyle\,\pm\,}0.078$ \\
     & \quad J-R1Q14B & 0.828 & 0.782 & $0.893{\scriptstyle\,\pm\,}0.017$ & $0.864{\scriptstyle\,\pm\,}0.044$ & $0.941{\scriptstyle\,\pm\,}0.017$ & $0.695{\scriptstyle\,\pm\,}0.070$ \\
     & \quad Ppl & 0.826 & 0.779 & $0.896{\scriptstyle\,\pm\,}0.016$ & $0.873{\scriptstyle\,\pm\,}0.037$ & $0.944{\scriptstyle\,\pm\,}0.014$ & $0.688{\scriptstyle\,\pm\,}0.067$ \\
     & \quad Ppl-Std & 0.815 & 0.725 & $0.894{\scriptstyle\,\pm\,}0.018$ & $0.842{\scriptstyle\,\pm\,}0.041$ & $0.939{\scriptstyle\,\pm\,}0.015$ & $0.659{\scriptstyle\,\pm\,}0.060$ \\
    \midrule
    \multirow{10}{*}{Gemma-4-E4B} & \textit{Greedy} & --- & --- & 0.680 & 1.000 & 0.680 & 1.000 \\
     & \textit{Oracle} & --- & 0.977 & 0.977 & 1.000 & 0.977 & 1.000 \\
     & \textit{MV + CRC} & 0.901 & --- & $0.914{\scriptstyle\,\pm\,}0.016$ & $0.972{\scriptstyle\,\pm\,}0.035$ & $0.952{\scriptstyle\,\pm\,}0.016$ & $0.792{\scriptstyle\,\pm\,}0.091$ \\
    \cmidrule(lr){2-8}
     & \quad R-L8B & 0.907 & 0.838 & $0.914{\scriptstyle\,\pm\,}0.013$ & $0.987{\scriptstyle\,\pm\,}0.023$ & $0.952{\scriptstyle\,\pm\,}0.017$ & $0.852{\scriptstyle\,\pm\,}0.074$ \\
     & \quad R-Q8B & 0.914 & 0.848 & $0.924{\scriptstyle\,\pm\,}0.013$ & $0.985{\scriptstyle\,\pm\,}0.019$ & $0.952{\scriptstyle\,\pm\,}0.017$ & $0.895{\scriptstyle\,\pm\,}0.050$ \\
     & \quad J-G26B & 0.887 & 0.675 & $0.909{\scriptstyle\,\pm\,}0.017$ & $0.952{\scriptstyle\,\pm\,}0.043$ & $0.949{\scriptstyle\,\pm\,}0.014$ & $0.774{\scriptstyle\,\pm\,}0.071$ \\
     & \quad J-Q30B & 0.902 & 0.688 & $0.913{\scriptstyle\,\pm\,}0.015$ & $0.974{\scriptstyle\,\pm\,}0.038$ & $0.948{\scriptstyle\,\pm\,}0.015$ & $0.811{\scriptstyle\,\pm\,}0.091$ \\
     & \quad J-R1Q14B & 0.847 & 0.641 & $0.896{\scriptstyle\,\pm\,}0.027$ & $0.888{\scriptstyle\,\pm\,}0.060$ & $0.940{\scriptstyle\,\pm\,}0.020$ & $0.712{\scriptstyle\,\pm\,}0.090$ \\
     & \quad Ppl & 0.882 & 0.780 & $0.905{\scriptstyle\,\pm\,}0.018$ & $0.947{\scriptstyle\,\pm\,}0.037$ & $0.945{\scriptstyle\,\pm\,}0.015$ & $0.805{\scriptstyle\,\pm\,}0.065$ \\
     & \quad Ppl-Std & 0.888 & 0.760 & $0.910{\scriptstyle\,\pm\,}0.020$ & $0.941{\scriptstyle\,\pm\,}0.052$ & $0.945{\scriptstyle\,\pm\,}0.019$ & $0.774{\scriptstyle\,\pm\,}0.094$ \\
    \bottomrule
  \end{tabular}
\end{table}

\begin{table}[p]
  \centering
  \caption{MATH results ($n_\text{cal}=200$, $m=16$, 20 splits). $p_v$: score-weighted MV accuracy; best: score-based best-of-m.}
  \label{tab:main-math}
  \small
  \begin{tabular}{llccccccc}
    \toprule
    \textbf{Model} & \textbf{Method} & \textbf{$p_v$} & \textbf{best} &
    \multicolumn{2}{c}{$\boldsymbol{\alpha=0.10}$} &
    \multicolumn{2}{c}{$\boldsymbol{\alpha=0.05}$} \\
    \cmidrule(lr){5-6}\cmidrule(lr){7-8}
    & & & & $A_c$ & $Y$ & $A_c$ & $Y$ \\
    \midrule
    \multirow{10}{*}{Qwen3-1.7B} & \textit{Greedy} & --- & --- & 0.555 & 1.000 & 0.555 & 1.000 \\
     & \textit{Oracle} & --- & 0.854 & 0.854 & 1.000 & 0.854 & 1.000 \\
     & \textit{MV + CRC} & 0.712 & --- & $0.873{\scriptstyle\,\pm\,}0.029$ & $0.712{\scriptstyle\,\pm\,}0.049$ & $0.924{\scriptstyle\,\pm\,}0.017$ & $0.579{\scriptstyle\,\pm\,}0.054$ \\
    \cmidrule(lr){2-8}
     & \quad R-L8B & 0.718 & 0.662 & $0.874{\scriptstyle\,\pm\,}0.021$ & $0.749{\scriptstyle\,\pm\,}0.038$ & $0.925{\scriptstyle\,\pm\,}0.015$ & $0.629{\scriptstyle\,\pm\,}0.034$ \\
     & \quad R-Q8B & 0.728 & 0.692 & $0.876{\scriptstyle\,\pm\,}0.020$ & $0.776{\scriptstyle\,\pm\,}0.030$ & $0.925{\scriptstyle\,\pm\,}0.019$ & $0.687{\scriptstyle\,\pm\,}0.035$ \\
     & \quad J-G26B & 0.700 & 0.520 & $0.860{\scriptstyle\,\pm\,}0.021$ & $0.721{\scriptstyle\,\pm\,}0.035$ & $0.914{\scriptstyle\,\pm\,}0.022$ & $0.590{\scriptstyle\,\pm\,}0.062$ \\
     & \quad J-Q30B & 0.698 & 0.575 & $0.863{\scriptstyle\,\pm\,}0.030$ & $0.720{\scriptstyle\,\pm\,}0.050$ & $0.917{\scriptstyle\,\pm\,}0.020$ & $0.594{\scriptstyle\,\pm\,}0.052$ \\
     & \quad J-R1Q14B & 0.691 & 0.547 & $0.855{\scriptstyle\,\pm\,}0.029$ & $0.705{\scriptstyle\,\pm\,}0.050$ & $0.910{\scriptstyle\,\pm\,}0.018$ & $0.562{\scriptstyle\,\pm\,}0.060$ \\
     & \quad Ppl & 0.698 & 0.635 & $0.862{\scriptstyle\,\pm\,}0.032$ & $0.727{\scriptstyle\,\pm\,}0.049$ & $0.916{\scriptstyle\,\pm\,}0.023$ & $0.614{\scriptstyle\,\pm\,}0.052$ \\
     & \quad Ppl-Std & 0.708 & 0.632 & $0.866{\scriptstyle\,\pm\,}0.026$ & $0.736{\scriptstyle\,\pm\,}0.037$ & $0.926{\scriptstyle\,\pm\,}0.022$ & $0.621{\scriptstyle\,\pm\,}0.046$ \\
    \midrule
    \multirow{10}{*}{R1-Qwen-1.5B} & \textit{Greedy} & --- & --- & 0.476 & 1.000 & 0.476 & 1.000 \\
     & \textit{Oracle} & --- & 0.765 & 0.765 & 1.000 & 0.765 & 1.000 \\
     & \textit{MV + CRC} & 0.603 & --- & $0.809{\scriptstyle\,\pm\,}0.015$ & $0.465{\scriptstyle\,\pm\,}0.062$ & $0.828{\scriptstyle\,\pm\,}0.012$ & $0.240{\scriptstyle\,\pm\,}0.075$ \\
    \cmidrule(lr){2-8}
     & \quad R-L8B & 0.623 & 0.575 & $0.809{\scriptstyle\,\pm\,}0.014$ & $0.501{\scriptstyle\,\pm\,}0.055$ & $0.833{\scriptstyle\,\pm\,}0.014$ & $0.271{\scriptstyle\,\pm\,}0.052$ \\
     & \quad R-Q8B & 0.648 & 0.622 & $0.822{\scriptstyle\,\pm\,}0.016$ & $0.552{\scriptstyle\,\pm\,}0.075$ & $0.841{\scriptstyle\,\pm\,}0.012$ & $0.296{\scriptstyle\,\pm\,}0.072$ \\
     & \quad J-G26B & 0.618 & 0.567 & $0.809{\scriptstyle\,\pm\,}0.015$ & $0.523{\scriptstyle\,\pm\,}0.066$ & $0.838{\scriptstyle\,\pm\,}0.013$ & $0.300{\scriptstyle\,\pm\,}0.078$ \\
     & \quad J-Q30B & 0.606 & 0.523 & $0.806{\scriptstyle\,\pm\,}0.011$ & $0.488{\scriptstyle\,\pm\,}0.054$ & $0.831{\scriptstyle\,\pm\,}0.011$ & $0.290{\scriptstyle\,\pm\,}0.061$ \\
     & \quad J-R1Q14B & 0.608 & 0.583 & $0.810{\scriptstyle\,\pm\,}0.013$ & $0.503{\scriptstyle\,\pm\,}0.048$ & $0.834{\scriptstyle\,\pm\,}0.013$ & $0.295{\scriptstyle\,\pm\,}0.078$ \\
     & \quad Ppl & 0.572 & 0.498 & $0.778{\scriptstyle\,\pm\,}0.015$ & $0.429{\scriptstyle\,\pm\,}0.056$ & $0.815{\scriptstyle\,\pm\,}0.012$ & $0.251{\scriptstyle\,\pm\,}0.048$ \\
     & \quad Ppl-Std & 0.596 & 0.532 & $0.794{\scriptstyle\,\pm\,}0.015$ & $0.465{\scriptstyle\,\pm\,}0.045$ & $0.826{\scriptstyle\,\pm\,}0.008$ & $0.277{\scriptstyle\,\pm\,}0.052$ \\
    \midrule
    \multirow{10}{*}{Gemma-4-E4B} & \textit{Greedy} & --- & --- & 0.521 & 1.000 & 0.521 & 1.000 \\
     & \textit{Oracle} & --- & 0.685 & 0.685 & 1.000 & 0.685 & 1.000 \\
     & \textit{MV + CRC} & 0.583 & --- & $0.685{\scriptstyle\,\pm\,}0.008$ & $0.211{\scriptstyle\,\pm\,}0.125$ & --- & $0.000{\scriptstyle\,\pm\,}0.000$ \\
    \cmidrule(lr){2-8}
     & \quad R-L8B & 0.587 & 0.571 & $0.691{\scriptstyle\,\pm\,}0.011$ & $0.315{\scriptstyle\,\pm\,}0.058$ & $0.727{\scriptstyle\,\pm\,}0.034$ & $0.095{\scriptstyle\,\pm\,}0.078$ \\
     & \quad R-Q8B & 0.609 & 0.577 & $0.702{\scriptstyle\,\pm\,}0.011$ & $0.322{\scriptstyle\,\pm\,}0.070$ & $0.646{\scriptstyle\,\pm\,}0.026$ & $0.075{\scriptstyle\,\pm\,}0.055$ \\
     & \quad J-G26B & 0.587 & 0.542 & $0.703{\scriptstyle\,\pm\,}0.015$ & $0.311{\scriptstyle\,\pm\,}0.069$ & $0.723{\scriptstyle\,\pm\,}0.024$ & $0.071{\scriptstyle\,\pm\,}0.044$ \\
     & \quad J-Q30B & 0.568 & 0.512 & $0.674{\scriptstyle\,\pm\,}0.031$ & $0.281{\scriptstyle\,\pm\,}0.067$ & $0.529{\scriptstyle\,\pm\,}0.071$ & $0.016{\scriptstyle\,\pm\,}0.004$ \\
     & \quad J-R1Q14B & 0.586 & 0.515 & $0.692{\scriptstyle\,\pm\,}0.011$ & $0.307{\scriptstyle\,\pm\,}0.051$ & $0.777{\scriptstyle\,\pm\,}0.038$ & $0.069{\scriptstyle\,\pm\,}0.058$ \\
     & \quad Ppl & 0.511 & 0.468 & $0.679{\scriptstyle\,\pm\,}0.017$ & $0.291{\scriptstyle\,\pm\,}0.039$ & $0.841{\scriptstyle\,\pm\,}0.057$ & $0.061{\scriptstyle\,\pm\,}0.044$ \\
     & \quad Ppl-Std & 0.506 & 0.459 & $0.678{\scriptstyle\,\pm\,}0.010$ & $0.291{\scriptstyle\,\pm\,}0.046$ & $0.673{\scriptstyle\,\pm\,}0.025$ & $0.077{\scriptstyle\,\pm\,}0.065$ \\
    \bottomrule
  \end{tabular}
\end{table}

\begin{table}[p]
  \centering
  \caption{MATH-Hard results ($n_\text{cal}=200$, $m=16$, 20 splits). $p_v$: score-weighted MV accuracy; best: score-based best-of-m.}
  \label{tab:main-mathhard}
  \small
  \begin{tabular}{llccccccc}
    \toprule
    \textbf{Model} & \textbf{Method} & \textbf{$p_v$} & \textbf{best} &
    \multicolumn{2}{c}{$\boldsymbol{\alpha=0.10}$} &
    \multicolumn{2}{c}{$\boldsymbol{\alpha=0.05}$} \\
    \cmidrule(lr){5-6}\cmidrule(lr){7-8}
    & & & & $A_c$ & $Y$ & $A_c$ & $Y$ \\
    \midrule
    \multirow{10}{*}{Qwen3-8B} & \textit{Greedy} & --- & --- & 0.440 & 1.000 & 0.440 & 1.000 \\
     & \textit{Oracle} & --- & 0.699 & 0.699 & 1.000 & 0.699 & 1.000 \\
     & \textit{MV + CRC} & 0.537 & --- & $0.844{\scriptstyle\,\pm\,}0.030$ & $0.498{\scriptstyle\,\pm\,}0.041$ & $0.901{\scriptstyle\,\pm\,}0.020$ & $0.393{\scriptstyle\,\pm\,}0.041$ \\
    \cmidrule(lr){2-8}
     & \quad R-L8B & 0.536 & 0.489 & $0.837{\scriptstyle\,\pm\,}0.034$ & $0.548{\scriptstyle\,\pm\,}0.035$ & $0.905{\scriptstyle\,\pm\,}0.024$ & $0.448{\scriptstyle\,\pm\,}0.039$ \\
     & \quad R-Q8B & 0.535 & 0.487 & $0.834{\scriptstyle\,\pm\,}0.033$ & $0.550{\scriptstyle\,\pm\,}0.034$ & $0.905{\scriptstyle\,\pm\,}0.024$ & $0.443{\scriptstyle\,\pm\,}0.039$ \\
     & \quad J-G26B & 0.538 & 0.448 & $0.837{\scriptstyle\,\pm\,}0.029$ & $0.515{\scriptstyle\,\pm\,}0.038$ & $0.896{\scriptstyle\,\pm\,}0.024$ & $0.397{\scriptstyle\,\pm\,}0.050$ \\
     & \quad J-Q30B & 0.531 & 0.439 & $0.827{\scriptstyle\,\pm\,}0.024$ & $0.506{\scriptstyle\,\pm\,}0.026$ & $0.896{\scriptstyle\,\pm\,}0.017$ & $0.367{\scriptstyle\,\pm\,}0.051$ \\
     & \quad J-R1Q14B & 0.547 & 0.441 & $0.843{\scriptstyle\,\pm\,}0.033$ & $0.536{\scriptstyle\,\pm\,}0.034$ & $0.901{\scriptstyle\,\pm\,}0.027$ & $0.436{\scriptstyle\,\pm\,}0.052$ \\
     & \quad Ppl & 0.538 & 0.488 & $0.835{\scriptstyle\,\pm\,}0.025$ & $0.550{\scriptstyle\,\pm\,}0.022$ & $0.905{\scriptstyle\,\pm\,}0.028$ & $0.453{\scriptstyle\,\pm\,}0.044$ \\
     & \quad Ppl-Std & 0.541 & 0.476 & $0.831{\scriptstyle\,\pm\,}0.032$ & $0.542{\scriptstyle\,\pm\,}0.030$ & $0.906{\scriptstyle\,\pm\,}0.025$ & $0.439{\scriptstyle\,\pm\,}0.044$ \\
    \bottomrule
  \end{tabular}
\end{table}

\begin{table}[p]
  \centering
  \caption{HotpotQA results ($n_\text{cal}=200$, $m=16$, 20 splits). $p_v$: score-weighted MV accuracy; best: score-based best-of-m.}
  \label{tab:main-hotpotqa}
  \small
  \begin{tabular}{llccccccc}
    \toprule
    \textbf{Model} & \textbf{Method} & \textbf{$p_v$} & \textbf{best} &
    \multicolumn{2}{c}{$\boldsymbol{\alpha=0.10}$} &
    \multicolumn{2}{c}{$\boldsymbol{\alpha=0.05}$} \\
    \cmidrule(lr){5-6}\cmidrule(lr){7-8}
    & & & & $A_c$ & $Y$ & $A_c$ & $Y$ \\
    \midrule
    \multirow{10}{*}{Qwen3-1.7B} & \textit{Greedy} & --- & --- & 0.426 & 1.000 & 0.426 & 1.000 \\
     & \textit{Oracle} & --- & 0.739 & 0.739 & 1.000 & 0.739 & 1.000 \\
     & \textit{MV + CRC} & 0.519 & --- & $0.686{\scriptstyle\,\pm\,}0.022$ & $0.244{\scriptstyle\,\pm\,}0.057$ & $0.722{\scriptstyle\,\pm\,}0.033$ & $0.127{\scriptstyle\,\pm\,}0.039$ \\
    \cmidrule(lr){2-8}
     & \quad R-L8B & 0.489 & 0.424 & $0.668{\scriptstyle\,\pm\,}0.032$ & $0.270{\scriptstyle\,\pm\,}0.048$ & $0.682{\scriptstyle\,\pm\,}0.021$ & $0.147{\scriptstyle\,\pm\,}0.051$ \\
     & \quad R-Q8B & 0.475 & 0.402 & $0.645{\scriptstyle\,\pm\,}0.027$ & $0.264{\scriptstyle\,\pm\,}0.033$ & $0.680{\scriptstyle\,\pm\,}0.019$ & $0.142{\scriptstyle\,\pm\,}0.045$ \\
     & \quad J-G26B & 0.505 & 0.398 & $0.674{\scriptstyle\,\pm\,}0.020$ & $0.278{\scriptstyle\,\pm\,}0.052$ & $0.701{\scriptstyle\,\pm\,}0.014$ & $0.151{\scriptstyle\,\pm\,}0.048$ \\
     & \quad J-Q30B & 0.511 & 0.455 & $0.683{\scriptstyle\,\pm\,}0.021$ & $0.272{\scriptstyle\,\pm\,}0.049$ & $0.709{\scriptstyle\,\pm\,}0.026$ & $0.142{\scriptstyle\,\pm\,}0.039$ \\
     & \quad J-Q8B & 0.522 & 0.452 & $0.686{\scriptstyle\,\pm\,}0.023$ & $0.266{\scriptstyle\,\pm\,}0.052$ & $0.712{\scriptstyle\,\pm\,}0.030$ & $0.133{\scriptstyle\,\pm\,}0.051$ \\
     & \quad Ppl & 0.485 & 0.408 & $0.681{\scriptstyle\,\pm\,}0.031$ & $0.302{\scriptstyle\,\pm\,}0.063$ & $0.728{\scriptstyle\,\pm\,}0.027$ & $0.155{\scriptstyle\,\pm\,}0.039$ \\
     & \quad Ppl-Std & 0.471 & 0.388 & $0.675{\scriptstyle\,\pm\,}0.034$ & $0.297{\scriptstyle\,\pm\,}0.058$ & $0.731{\scriptstyle\,\pm\,}0.035$ & $0.150{\scriptstyle\,\pm\,}0.045$ \\
    \midrule
    \multirow{10}{*}{Gemma-4-E4B} & \textit{Greedy} & --- & --- & 0.700 & 1.000 & 0.700 & 1.000 \\
     & \textit{Oracle} & --- & 0.791 & 0.791 & 1.000 & 0.791 & 1.000 \\
     & \textit{MV + CRC} & 0.724 & --- & $0.828{\scriptstyle\,\pm\,}0.012$ & $0.333{\scriptstyle\,\pm\,}0.311$ & --- & $0.000{\scriptstyle\,\pm\,}0.000$ \\
    \cmidrule(lr){2-8}
     & \quad R-L8B & 0.728 & 0.719 & $0.834{\scriptstyle\,\pm\,}0.015$ & $0.562{\scriptstyle\,\pm\,}0.093$ & $0.862{\scriptstyle\,\pm\,}0.015$ & $0.179{\scriptstyle\,\pm\,}0.133$ \\
     & \quad R-Q8B & 0.724 & 0.714 & $0.831{\scriptstyle\,\pm\,}0.012$ & $0.537{\scriptstyle\,\pm\,}0.123$ & $0.820{\scriptstyle\,\pm\,}0.013$ & $0.133{\scriptstyle\,\pm\,}0.089$ \\
     & \quad J-G26B & 0.719 & 0.712 & $0.838{\scriptstyle\,\pm\,}0.020$ & $0.535{\scriptstyle\,\pm\,}0.104$ & $0.851{\scriptstyle\,\pm\,}0.011$ & $0.206{\scriptstyle\,\pm\,}0.145$ \\
     & \quad J-Q30B & 0.727 & 0.718 & $0.834{\scriptstyle\,\pm\,}0.015$ & $0.563{\scriptstyle\,\pm\,}0.079$ & $0.837{\scriptstyle\,\pm\,}0.014$ & $0.191{\scriptstyle\,\pm\,}0.144$ \\
     & \quad J-Q8B & 0.723 & 0.713 & $0.839{\scriptstyle\,\pm\,}0.018$ & $0.527{\scriptstyle\,\pm\,}0.174$ & $0.874{\scriptstyle\,\pm\,}0.023$ & $0.046{\scriptstyle\,\pm\,}0.003$ \\
     & \quad Ppl & 0.717 & 0.693 & $0.834{\scriptstyle\,\pm\,}0.015$ & $0.551{\scriptstyle\,\pm\,}0.099$ & $0.854{\scriptstyle\,\pm\,}0.013$ & $0.136{\scriptstyle\,\pm\,}0.094$ \\
     & \quad Ppl-Std & 0.718 & 0.688 & $0.831{\scriptstyle\,\pm\,}0.015$ & $0.547{\scriptstyle\,\pm\,}0.105$ & $0.807{\scriptstyle\,\pm\,}0.016$ & $0.149{\scriptstyle\,\pm\,}0.110$ \\
    \midrule
    \multirow{10}{*}{Qwen3-8B} & \textit{Greedy} & --- & --- & 0.676 & 1.000 & 0.676 & 1.000 \\
     & \textit{Oracle} & --- & 0.778 & 0.778 & 1.000 & 0.778 & 1.000 \\
     & \textit{MV + CRC} & 0.708 & --- & $0.816{\scriptstyle\,\pm\,}0.008$ & $0.367{\scriptstyle\,\pm\,}0.281$ & --- & $0.000{\scriptstyle\,\pm\,}0.000$ \\
    \cmidrule(lr){2-8}
     & \quad R-L8B & 0.704 & 0.690 & $0.816{\scriptstyle\,\pm\,}0.011$ & $0.533{\scriptstyle\,\pm\,}0.100$ & $0.832{\scriptstyle\,\pm\,}0.017$ & $0.131{\scriptstyle\,\pm\,}0.105$ \\
     & \quad R-Q8B & 0.705 & 0.688 & $0.814{\scriptstyle\,\pm\,}0.012$ & $0.547{\scriptstyle\,\pm\,}0.105$ & $0.778{\scriptstyle\,\pm\,}0.015$ & $0.112{\scriptstyle\,\pm\,}0.060$ \\
     & \quad J-G26B & 0.708 & 0.678 & $0.811{\scriptstyle\,\pm\,}0.009$ & $0.536{\scriptstyle\,\pm\,}0.096$ & $0.824{\scriptstyle\,\pm\,}0.014$ & $0.085{\scriptstyle\,\pm\,}0.003$ \\
     & \quad J-Q30B & 0.710 & 0.685 & $0.818{\scriptstyle\,\pm\,}0.014$ & $0.511{\scriptstyle\,\pm\,}0.174$ & $0.846{\scriptstyle\,\pm\,}0.025$ & $0.030{\scriptstyle\,\pm\,}0.002$ \\
     & \quad J-Q8B & 0.709 & 0.691 & $0.845{\scriptstyle\,\pm\,}0.056$ & $0.400{\scriptstyle\,\pm\,}0.264$ & $0.896{\scriptstyle\,\pm\,}0.038$ & $0.013{\scriptstyle\,\pm\,}0.001$ \\
     & \quad Ppl & 0.696 & 0.669 & $0.816{\scriptstyle\,\pm\,}0.010$ & $0.536{\scriptstyle\,\pm\,}0.093$ & $0.829{\scriptstyle\,\pm\,}0.016$ & $0.105{\scriptstyle\,\pm\,}0.059$ \\
     & \quad Ppl-Std & 0.695 & 0.666 & $0.819{\scriptstyle\,\pm\,}0.015$ & $0.536{\scriptstyle\,\pm\,}0.109$ & $0.831{\scriptstyle\,\pm\,}0.013$ & $0.142{\scriptstyle\,\pm\,}0.100$ \\
    \bottomrule
  \end{tabular}
\end{table}

\clearpage

\subsection{Separability and Strict-Separability Profiles}
\label{app:separability}

Figure~\ref{fig:delta-profiles} extends the bottom row of Figure~\ref{fig:frontiers} to all ten configurations and to all three score families.
Each panel shows the separability gap $\Delta(\lambda) = S_\text{cor}(\lambda) - S_\text{err}(\lambda)$ (solid) and the strict separability gap $\delta(\lambda) = h_\text{err}(\lambda) - h_\text{cor}(\lambda)$ (dashed) computed on the calibration set.
Two qualitative patterns emerge.
First, on the math cells $\Delta$ and $\delta$ remain visibly positive across the operationally relevant interior region $\lambda\in(0.4, 0.7)$ for the Reward and Judge scores, while Perplexity is positive but flatter --- consistent with the frontier ordering in Figure~\ref{fig:frontiers}.
Second, on HotpotQA all three score families produce profiles that hover near zero for much of the range, which is exactly the empirical signature of the failure mode characterised by Proposition~\ref{prop:nonsep}.

\begin{figure}[h]
	\centering
	\includegraphics[width=\textwidth]{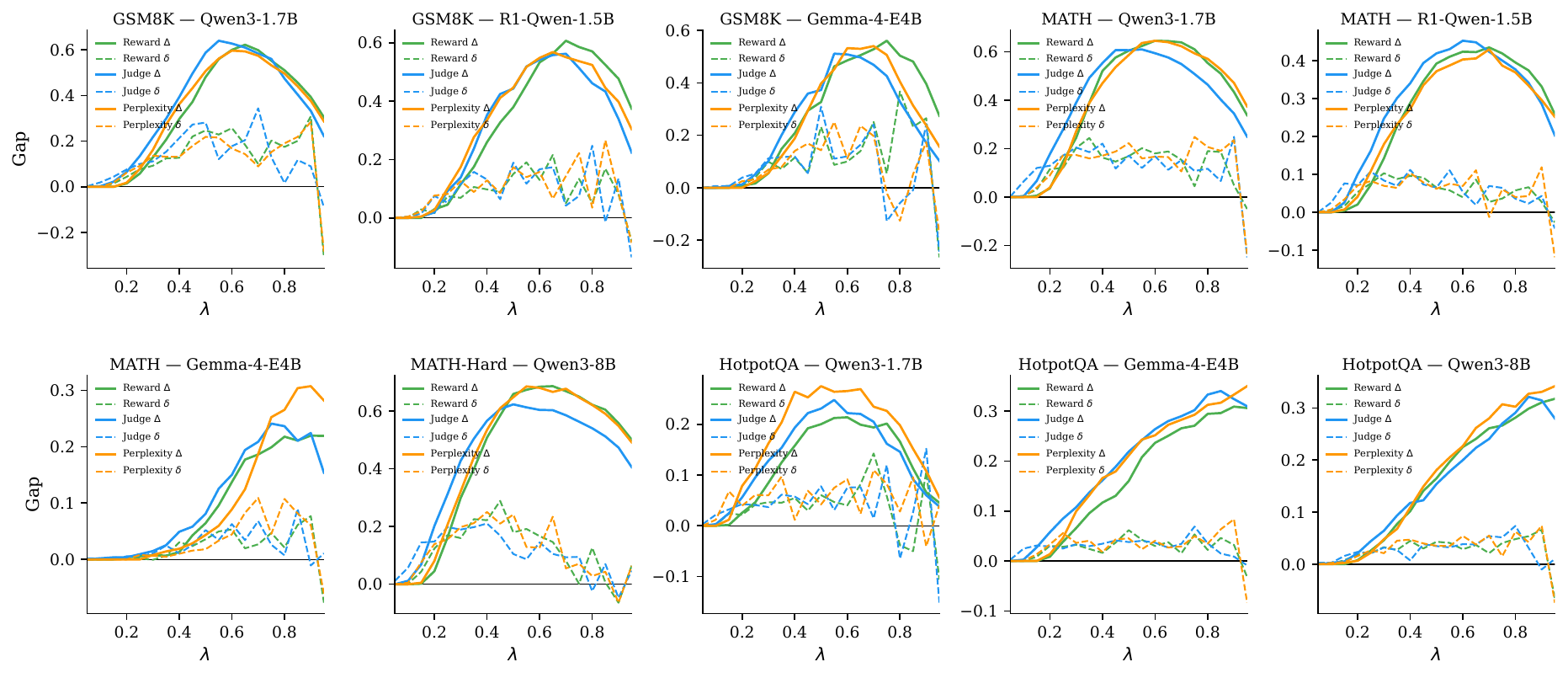}
	\caption{Separability gap $\Delta(\lambda)$ (solid) and strict separability gap $\delta(\lambda)$ (dashed) across the threshold range, per configuration. Three score families per panel: Reward, Judge, Perplexity. Calibration set, single split.}
	\label{fig:delta-profiles}
\end{figure}

\subsection{Predicted vs.\ Observed Accuracy: Extended Panels}
\label{app:pred-obs-extended}

Figure~\ref{fig:pred-vs-obs-extended} extends Figure~\ref{fig:pred-vs-obs} to all ten configurations.
The plug-in predictor of Corollary~\ref{cor:predictor} continues to track the empirical $A_c$ closely on cells not shown in the main text, including the harder MATH $\times$ Gemma-4-E4B and HotpotQA $\times$ Qwen3-8B cells where the answering margin is smaller and the predictor is correspondingly noisier.

\begin{figure}[h]
	\centering
	\includegraphics[width=\textwidth]{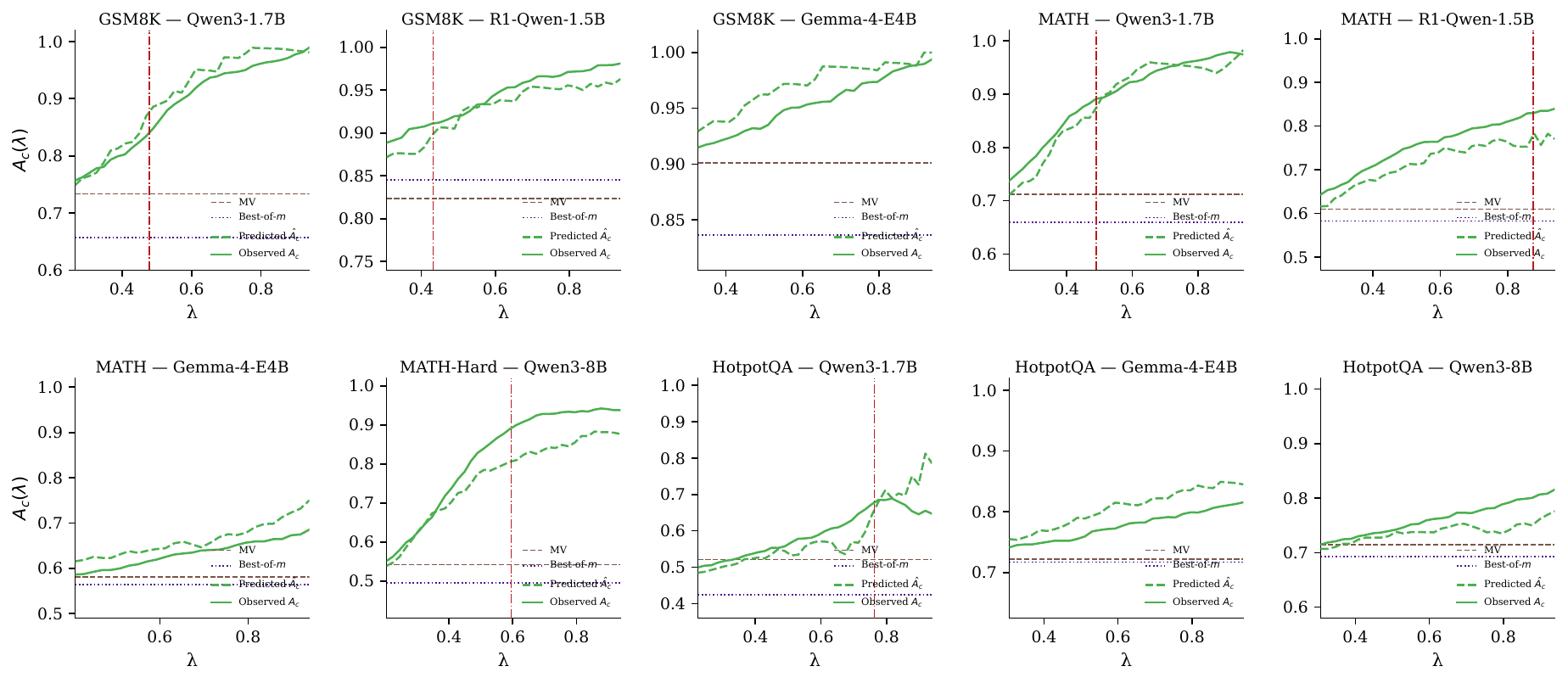}
	\caption{Predicted (dashed) vs.\ observed (solid) selective accuracy $A_c(\lambda)$ across all ten configurations. Calibration set, single split.}
	\label{fig:pred-vs-obs-extended}
\end{figure}

\subsection{Frontier Comparison: Extended Panels}
\label{app:frontier-extended}

Figure~\ref{fig:frontiers-extended} extends Figure~\ref{fig:frontiers} to all ten configurations and overlays the secondary scorer instantiations per family (dashed) and the appendix-only ablation scores (Perplexity-Std, SC; dotted).
The full picture confirms that score-family rankings are stable across secondary instantiations, while SC tends to lie below the other score families (Table~\ref{tab:abla-score-full}).

\begin{figure}[h]
	\centering
	\includegraphics[width=\textwidth]{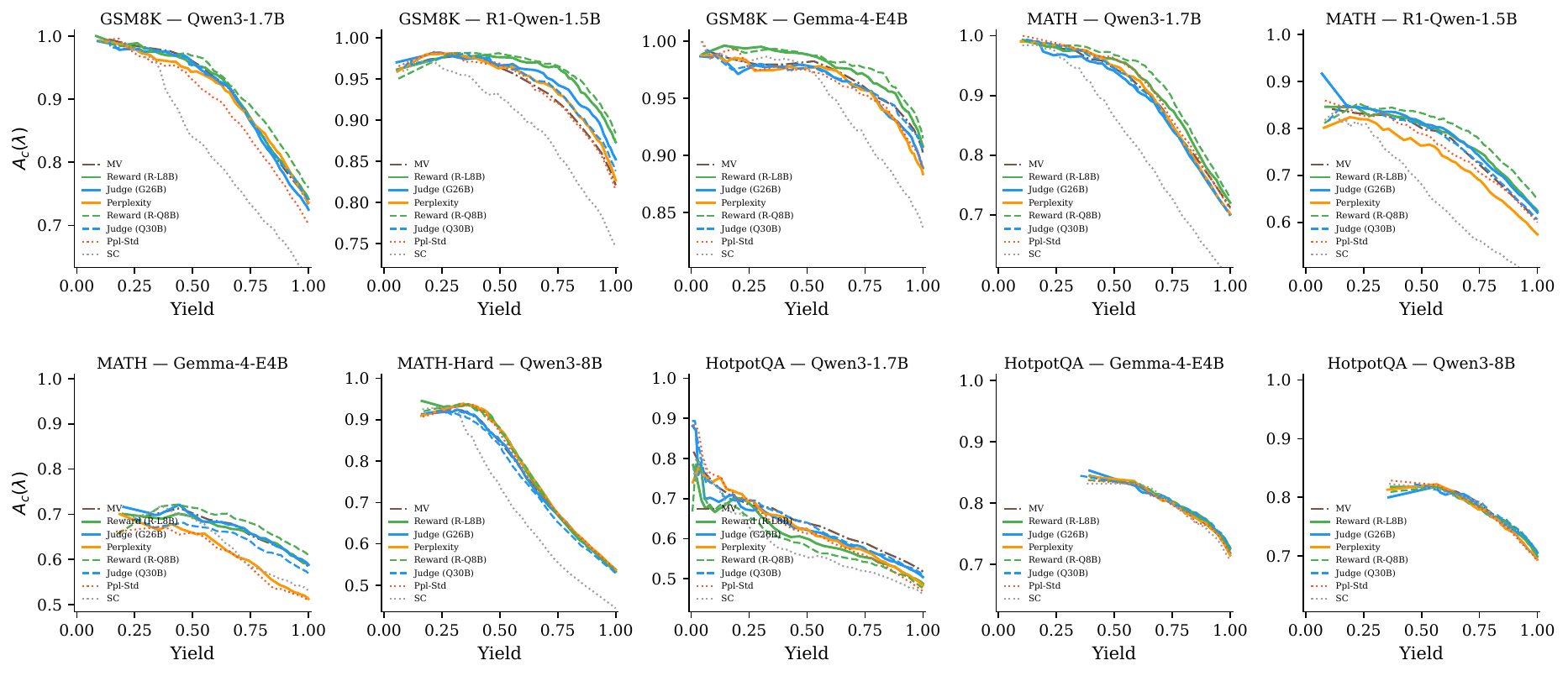}
	\caption{Yield--$A_c$ frontiers across all ten configurations, including secondary scorer instantiations (dashed) and ablation scores (Perplexity-Std, SC; dotted).}
	\label{fig:frontiers-extended}
\end{figure}

\paragraph{Frontier AUC computation.}
Frontier AUC is the area under the (yield, $A_c$) curve, computed by the trapezoid rule over the yield range $[y_{\min},\,1]$ observed when sweeping $\lambda \in [0,1]$.
For score-weighted methods the continuous-valued confidence $\nu$ can reach small values whenever a few high-scoring drafts dominate.
For MV, $\nu$ is a vote fraction $k/m$ ($m=16$), so once $\lambda$ exceeds the second-most-common fraction no further abstention is possible.
Because the attainable yield range $[y_{\min}, 1]$ differs across score families, AUC is computed over each method’s observed domain rather than a shared one. This avoids extrapolating accuracy to unattainable operating points, which would introduce unsupported assumptions. The domain width is itself informative: methods confined to a narrow yield range are intrinsically less useful than those maintaining performance across a broader range. Enforcing a common domain (e.g., truncating at the worst-case $y_{\min}$) would instead discard valid operating regions and systematically penalize methods with stronger separability. Frontier AUC therefore measures performance over the set of achievable operating points for each method, providing a conservative and faithful summary of the accuracy–yield trade-off.

Frontier AUC values for every (cell, scorer, ablation-score) triple are tabulated in Table~\ref{tab:frontier-auc-full}.

\begin{table}[h]
	\centering
	\caption{Frontier AUC for each cell and score type. Higher is better; the per-cell maximum across the main-paper score families is in \textbf{bold}.}
	\label{tab:frontier-auc-full}
	\small
	\begin{tabular}{llccccccc}
		\toprule
		\textbf{Dataset} & \textbf{Model} & \textbf{R-L8B} & \textbf{R-Q8B} & \textbf{J-G26B} & \textbf{J-Q30B} & \textbf{Ppl} & \textbf{Ppl-Std} & \textbf{SC} \\
		\midrule
		GSM8K & Qwen3-1.7B & \textbf{0.8401} & 0.8387 & 0.8281 & 0.7820 & 0.8206 & 0.8052 & 0.7497 \\
		GSM8K & R1-Qwen-1.5B & \textbf{0.9062} & 0.9032 & 0.9014 & 0.8671 & 0.8918 & 0.8817 & 0.8481 \\
		GSM8K & Gemma-4-E4B & \textbf{0.9321} & 0.9329 & 0.9255 & 0.9079 & 0.9212 & 0.9188 & 0.9071 \\
		MATH & Qwen3-1.7B & \textbf{0.8157} & 0.8205 & 0.8039 & 0.7923 & 0.8157 & 0.8143 & 0.7388 \\
		MATH & R1-Qwen-1.5B & 0.7126 & 0.7332 & \textbf{0.7311} & 0.6938 & 0.6791 & 0.6973 & 0.6072 \\
		MATH & Gemma-4-E4B & 0.5347 & 0.5619 & \textbf{0.5380} & 0.4862 & 0.5072 & 0.5107 & 0.5087 \\
		MATH-Hard & Qwen3-8B & 0.6540 & 0.6465 & 0.6433 & 0.6181 & \textbf{0.6566} & 0.6473 & 0.5769 \\
		HotpotQA & Qwen3-1.7B & 0.6007 & 0.5910 & 0.6224 & 0.6366 & \textbf{0.6227} & 0.6192 & 0.5775 \\
		HotpotQA & Gemma-4-E4B & 0.4879 & 0.4928 & \textbf{0.4880} & 0.5218 & 0.4832 & 0.4880 & 0.4947 \\
		HotpotQA & Qwen3-8B & 0.4961 & 0.4982 & 0.5032 & 0.4165 & \textbf{0.5077} & 0.4966 & 0.5030 \\
		\bottomrule
	\end{tabular}
\end{table}

\subsection{Small-and-Fast Reward Scorer}
\label{app:reward-3b}

The main-paper Reward family uses an 8B-parameter scorer; for completeness we also report results with the 3B-parameter Skywork-V2-Llama-3.2-3B scorer on all 10 cells. The smaller scorer demonstrates selective accuracy and yield on par with its 8B counterparts (\Cref{tab:reward-3b}), while reducing scoring wall-clock cost by $1.9$--$2.3\times$ (\Cref{tab:runtime-score}). This supports the practical claim that score separability, not scorer size, is the dominant factor for our framework once a reasonable scorer is in hand. This also suggests that a lightweight reward model is a practical choice when inference budget is constrained.

\begin{table}[h]
	\centering
	\caption{Skywork-V2-Llama-3.2-3B vs.\ 8B counterparts across all 10 cells. $A_c$/$Y$: selective accuracy and yield at $\alpha{=}0.10$. Bold marks the best per metric per cell.}
	\label{tab:reward-3b}
	\small
	\begin{tabular}{llccc ccc}
		\toprule
		& & \multicolumn{3}{c}{$\boldsymbol{A_c}$ $(\alpha{=}0.10)$} & \multicolumn{3}{c}{$\boldsymbol{Y}$ $(\alpha{=}0.10)$} \\
        \midrule
		\textbf{Dataset} & \textbf{Model} & R-L3B & R-L8B & R-Q8B & R-L3B & R-L8B & R-Q8B \\
		\midrule
		GSM8K     & Qwen3-1.7B   & \textbf{0.883} & 0.866 & 0.863 & 0.772 & 0.758 & \textbf{0.810} \\
		GSM8K     & R1-Qwen-1.5B & 0.902 & 0.901 & \textbf{0.908} & 0.945 & 0.951 & \textbf{0.957} \\
		GSM8K     & Gemma-4-E4B  & 0.922 & 0.916 & \textbf{0.929} & 0.965 & \textbf{0.979} & 0.972 \\
		MATH      & Qwen3-1.7B   & 0.865 & 0.879 & \textbf{0.882} & 0.740 & 0.740 & \textbf{0.769} \\
		MATH      & R1-Qwen-1.5B & 0.808 & 0.812 & \textbf{0.829} & 0.441 & 0.489 & \textbf{0.532} \\
		MATH      & Gemma-4-E4B  & 0.695 & 0.693 & \textbf{0.706} & 0.298 & 0.299 & \textbf{0.314} \\
		MATH-Hard & Qwen3-8B     & 0.846 & \textbf{0.847} & 0.829 & 0.524 & 0.537 & \textbf{0.558} \\
		HotpotQA  & Qwen3-1.7B   & 0.678 & \textbf{0.679} & 0.657 & \textbf{0.252} & 0.246 & 0.246 \\
		HotpotQA  & Gemma-4-E4B  & \textbf{0.833} & 0.832 & 0.825 & 0.533 & \textbf{0.580} & 0.568 \\
		HotpotQA  & Qwen3-8B     & 0.808 & \textbf{0.813} & 0.806 & 0.582 & 0.569 & \textbf{0.585} \\
		\bottomrule
	\end{tabular}
\end{table}

\subsection{Wall-Clock Runtimes}
\label{app:runtimes}

All timing experiments were run on a single NVIDIA A100 (80 GB) GPU.

\begin{table}[h]
	\centering
	\caption{Path-generation wall-clock time per cell. Format is HH:MM.}
	\label{tab:runtime-gen}
	\small
	\begin{tabular}{llcc}
		\toprule
		\textbf{Dataset} & \textbf{Model} & \textbf{Time} \\
		\midrule
		GSM8K     & Qwen3-1.7B     & 44:14 \\
		GSM8K     & R1-Qwen-1.5B   & 01:58 \\
		GSM8K     & Gemma-4-E4B    & 53:06 \\
		MATH      & Qwen3-1.7B     & 102:05 \\
		MATH      & R1-Qwen-1.5B   & 32:02 \\
		MATH      & Gemma-4-E4B    & 121:52 \\
		MATH-Hard & Qwen3-8B       & 122:32 \\
		HotpotQA  & Qwen3-1.7B     & 20:28 \\
		HotpotQA  & Gemma-4-E4B    & 49:32 \\
		HotpotQA  & Qwen3-8B       & 12:10 \\
		\bottomrule
	\end{tabular}
\end{table}

\begin{table}[h]
	\centering
	\caption{Scoring wall-clock time per cell per scorer family. Reward = Skywork-V2-Llama-8B; Reward-small = Skywork-V2-Llama-3B; Judge = Gemma-26B. Format is HH:MM.}
	\label{tab:runtime-score}
	\small
	\begin{tabular}{llcccc}
		\toprule
		\textbf{Dataset} & \textbf{Model} & \textbf{Reward} & \textbf{Reward-small} & \textbf{Judge} & \textbf{Perplexity} \\
		\midrule
		GSM8K     & Qwen3-1.7B   & 00:32 & 00:16 & 00:50 & 00:17 \\
		GSM8K     & R1-Qwen-1.5B & 00:09 & 00:04 & 00:18 & 00:14 \\
		GSM8K     & Gemma-4-E4B  & 00:19 & 00:09 & 00:38 & 00:30 \\
		MATH      & Qwen3-1.7B   & 00:55 & 00:25 & 01:22 & 00:34 \\
		MATH      & R1-Qwen-1.5B & 00:44 & 00:20 & 01:01 & 00:20 \\
		MATH      & Gemma-4-E4B  & 00:41 & 00:20 & 00:58 & 00:51 \\
		MATH-Hard & Qwen3-8B     & 01:05 & 00:28 & 01:35 & 02:03 \\
		HotpotQA  & Qwen3-1.7B   & 00:27 & 00:14 & 00:35 & 00:14 \\
		HotpotQA  & Gemma-4-E4B  & 00:21 & 00:10 & 00:29 & 00:29 \\
		HotpotQA  & Qwen3-8B     & 00:17 & 00:09 & 00:26 & 00:21 \\
		\bottomrule
	\end{tabular}
\end{table}

Path generation dominates total runtime; scoring overhead is modest (under 2 hours per scorer per cell in the reference setting).
Judge scoring is overall the costliest scorer family owing to its use of a 26--30B model per draft, but remains well under 2 hours.
Perplexity scoring is cheapest on easy datasets (GSM8K) but rises on MATH-Hard where longer chains increase per-token inference cost.

\subsection{Qualitative Examples}
\label{app:cr-examples}

We show one answered and one abstained question from MATH with Qwen3-1.7B (Reward score, $\alpha=0.10$, one calibration split with calibrated threshold $\hatl = 0.489$). For each question the 16 paths are grouped by their extracted answer; for every group we list the number of paths, the group's score weight, and how many of its paths exhausted the 4096-token generation budget (Truncated). The confidence $\nu$ is the weight of the leading group.

\begin{tcolorbox}[title={Answered (correct)}, fonttitle=\small\bfseries, fontupper=\small]
\textbf{Question.} A parametric curve is defined by $(x,y) = (\sin^2 t, 2\cos t)$, where $t$ ranges over all real numbers. The same parametric curve can be defined by $(x,y) = (-s^2 - 2s, 2s + 2)$, where $s$ ranges over some interval $I$. Find the interval $I$.\\[2pt]
\textbf{Reference answer.} $[-2, 0]$\\[4pt]
\begin{tabular}{lcccc}
\toprule
Answer & Paths & Weight & Truncated & Correct \\
\midrule
$[-2,0]$ & 15 & 0.984 & 0 & yes \\
$-2$     & 1  & 0.016 & 1 & no  \\
\bottomrule
\end{tabular}\\[4pt]
$\nu = 0.984 > \hatl = 0.489$ $\Rightarrow$ the system answers $[-2,0]$, which is correct.\\[2pt]
\textbf{End of the highest-weighted path (1493 tokens):} ``\ldots\ So all these examples check out. Therefore, the interval $I$ is indeed $[-2, 0]$. \ldots\ The interval $I$ is $\boxed{[-2, 0]}$.''
\end{tcolorbox}

\begin{tcolorbox}[title={Abstained (the vote would have been wrong)}, fonttitle=\small\bfseries, fontupper=\small]
\textbf{Question.} Positive integers $a$, $b$, $c$, and $d$ satisfy $a > b > c > d$, $a + b + c + d = 2010$, and $a^2 - b^2 + c^2 - d^2 = 2010$. Find the number of possible values of $a$.\\[2pt]
\textbf{Reference answer.} $501$\\[4pt]
\begin{tabular}{lcccc}
\toprule
Answer & Paths & Weight & Truncated & Correct \\
\midrule
$3$    & 1 & 0.360 & 1 & no  \\
$1$    & 8 & 0.293 & 8 & no  \\
$501$  & 2 & 0.144 & 1 & yes \\
$2$    & 2 & 0.129 & 2 & no  \\
$2010$ & 2 & 0.053 & 2 & no  \\
$0$    & 1 & 0.021 & 1 & no  \\
\bottomrule
\end{tabular}\\[4pt]
$\nu = 0.360 \le \hatl = 0.489$ $\Rightarrow$ the system abstains. Answering would have been wrong under either aggregation: the weighted vote selects $3$ and the majority vote selects $1$.\\[2pt]
\textbf{End of the highest-weighted path (truncated at 4096 tokens):} ``\ldots\ Therefore, $a \geq b + 1 \geq c + 2 \geq d + 3$. So, $a \geq d + 3$. But not sure if helpful. Alternatively, since $a + b$'' [budget exhausted]
\end{tcolorbox}

In the second example, 15 of the 16 paths run out of budget before reaching an answer. The ``answers'' extracted from them are the last numbers they mention, and $\nu$ stays well below $\hatl$; the only completed path gives the correct answer. Abstaining on unfinished reasoning is typical of this configuration (averages over 20 calibration splits): 69\% of the paths of abstained questions are truncated, against 27\% of all paths. The same mechanism can also produce confident errors. In 37\% of the questions answered incorrectly at $\alpha=0.10$, the winning answer comes mostly from truncated paths that end on the same number.

\section{Ablation Experiment Results}
\label{app:ablation-results}

\subsection{Number of Paths}
\label{app:abla-m}

Table~\ref{tab:abla-m} reports CRC results at $\alpha=0.10$ for $m \in \{4, 8, 12, 16, 20\}$ on the representative cell GSM8K $\times$ Qwen3-1.7B with the Reward scorer (mean $\pm$ std over 20 calibration--test splits).
$A_c$ saturates around $m=8$ and additional paths primarily increase yield --- the qualitative picture predicted by Theorem~\ref{thm:ca-gain}: $\Delta(\lambda)$ is governed by the intrinsic separability of the score, and additional paths sharpen $\nu$ rather than shifting the frontier.
Confident-error is consistent with the nominal targets $\alpha \in \{0.10, 0.05\}$ across all rows, up to calibration and test variability, in line with the finite-sample conformal guarantee.
\begin{table}[h]
	\centering
	\caption{Effect of the number of sampled paths $m$ at $\alpha=0.10$ on GSM8K $\times$ Qwen3-1.7B (Reward scorer Skywork-V2-Llama-3.2-3B; mean $\pm$ std over 20 splits).}
	\label{tab:abla-m}
	\small
	\begin{tabular}{ccccc}
		\toprule
		$m$ & $p_v$ & \textbf{$A_c$}    & \textbf{Yield}    & \textbf{Confident-error} \\
		\midrule
		4   & 0.671 & 0.853 $\pm$ 0.021 & 0.649 $\pm$ 0.033 & 0.096 $\pm$ 0.018   \\
		8   & 0.721 & 0.862 $\pm$ 0.025 & 0.727 $\pm$ 0.055 & 0.102 $\pm$ 0.026   \\
		12  & 0.752 & 0.874 $\pm$ 0.024 & 0.770 $\pm$ 0.052 & 0.098 $\pm$ 0.025   \\
		16  & 0.765 & 0.874 $\pm$ 0.021 & 0.797 $\pm$ 0.043 & 0.102 $\pm$ 0.023   \\
		20  & 0.773 & 0.876 $\pm$ 0.026 & 0.807 $\pm$ 0.051 & 0.101 $\pm$ 0.028   \\
		\bottomrule
	\end{tabular}
\end{table}

\subsection{Calibration Set Size}
\label{app:abla-ncal}

Table~\ref{tab:abla-ncal} reports the calibrated threshold's mean and standard deviation alongside CRC results at $\alpha=0.10$ as $n_\text{cal}$ varies in $\{50, 100, 150, 200, 300, 500\}$ on the same representative cell (mean $\pm$ std over 20 splits).
$A_c$ and yield are essentially flat across the range ($\pm 0.01$); the primary effect of larger $n_\text{cal}$ is to reduce the variance of $\hat\lambda$, which stabilizes beyond $n_\text{cal}\approx 200$.
The observed confident-error rates fluctuate around
the nominal target across calibration sizes, consistent
with the marginal CRC guarantee.

\begin{table}[h]
	\centering
	\caption{Effect of calibration set size $n_\text{cal}$ at $\alpha=0.10$ on GSM8K $\times$ Qwen3-1.7B (Reward scorer; mean $\pm$ std over 20 splits).}
	\label{tab:abla-ncal}
	\small
	\begin{tabular}{ccccc}
		\toprule
		$n_\text{cal}$ & \textbf{$A_c$}    & \textbf{Yield}    & \textbf{Confident-error} & \textbf{$\hat\lambda$} \\
		\midrule
		50             & 0.883 $\pm$ 0.038 & 0.773 $\pm$ 0.083 & 0.094 $\pm$ 0.040   & 0.509 $\pm$ 0.085      \\
		100            & 0.882 $\pm$ 0.028 & 0.778 $\pm$ 0.063 & 0.094 $\pm$ 0.029   & 0.506 $\pm$ 0.060      \\
		150            & 0.878 $\pm$ 0.029 & 0.786 $\pm$ 0.061 & 0.098 $\pm$ 0.030   & 0.498 $\pm$ 0.058      \\
		200            & 0.874 $\pm$ 0.021 & 0.797 $\pm$ 0.043 & 0.102 $\pm$ 0.023   & 0.488 $\pm$ 0.039      \\
		300            & 0.874 $\pm$ 0.016 & 0.798 $\pm$ 0.032 & 0.101 $\pm$ 0.017   & 0.488 $\pm$ 0.029      \\
		500            & 0.874 $\pm$ 0.011 & 0.797 $\pm$ 0.020 & 0.101 $\pm$ 0.011   & 0.489 $\pm$ 0.017      \\
		\bottomrule
	\end{tabular}
\end{table}

\subsection{Weight Function}
\label{app:abla-weight}

Table~\ref{tab:abla-weight} reports CRC results at $\alpha=0.10$ across five weight functions on four representative cells.
Across every cell tested, $w(q) = \exp(q)$ achieves the best overall trade-off between yield and accuracy.
Large magnifiers, ($\exp(5 q)$), collapse the weighted vote toward best-of-$m$ selection; uniform weighting is competitive on $A_c$ on the easier cells but underperforms exponential weighting on harder settings.

\begin{table}[h]
	\centering
	\caption{Weight function ablation at $\alpha=0.10$ across four cells (Judge scorer Gemma-26B; mean $\pm$ std over 20 splits).}
	\label{tab:abla-weight}
	\small
	\begin{tabular}{llcccc}
		\toprule
		\textbf{Cell} & \textbf{Weight} & \textbf{$p_v$} & \textbf{$A_c$} & \textbf{Yield} \\
		\midrule
		\multirow{5}{*}{GSM8K$\times$Qwen3-1.7B} & Uniform & 0.7386 & 0.8851 $\pm$ 0.0233 & 0.7344 $\pm$ 0.0413 \\
		 & Linear & 0.7394 & 0.8654 $\pm$ 0.0252 & 0.7729 $\pm$ 0.0452 \\
		 & $\exp(q)$ & 0.7358 & 0.8626 $\pm$ 0.0223 & 0.7731 $\pm$ 0.0413 \\
		 & $\exp(2 q)$ & 0.7338 & 0.8635 $\pm$ 0.0216 & 0.7629 $\pm$ 0.0406 \\
		 & $\exp(5 q)$ & 0.7300 & 0.8683 $\pm$ 0.0184 & 0.7375 $\pm$ 0.0358 \\
		\midrule
		\multirow{5}{*}{GSM8K$\times$R1-Qwen-1.5B} & Uniform & 0.8236 & 0.8979 $\pm$ 0.0206 & 0.8391 $\pm$ 0.0585 \\
		 & Linear & 0.8558 & 0.8982 $\pm$ 0.0171 & 0.9114 $\pm$ 0.0416 \\
		 & $\exp(q)$ & 0.8610 & 0.8997 $\pm$ 0.0139 & 0.9230 $\pm$ 0.0318 \\
		 & $\exp(2 q)$ & 0.8676 & 0.9003 $\pm$ 0.0138 & 0.9302 $\pm$ 0.0339 \\
		 & $\exp(5 q)$ & 0.8732 & 0.9012 $\pm$ 0.0144 & 0.9351 $\pm$ 0.0372 \\
		\midrule
		\multirow{5}{*}{MATH$\times$Qwen3-1.7B} & Uniform & 0.7007 & 0.8811 $\pm$ 0.0213 & 0.6747 $\pm$ 0.0406 \\
		 & Linear & 0.6943 & 0.8689 $\pm$ 0.0241 & 0.6823 $\pm$ 0.0481 \\
		 & $\exp(q)$ & 0.6920 & 0.8696 $\pm$ 0.0228 & 0.6712 $\pm$ 0.0457 \\
		 & $\exp(2 q)$ & 0.6807 & 0.8652 $\pm$ 0.0213 & 0.6589 $\pm$ 0.0427 \\
		 & $\exp(5 q)$ & 0.6570 & 0.8572 $\pm$ 0.0182 & 0.6293 $\pm$ 0.0323 \\
		\midrule
		\multirow{5}{*}{MATH$\times$R1-Qwen-1.5B} & Uniform & 0.5977 & 0.7993 $\pm$ 0.0202 & 0.4363 $\pm$ 0.0778 \\
		 & Linear & 0.6107 & 0.8010 $\pm$ 0.0173 & 0.4811 $\pm$ 0.0647 \\
		 & $\exp(q)$ & 0.6093 & 0.8016 $\pm$ 0.0156 & 0.4808 $\pm$ 0.0634 \\
		 & $\exp(2 q)$ & 0.6080 & 0.8007 $\pm$ 0.0144 & 0.4874 $\pm$ 0.0620 \\
		 & $\exp(5 q)$ & 0.6020 & 0.8005 $\pm$ 0.0164 & 0.4855 $\pm$ 0.0650 \\
		\bottomrule
	\end{tabular}
\end{table}

\vspace{3em}
\subsection{Score Type Ablation: Full Frontier-AUC Comparison}
\label{app:abla-score}

Table~\ref{tab:abla-score-full} shows that no single score family dominates. Reward or Judge attains the highest Frontier AUC on six of the ten configurations and Perplexity on three (one tie with Reward), while SC is the strongest score on only one configuration and the weakest on six. The learned scorers lead on six of the seven math configurations. The MV baseline, which receives a frontier AUC by sweeping the abstention threshold over the discrete vote-fraction grid $\{k/m\}$, is the lowest performer on four configurations.

\begin{table}[h]
	\centering
	\caption{Frontier AUC across all score types (main-paper + ablations) for all ten configurations.}
	\label{tab:abla-score-full}
	\small
	\begin{tabular}{llcccccc}
		\toprule
		\textbf{Dataset} & \textbf{Model} & \textbf{Reward} & \textbf{Judge} & \textbf{Perplexity} & \textbf{Ppl-Std} & \textbf{SC} & \textbf{MV} \\
		\midrule
		GSM8K & Qwen3-1.7B & \textbf{0.8401} & 0.8281 & 0.8206 & 0.8052 & 0.7497 & 0.7837 \\
		GSM8K & R1-Qwen-1.5B & \textbf{0.9062} & 0.9014 & 0.8918 & 0.8817 & 0.8481 & 0.8598 \\
		GSM8K & Gemma-4-E4B & \textbf{0.9321} & 0.9255 & 0.9212 & 0.9188 & 0.9071 & 0.9053 \\
		MATH & Qwen3-1.7B & \textbf{0.8157} & 0.8039 & 0.8157 & 0.8143 & 0.7388 & 0.7619 \\
		MATH & R1-Qwen-1.5B & 0.7126 & \textbf{0.7311} & 0.6791 & 0.6973 & 0.6072 & 0.6788 \\
		MATH & Gemma-4-E4B & 0.5347 & \textbf{0.5380} & 0.5072 & 0.5107 & 0.5087 & 0.4787 \\
		MATH-Hard & Qwen3-8B & 0.6540 & 0.6433 & \textbf{0.6566} & 0.6473 & 0.5769 & 0.5782 \\
		HotpotQA & Qwen3-1.7B & 0.6007 & 0.6224 & 0.6227 & 0.6192 & 0.5775 & \textbf{0.6342} \\
		HotpotQA & Gemma-4-E4B & 0.4879 & 0.4880 & 0.4832 & 0.4880 & \textbf{0.4947} & 0.3373 \\
		HotpotQA & Qwen3-8B & 0.4961 & 0.5032 & \textbf{0.5077} & 0.4966 & 0.5030 & 0.3679 \\
		\bottomrule
	\end{tabular}
\end{table}

\section{Baselines, a Stronger Model, and Robustness Analyses}
\label{app:cr}
 
This section reports the comparison of Section~\ref{sec:exp-q4} in full, the additional Qwen3-14B experiment, and analyses of distribution shift, calibration-set size, confidence granularity and the yield identity. Unless stated otherwise, results use the Reward-L8B scorer, $m=16$, $n_\mathrm{cal}=200$ and 20 calibration--test splits shared by all methods of a configuration.
 
\subsection{Additional Experiment: Qwen3-14B}
\label{app:cr-qwen14}
 
As an additional experiment, we run Qwen3-14B on MATH and GSM8K to test whether calibrated abstention remains useful when the base model is accurate. The setup follows the Qwen3-1.7B configurations (same prompts, generation prompt, sampling parameters, $m = 16$ and $n_\mathrm{cal} = 200$), with a generation budget of 16384 tokens for the longer reasoning of the larger model.
 
\subsection{Answer Extraction for Qwen3 Paths}
\label{app:cr-extraction}
 
Qwen3 paths end with the closing tag \texttt{</think>}. When a path has no \verb|\boxed{...}| expression, the fallback of Appendix~\ref{app:answer-extraction} can keep this tag in the extracted answer, which then does not match the reference. For Qwen3-14B we also report results with the tag stripped before extraction, marked $\dagger$. Table~\ref{tab:cr-extraction} shows the effect of this change on every math configuration: it matters for Qwen3-1.7B and Qwen3-14B and is negligible for the other models.
 
\begin{table}[h]
	\centering
	\caption{Effect of stripping the \texttt{</think>} tag before answer extraction (original $\to$ stripped). All other results in the paper use the original extractor.}
	\label{tab:cr-extraction}
	\footnotesize
	\begin{tabular}{llcccc}
		\toprule
		Dataset & Model & MV & Oracle & Path accuracy & $p_v$ (Reward-L8B) \\
		\midrule
		GSM8K & Qwen3-1.7B & $0.730\to0.908$ & $0.899\to0.964$ & $0.584\to0.736$ & $0.740\to0.892$ \\
		GSM8K & R1-Qwen-1.5B & $0.822\to0.823$ & $0.956\to0.956$ & $0.624\to0.625$ & $0.873\to0.873$ \\
		GSM8K & Gemma-4-E4B & $0.901\to0.901$ & $0.977\to0.977$ & $0.678\to0.678$ & $0.907\to0.907$ \\
		MATH & Qwen3-1.7B & $0.712\to0.765$ & $0.854\to0.862$ & $0.551\to0.608$ & $0.718\to0.759$ \\
		MATH & R1-Qwen-1.5B & $0.603\to0.605$ & $0.765\to0.772$ & $0.474\to0.478$ & $0.623\to0.625$ \\
		MATH & Gemma-4-E4B & $0.583\to0.583$ & $0.685\to0.685$ & $0.524\to0.524$ & $0.587\to0.587$ \\
		MATH-Hard & Qwen3-8B & $0.537\to0.542$ & $0.699\to0.700$ & $0.438\to0.442$ & $0.536\to0.538$ \\
		MATH & Qwen3-14B & $0.881\to0.898$ & $0.930\to0.942$ & $0.795\to0.834$ & $0.869\to0.889$ \\
		GSM8K & Qwen3-14B & $0.890\to0.953$ & $0.972\to0.980$ & $0.773\to0.852$ & $0.882\to0.937$ \\
		\bottomrule
	\end{tabular}
\end{table}
 
\subsection{Separating Weighting, Abstention and Calibration}
\label{app:cr-baselines}
 
Table~\ref{tab:ac-ladder-full} extends Table~\ref{tab:ac-ladder} to $\alpha \in \{0.10, 0.05, 0.02\}$, adds the Qwen3-14B rows with corrected extraction, and reports all metrics for every rule. At $\alpha = 0.10$, the unfiltered vote of the most accurate configurations nearly meets the target, and all CRC rules answer almost every question. At $\alpha = 0.02$, MATH-Hard and MATH $\times$ Qwen3-14B abstain on every question in some splits (Appendix~\ref{app:cr-floor}).
 
Table~\ref{tab:cr-paired} gives the paired differences behind the attribution in Section~\ref{sec:exp-q4}. Weighting without abstention helps on GSM8K $\times$ R1-Qwen-1.5B, is neutral on the two harder main-grid configurations, and is slightly worse than majority vote on Qwen3-14B. Under CRC, weighting leaves $A_c$ essentially unchanged and changes yield: higher on the main-grid configurations and on MATH $\times$ Qwen3-14B, slightly lower on GSM8K $\times$ Qwen3-14B. Replacing the naive threshold by CRC lowers both the realized error and the yield, with $A_c$ essentially unchanged.
 
Table~\ref{tab:cr-errctl} summarizes error control. The mean error of the naive rule exceeds $\alpha$ in most entries, that of Weighted + CRC in a few and by a small margin, and that of MV + CRC in none. A fixed threshold $\lambda = 0.5$ (Table~\ref{tab:cr-fixed-lambda}) does not adapt to $\alpha$ and exceeds $\alpha = 0.05$ on every split of every configuration except GSM8K $\times$ Qwen3-14B$^\dagger$.
 
Because the vote fraction takes only about $m$ distinct values, MV + CRC often realizes an error below $\alpha$. Table~\ref{tab:cr-matched} therefore compares the two votes at matched realized error, with thresholds chosen on test labels. The weighted vote answers more questions on every main-grid configuration, at every error level and on every split. On Qwen3-14B the advantage remains on MATH at $\varepsilon = 0.05$ and reverses on GSM8K.
 
\begin{table}[p]
\centering
\scriptsize
\setlength{\tabcolsep}{2.5pt}
\caption{Full comparison (extends Table~\ref{tab:ac-ladder}). Selective accuracy $A_c$, yield $Y$ and realized confident-error rate (Err.) on test, mean over 20 shared splits (Reward-L8B). Without abstention, Err.\ $= 1 - A_c$ and $Y = 1$. $^{\dagger}$: Qwen3-14B with the \texttt{</think>} tag stripped before answer extraction (Appendix~\ref{app:cr-extraction}). $^{\ast}$: CRC abstains on every question in some splits; $A_c$ averages the remaining splits.}
\label{tab:ac-ladder-full}
\begin{tabular}{ll cc ccc ccc ccc}
\toprule
 & & \multicolumn{2}{c}{No abstention} & \multicolumn{3}{c}{MV + CRC} & \multicolumn{3}{c}{Weighted + naive} & \multicolumn{3}{c}{Weighted + CRC (ours)} \\
\cmidrule(lr){3-4}\cmidrule(lr){5-7}\cmidrule(lr){8-10}\cmidrule(lr){11-13}
Dataset & Model & MV & Weighted & Err. & $A_c$ & $Y$ & Err. & $A_c$ & $Y$ & Err. & $A_c$ & $Y$ \\
\midrule
\multicolumn{13}{l}{\emph{$\alpha = 0.10$}} \\
\midrule
MATH-Hard & Qwen3-8B & 0.536 & 0.534 & 0.079 & 0.844 & 0.50 & 0.098 & 0.827 & 0.56 & 0.094 & 0.832 & 0.55 \\
MATH & Qwen3-1.7B & 0.713 & 0.720 & 0.092 & 0.873 & 0.71 & 0.101 & 0.868 & 0.76 & 0.098 & 0.872 & 0.75 \\
GSM8K & R1-Qwen-1.5B & 0.823 & 0.874 & 0.088 & 0.896 & 0.84 & 0.099 & 0.897 & 0.96 & 0.094 & 0.901 & 0.95 \\
MATH & Qwen3-14B & 0.880 & 0.868 & 0.096 & 0.899 & 0.94 & 0.110 & 0.887 & 0.97 & 0.107 & 0.890 & 0.96 \\
GSM8K & Qwen3-14B & 0.892 & 0.883 & 0.083 & 0.914 & 0.96 & 0.093 & 0.904 & 0.96 & 0.089 & 0.907 & 0.96 \\
MATH & Qwen3-14B$^{\dagger}$ & 0.897 & 0.888 & 0.094 & 0.905 & 0.98 & 0.104 & 0.895 & 0.99 & 0.100 & 0.899 & 0.98 \\
GSM8K & Qwen3-14B$^{\dagger}$ & 0.954 & 0.938 & 0.046 & 0.954 & 1.00 & 0.062 & 0.938 & 1.00 & 0.062 & 0.938 & 1.00 \\
\midrule
\multicolumn{13}{l}{\emph{$\alpha = 0.05$}} \\
\midrule
MATH-Hard & Qwen3-8B & 0.536 & 0.534 & 0.040 & 0.901 & 0.39 & 0.049 & 0.897 & 0.47 & 0.046 & 0.902 & 0.45 \\
MATH & Qwen3-1.7B & 0.713 & 0.720 & 0.045 & 0.924 & 0.58 & 0.053 & 0.918 & 0.65 & 0.049 & 0.923 & 0.63 \\
GSM8K & R1-Qwen-1.5B & 0.823 & 0.874 & 0.039 & 0.942 & 0.65 & 0.048 & 0.943 & 0.83 & 0.042 & 0.949 & 0.80 \\
MATH & Qwen3-14B & 0.880 & 0.868 & 0.049 & 0.934 & 0.73 & 0.056 & 0.932 & 0.81 & 0.050 & 0.935 & 0.77 \\
GSM8K & Qwen3-14B & 0.892 & 0.883 & 0.040 & 0.955 & 0.87 & 0.055 & 0.939 & 0.88 & 0.048 & 0.946 & 0.86 \\
MATH & Qwen3-14B$^{\dagger}$ & 0.897 & 0.888 & 0.048 & 0.940 & 0.79 & 0.056 & 0.936 & 0.86 & 0.050 & 0.940 & 0.82 \\
GSM8K & Qwen3-14B$^{\dagger}$ & 0.954 & 0.938 & 0.039 & 0.961 & 0.98 & 0.050 & 0.949 & 0.98 & 0.045 & 0.953 & 0.96 \\
\midrule
\multicolumn{13}{l}{\emph{$\alpha = 0.02$}} \\
\midrule
MATH-Hard & Qwen3-8B & 0.536 & 0.534 & 0.003 & 0.914$^{\ast}$ & 0.03 & 0.012 & 0.919$^{\ast}$ & 0.15 & 0.003 & 0.914$^{\ast}$ & 0.04 \\
MATH & Qwen3-1.7B & 0.713 & 0.720 & 0.018 & 0.959 & 0.42 & 0.023 & 0.955 & 0.50 & 0.020 & 0.959 & 0.47 \\
GSM8K & R1-Qwen-1.5B & 0.823 & 0.874 & 0.015 & 0.968 & 0.44 & 0.022 & 0.967 & 0.65 & 0.019 & 0.970 & 0.59 \\
MATH & Qwen3-14B & 0.880 & 0.868 & 0.003 & 0.947$^{\ast}$ & 0.06 & 0.011 & 0.942$^{\ast}$ & 0.18 & 0.003 & 0.947$^{\ast}$ & 0.06 \\
GSM8K & Qwen3-14B & 0.892 & 0.883 & 0.019 & 0.976 & 0.75 & 0.023 & 0.970 & 0.74 & 0.019 & 0.974 & 0.71 \\
MATH & Qwen3-14B$^{\dagger}$ & 0.897 & 0.888 & 0.003 & 0.952$^{\ast}$ & 0.06 & 0.011 & 0.946$^{\ast}$ & 0.20 & 0.003 & 0.951$^{\ast}$ & 0.06 \\
GSM8K & Qwen3-14B$^{\dagger}$ & 0.954 & 0.938 & 0.018 & 0.980 & 0.84 & 0.024 & 0.973 & 0.85 & 0.019 & 0.977 & 0.81 \\
\bottomrule
\end{tabular}
\end{table}
 
\begin{table}[h]
	\centering
	\caption{Paired differences on the shared splits (Reward-L8B; mean $\pm$ std over 20 splits). \emph{No abstention}: Weighted $-$ MV accuracy. \emph{Middle}: Weighted + CRC $-$ MV + CRC. \emph{Right}: Weighted + CRC $-$ Weighted + naive.}
	\label{tab:cr-paired}
	\scriptsize
	\setlength{\tabcolsep}{2pt}
	\begin{tabular}{ll c c ccc ccc}
		\toprule
		 & & & No abstention & \multicolumn{3}{c}{Weighted+CRC $-$ MV+CRC} & \multicolumn{3}{c}{Weighted: CRC $-$ naive} \\
		\cmidrule(lr){4-4}\cmidrule(lr){5-7}\cmidrule(lr){8-10}
		Dataset & Model & $\alpha$ & $\Delta A_c$ & $\Delta A_c$ & $\Delta Y$ & $\Delta$Err. & $\Delta A_c$ & $\Delta Y$ & $\Delta$Err. \\
		\midrule
		\multirow{2}{*}{MATH-Hard} & \multirow{2}{*}{Qwen3-8B} & 0.10 & $-0.001{\scriptstyle\pm}.004$ & $-0.012{\scriptstyle\pm}.035$ & $+0.055{\scriptstyle\pm}.043$ & $+0.015$ & $+0.005$ & $-0.006$ & $-0.004$ \\
		 &  & 0.05 & $-0.001{\scriptstyle\pm}.004$ & $+0.001{\scriptstyle\pm}.019$ & $+0.060{\scriptstyle\pm}.038$ & $+0.006$ & $+0.004$ & $-0.014$ & $-0.003$ \\
		\midrule
		\multirow{2}{*}{MATH} & \multirow{2}{*}{Qwen3-1.7B} & 0.10 & $+0.006{\scriptstyle\pm}.003$ & $-0.001{\scriptstyle\pm}.027$ & $+0.043{\scriptstyle\pm}.047$ & $+0.006$ & $+0.004$ & $-0.008$ & $-0.004$ \\
		 &  & 0.05 & $+0.006{\scriptstyle\pm}.003$ & $-0.002{\scriptstyle\pm}.016$ & $+0.054{\scriptstyle\pm}.045$ & $+0.005$ & $+0.005$ & $-0.012$ & $-0.004$ \\
		\midrule
		\multirow{2}{*}{GSM8K} & \multirow{2}{*}{R1-Qwen-1.5B} & 0.10 & $+0.051{\scriptstyle\pm}.004$ & $+0.005{\scriptstyle\pm}.015$ & $+0.111{\scriptstyle\pm}.039$ & $+0.006$ & $+0.004$ & $-0.009$ & $-0.004$ \\
		 &  & 0.05 & $+0.051{\scriptstyle\pm}.004$ & $+0.008{\scriptstyle\pm}.012$ & $+0.154{\scriptstyle\pm}.067$ & $+0.003$ & $+0.006$ & $-0.026$ & $-0.006$ \\
		\midrule
		\multirow{2}{*}{MATH} & \multirow{2}{*}{Qwen3-14B} & 0.10 & $-0.012{\scriptstyle\pm}.003$ & $-0.009{\scriptstyle\pm}.012$ & $+0.016{\scriptstyle\pm}.031$ & $+0.010$ & $+0.003$ & $-0.006$ & $-0.003$ \\
		 &  & 0.05 & $-0.012{\scriptstyle\pm}.003$ & $+0.001{\scriptstyle\pm}.006$ & $+0.036{\scriptstyle\pm}.050$ & $+0.002$ & $+0.003$ & $-0.042$ & $-0.005$ \\
		\midrule
		\multirow{2}{*}{GSM8K} & \multirow{2}{*}{Qwen3-14B} & 0.10 & $-0.009{\scriptstyle\pm}.003$ & $-0.007{\scriptstyle\pm}.015$ & $-0.005{\scriptstyle\pm}.022$ & $+0.006$ & $+0.003$ & $-0.006$ & $-0.004$ \\
		 &  & 0.05 & $-0.009{\scriptstyle\pm}.003$ & $-0.009{\scriptstyle\pm}.015$ & $-0.010{\scriptstyle\pm}.042$ & $+0.008$ & $+0.007$ & $-0.020$ & $-0.007$ \\
		\midrule
		\multirow{2}{*}{MATH} & \multirow{2}{*}{Qwen3-14B$^{\dagger}$} & 0.10 & $-0.009{\scriptstyle\pm}.002$ & $-0.006{\scriptstyle\pm}.007$ & $+0.003{\scriptstyle\pm}.016$ & $+0.006$ & $+0.004$ & $-0.006$ & $-0.004$ \\
		 &  & 0.05 & $-0.009{\scriptstyle\pm}.002$ & $0.000{\scriptstyle\pm}.005$ & $+0.033{\scriptstyle\pm}.046$ & $+0.002$ & $+0.004$ & $-0.040$ & $-0.005$ \\
		\midrule
		\multirow{2}{*}{GSM8K} & \multirow{2}{*}{Qwen3-14B$^{\dagger}$} & 0.10 & $-0.016{\scriptstyle\pm}.002$ & $-0.015{\scriptstyle\pm}.005$ & $-0.002{\scriptstyle\pm}.007$ & $+0.015$ & $0.000$ & $0.000$ & $0.000$ \\
		 &  & 0.05 & $-0.016{\scriptstyle\pm}.002$ & $-0.007{\scriptstyle\pm}.009$ & $-0.021{\scriptstyle\pm}.023$ & $+0.006$ & $+0.004$ & $-0.013$ & $-0.005$ \\
		\bottomrule
	\end{tabular}
\end{table}
 
\begin{table}[h]
	\centering
	\caption{Mean realized confident-error rate against the target over the 28 (configuration, $\alpha$) entries of Table~\ref{tab:ac-ladder-full} at $\alpha\in\{0.10,0.05\}$. Max excess: largest mean error minus $\alpha$.}
	\label{tab:cr-errctl}
	\footnotesize
	\setlength{\tabcolsep}{4pt}
	\begin{tabular}{lccc}
		\toprule
		Rule & Entries & Mean Err.\ above $\alpha$ & Max excess \\
		\midrule
		Weighted + CRC (ours) & 28 & 4 & $+0.007$ \\
		MV + CRC & 28 & 0 & $-0.001$ \\
		Weighted + naive threshold & 28 & 18 & $+0.010$ \\
		\bottomrule
	\end{tabular}
\end{table}
 
\begin{table}[h]
	\centering
	\caption{Weighted vote with a fixed threshold $\lambda = 0.5$ (Reward-L8B, 20 splits). Last two columns: share of splits whose realized error exceeds each target.}
	\label{tab:cr-fixed-lambda}
	\small
	\begin{tabular}{llccccc}
		\toprule
		 & & \multicolumn{3}{c}{Weighted + fixed $\lambda=0.5$} & \multicolumn{2}{c}{Splits with Err.\ $>$} \\
		\cmidrule(lr){3-5}\cmidrule(lr){6-7}
		Dataset & Model & Err. & $A_c$ & $Y$ & $0.05$ & $0.10$ \\
		\midrule
		MATH-Hard & Qwen3-8B & 0.096 & 0.828 & 0.56 & 100\% & 15\% \\
		MATH & Qwen3-1.7B & 0.078 & 0.890 & 0.71 & 100\% & 0\% \\
		GSM8K & R1-Qwen-1.5B & 0.073 & 0.920 & 0.91 & 100\% & 0\% \\
		MATH & Qwen3-14B & 0.090 & 0.904 & 0.93 & 100\% & 0\% \\
		GSM8K & Qwen3-14B & 0.063 & 0.930 & 0.91 & 100\% & 0\% \\
		MATH & Qwen3-14B$^{\dagger}$ & 0.079 & 0.917 & 0.95 & 100\% & 0\% \\
		GSM8K & Qwen3-14B$^{\dagger}$ & 0.043 & 0.955 & 0.96 & 0\% & 0\% \\
		\bottomrule
	\end{tabular}
\end{table}
 
\begin{table}[h]
	\centering
	\caption{Yield of MV and of the weighted vote (Reward-L8B) at matched realized error $\varepsilon$, mean over 20 splits. For each split and rule, the threshold is the smallest $\lambda$ whose test error is at most $\varepsilon$, so this comparison is retrospective. In parentheses: share of splits on which the weighted vote answers more. \emph{infeasible}: only full abstention meets $\varepsilon$.}
	\label{tab:cr-matched}
	\scriptsize
	\setlength{\tabcolsep}{3pt}
	\begin{tabular}{llccc}
		\toprule
		 & & \multicolumn{3}{c}{$Y_\mathrm{MV}\to Y_\mathrm{W}$ (share of splits with $Y_\mathrm{W} > Y_\mathrm{MV}$)} \\
		\cmidrule(lr){3-5}
		Dataset & Model & $\varepsilon = 0.10$ & $\varepsilon = 0.05$ & $\varepsilon = 0.02$ \\
		\midrule
		MATH-Hard & Qwen3-8B & $0.526\to0.565$ (100\%) & $0.422\to0.475$ (100\%) & $0.237\to0.284$ (100\%) \\
		MATH & Qwen3-1.7B & $0.712\to0.762$ (100\%) & $0.575\to0.644$ (100\%) & $0.437\to0.494$ (100\%) \\
		GSM8K & R1-Qwen-1.5B & $0.851\to0.962$ (100\%) & $0.703\to0.842$ (100\%) & $0.494\to0.651$ (100\%) \\
		MATH & Qwen3-14B & $0.954\to0.954$ (30\%) & $0.746\to0.793$ (100\%) & infeasible \\
		GSM8K & Qwen3-14B & $0.986\to0.976$ (5\%) & $0.897\to0.873$ (5\%) & $0.765\to0.749$ (25\%) \\
		MATH & Qwen3-14B$^{\dagger}$ & $0.994\to0.985$ (5\%) & $0.803\to0.839$ (100\%) & infeasible \\
		GSM8K & Qwen3-14B$^{\dagger}$ & $1.000\to1.000$ (0\%) & $1.000\to0.980$ (0\%) & $0.903\to0.861$ (0\%) \\
		\bottomrule
	\end{tabular}
\end{table}
 
\subsection{Accuracy Gain and Base Accuracy}
\label{app:cr-gain}
 
Table~\ref{tab:cr-gain} places the Qwen3-14B rows next to the main grid. The observed gain generally decreases as $p_v$ increases, together with the separability estimated at the calibrated threshold, and the calibration-only predictor stays close to the observed gain on every row. On GSM8K $\times$ Gemma-4-E4B and GSM8K $\times$ Qwen3-14B$^\dagger$ at $\alpha = 0.10$, the unfiltered vote nearly meets the target and abstention changes little.
 
\begin{table}[h]
	\centering
	\caption{Gain $A_c(\hat\lambda) - p_v$ of Weighted + CRC against base accuracy (Reward-L8B, 20 splits). Observed: mean $\pm$ std on test. Predicted: the calibration-only predictor of Corollary~\ref{cor:predictor} at $\hat\lambda$. $\hat\Delta(\hat\lambda)$ is estimated on calibration.}
	\label{tab:cr-gain}
	\scriptsize
	\setlength{\tabcolsep}{2pt}
	\begin{tabular}{llc cccc cccc}
		\toprule
		 & & & \multicolumn{4}{c}{$\alpha = 0.10$} & \multicolumn{4}{c}{$\alpha = 0.05$} \\
		\cmidrule(lr){4-7}\cmidrule(lr){8-11}
		Dataset & Model & $p_v$ & $\hat\Delta(\hat\lambda)$ & Observed & Predicted & $Y$ & $\hat\Delta(\hat\lambda)$ & Observed & Predicted & $Y$ \\
		\midrule
		MATH-Hard & Qwen3-8B & 0.534 & 0.641 & $+0.298{\scriptstyle\pm}.034$ & $+0.283$ & 0.553 & 0.650 & $+0.367{\scriptstyle\pm}.024$ & $+0.354$ & 0.454 \\
		MATH & Qwen3-1.7B & 0.720 & 0.590 & $+0.152{\scriptstyle\pm}.019$ & $+0.166$ & 0.755 & 0.676 & $+0.203{\scriptstyle\pm}.015$ & $+0.222$ & 0.633 \\
		GSM8K & R1-Qwen-1.5B & 0.874 & 0.242 & $+0.027{\scriptstyle\pm}.012$ & $+0.031$ & 0.951 & 0.520 & $+0.075{\scriptstyle\pm}.015$ & $+0.075$ & 0.801 \\
		GSM8K & Gemma-4-E4B & 0.907 & 0.059 & $+0.005{\scriptstyle\pm}.009$ & $+0.006$ & 0.989 & 0.393 & $+0.044{\scriptstyle\pm}.015$ & $+0.041$ & 0.856 \\
		MATH & Qwen3-14B & 0.868 & 0.202 & $+0.022{\scriptstyle\pm}.017$ & $+0.025$ & 0.960 & 0.447 & $+0.068{\scriptstyle\pm}.010$ & $+0.064$ & 0.767 \\
		GSM8K & Qwen3-14B & 0.883 & 0.221 & $+0.024{\scriptstyle\pm}.015$ & $+0.028$ & 0.958 & 0.532 & $+0.063{\scriptstyle\pm}.017$ & $+0.072$ & 0.858 \\
		MATH & Qwen3-14B$^{\dagger}$ & 0.888 & 0.097 & $+0.011{\scriptstyle\pm}.013$ & $+0.010$ & 0.983 & 0.407 & $+0.052{\scriptstyle\pm}.010$ & $+0.047$ & 0.820 \\
		GSM8K & Qwen3-14B$^{\dagger}$ & 0.938 & 0.012 & $+0.001{\scriptstyle\pm}.003$ & $+0.001$ & 0.998 & 0.273 & $+0.016{\scriptstyle\pm}.010$ & $+0.021$ & 0.962 \\
		\bottomrule
	\end{tabular}
\end{table}
 
\subsection{Qwen3-14B at Stricter Targets}
\label{app:cr-floor}
 
Table~\ref{tab:cr-grid} reports Qwen3-14B over the pre-fixed targets. On GSM8K, yield decreases smoothly down to $\alpha = 0.02$. On MATH, CRC abstains on every question in half of the splits at $\alpha = 0.03$ and in most splits at $\alpha = 0.02$. The reason is that 35 of the 1300 MATH questions are unanimous and wrong: all 16 paths agree on an answer the grader rejects, so $\nu = 1$ under any weighting and no threshold removes them. A calibration set of 200 contains about five such questions, more than the error budget allows at the strictest targets. At least 24 of the 35 are correct answers rejected by exact-match grading, such as equivalent forms ($-1/4$ against $-0.25$) or the correct option letter instead of its value. On GSM8K such questions are rare and the strict targets remain attainable. At high base accuracy, the attainable confident-error rate is therefore limited by answer grading rather than by separability.
 
\begin{table}[h]
	\centering
	\caption{Qwen3-14B with Weighted + CRC across the pre-fixed targets (Reward-L8B; mean $\pm$ std over 20 splits). Viol.: share of splits with error above $\alpha$. $\hat\lambda{=}0$: share of splits answering every question. Abstain-all: share of splits abstaining on every question; $A_c$ averages the remaining splits.}
	\label{tab:cr-grid}
	\footnotesize
	\setlength{\tabcolsep}{4pt}
	\begin{tabular}{cccccccc}
		\toprule
		$\alpha$ & $\hat\lambda$ & Err. & Viol. & $A_c$ & $Y$ & $\hat\lambda{=}0$ & Abstain-all \\
		\midrule
		\multicolumn{8}{l}{\emph{MATH, Qwen3-14B}} \\
		0.10 & $0.365\pm0.179$ & 0.107 & 0.70 & $0.890\pm0.020$ & $0.960\pm0.034$ & 15\% & 0\% \\
		0.05 & $0.787\pm0.131$ & 0.050 & 0.45 & $0.935\pm0.010$ & $0.767\pm0.097$ & 0\% & 0\% \\
		0.03 & $0.959\pm0.070$ & 0.020 & 0.45 & $0.939\pm0.005$ & $0.330\pm0.343$ & 0\% & 50\% \\
		0.02 & $0.999\pm0.002$ & 0.003 & 0.10 & $0.947\pm0.000$ & $0.056\pm0.173$ & 0\% & 90\% \\
		\midrule
		\multicolumn{8}{l}{\emph{MATH, Qwen3-14B$^{\dagger}$}} \\
		0.10 & $0.259\pm0.209$ & 0.100 & 0.45 & $0.899\pm0.015$ & $0.983\pm0.021$ & 35\% & 0\% \\
		0.05 & $0.780\pm0.152$ & 0.050 & 0.40 & $0.940\pm0.011$ & $0.820\pm0.095$ & 0\% & 0\% \\
		0.03 & $0.958\pm0.072$ & 0.020 & 0.45 & $0.944\pm0.006$ & $0.359\pm0.372$ & 0\% & 50\% \\
		0.02 & $0.999\pm0.002$ & 0.003 & 0.10 & $0.951\pm0.000$ & $0.061\pm0.187$ & 0\% & 90\% \\
		\midrule
		\multicolumn{8}{l}{\emph{GSM8K, Qwen3-14B}} \\
		0.10 & $0.368\pm0.139$ & 0.089 & 0.30 & $0.907\pm0.019$ & $0.958\pm0.028$ & 10\% & 0\% \\
		0.05 & $0.558\pm0.065$ & 0.048 & 0.45 & $0.946\pm0.020$ & $0.858\pm0.050$ & 0\% & 0\% \\
		0.03 & $0.668\pm0.074$ & 0.026 & 0.30 & $0.967\pm0.014$ & $0.763\pm0.065$ & 0\% & 0\% \\
		0.02 & $0.725\pm0.078$ & 0.019 & 0.35 & $0.974\pm0.011$ & $0.711\pm0.070$ & 0\% & 0\% \\
		\midrule
		\multicolumn{8}{l}{\emph{GSM8K, Qwen3-14B$^{\dagger}$}} \\
		0.10 & $0.023\pm0.103$ & 0.062 & 0.00 & $0.938\pm0.006$ & $0.998\pm0.007$ & 95\% & 0\% \\
		0.05 & $0.414\pm0.196$ & 0.045 & 0.25 & $0.953\pm0.014$ & $0.962\pm0.031$ & 15\% & 0\% \\
		0.03 & $0.609\pm0.180$ & 0.027 & 0.30 & $0.970\pm0.016$ & $0.877\pm0.072$ & 5\% & 0\% \\
		0.02 & $0.721\pm0.121$ & 0.019 & 0.30 & $0.977\pm0.011$ & $0.811\pm0.093$ & 0\% & 0\% \\
		\bottomrule
	\end{tabular}
\end{table}
 
\subsection{Distribution Shift}
\label{app:cr-shift}
 
Table~\ref{tab:cr-shift} calibrates CRC on one dataset and evaluates it on the other. Calibrating on the easier dataset and testing on the harder one over-answers, with realized error several times the target for R1-Qwen-1.5B and Gemma-4-E4B; the reverse direction is over-conservative. Writing the error ratio as the product of a difficulty factor and a score-behaviour factor, neither factor dominates: difficulty is the larger one in the easy-to-hard rows that exceed $\alpha$, score behaviour in several hard-to-easy rows. Qwen3-1.7B, whose base accuracy is similar on both datasets, stays close to the target in both directions. Recalibrating on 25 labeled target prompts restores error control in every case, at reduced yield (Table~\ref{tab:cr-recal}).
 
\begin{table}[h]
	\centering
	\caption{CRC calibrated on 200 prompts of the source dataset and evaluated on the full target dataset (Reward-L8B, 20 splits). The target/source error ratio factors into a difficulty factor $(1-p_v^\mathrm{tgt})/(1-p_v^\mathrm{src})$ and a score-behaviour factor $S^\mathrm{tgt}_\mathrm{err}(\hat\lambda)/S^\mathrm{src}_\mathrm{err}(\hat\lambda)$. ---: CRC abstains on every target question.}
	\label{tab:cr-shift}
	\scriptsize
	\setlength{\tabcolsep}{2.5pt}
	\begin{tabular}{llcccccccc}
		\toprule
		 & & & & & & & \multicolumn{3}{c}{Err.\ ratio target/source} \\
		\cmidrule(lr){8-10}
		Model & Calib.$\to$Test & $p_v$ & Err. & Err./$\alpha$ & $A_c$ & $Y$ & Difficulty & Score & Product \\
		\midrule
		\multicolumn{10}{l}{\emph{$\alpha = 0.05$}} \\
		Qwen3-1.7B & GSM8K$\to$MATH & $0.732\to0.718$ & $0.043{\scriptstyle\pm}.008$ & 0.85$\times$ & 0.931 & 0.612 & 1.07$\times$ & 0.90$\times$ & 0.95$\times$ \\
		Qwen3-1.7B & MATH$\to$GSM8K & $0.706\to0.740$ & $0.059{\scriptstyle\pm}.017$ & 1.18$\times$ & 0.912 & 0.660 & 0.89$\times$ & 1.46$\times$ & 1.31$\times$ \\
		R1-Qwen-1.5B & GSM8K$\to$MATH & $0.868\to0.623$ & $0.150{\scriptstyle\pm}.028$ & 3.00$\times$ & 0.774 & 0.656 & 2.93$\times$ & 1.14$\times$ & 3.33$\times$ \\
		R1-Qwen-1.5B & MATH$\to$GSM8K & $0.617\to0.873$ & $0.008{\scriptstyle\pm}.002$ & 0.16$\times$ & 0.978 & 0.361 & 0.33$\times$ & 0.54$\times$ & 0.18$\times$ \\
		Gemma-4-E4B & GSM8K$\to$MATH & $0.906\to0.587$ & $0.341{\scriptstyle\pm}.040$ & 6.82$\times$ & 0.620 & 0.895 & 4.58$\times$ & 1.70$\times$ & 7.58$\times$ \\
		Gemma-4-E4B & MATH$\to$GSM8K & $0.587\to0.907$ & $0.000{\scriptstyle\pm}.000$ & 0.00$\times$ & --- & 0.000 & 0.23$\times$ & --- & --- \\
		\midrule
		\multicolumn{10}{l}{\emph{$\alpha = 0.10$}} \\
		Qwen3-1.7B & GSM8K$\to$MATH & $0.732\to0.718$ & $0.073{\scriptstyle\pm}.009$ & 0.73$\times$ & 0.896 & 0.700 & 1.07$\times$ & 0.73$\times$ & 0.77$\times$ \\
		Qwen3-1.7B & MATH$\to$GSM8K & $0.706\to0.740$ & $0.140{\scriptstyle\pm}.026$ & 1.40$\times$ & 0.831 & 0.821 & 0.89$\times$ & 1.65$\times$ & 1.47$\times$ \\
		R1-Qwen-1.5B & GSM8K$\to$MATH & $0.868\to0.623$ & $0.261{\scriptstyle\pm}.046$ & 2.61$\times$ & 0.697 & 0.857 & 2.93$\times$ & 0.94$\times$ & 2.76$\times$ \\
		R1-Qwen-1.5B & MATH$\to$GSM8K & $0.617\to0.873$ & $0.021{\scriptstyle\pm}.004$ & 0.21$\times$ & 0.968 & 0.645 & 0.33$\times$ & 0.66$\times$ & 0.22$\times$ \\
		Gemma-4-E4B & GSM8K$\to$MATH & $0.906\to0.587$ & $0.409{\scriptstyle\pm}.010$ & 4.09$\times$ & 0.589 & 0.994 & 4.58$\times$ & 1.06$\times$ & 4.81$\times$ \\
		Gemma-4-E4B & MATH$\to$GSM8K & $0.587\to0.907$ & $0.001{\scriptstyle\pm}.001$ & 0.01$\times$ & 0.992 & 0.090 & 0.23$\times$ & 0.04$\times$ & 0.01$\times$ \\
		\bottomrule
	\end{tabular}
\end{table}
 
\begin{table}[h]
	\centering
	\caption{Recalibrating on $n_\mathrm{tgt}$ labeled target prompts under shift ($\alpha = 0.10$): realized confident-error rate, with yield in parentheses.}
	\label{tab:cr-recal}
	\footnotesize
	\begin{tabular}{llcccc}
		\toprule
		 & & No target labels & \multicolumn{3}{c}{Recalibrated on $n_\mathrm{tgt}$ target labels} \\
		\cmidrule(lr){4-6}
		Model & Source$\to$Target & Err. & $n_\mathrm{tgt}=25$ & $n_\mathrm{tgt}=50$ & $n_\mathrm{tgt}=100$ \\
		\midrule
		Qwen3-1.7B & GSM8K$\to$MATH & 0.073 & 0.054 (0.62) & 0.089 (0.72) & 0.084 (0.72) \\
		Qwen3-1.7B & MATH$\to$GSM8K & 0.140 & 0.085 (0.68) & 0.108 (0.76) & 0.104 (0.76) \\
		R1-Qwen-1.5B & GSM8K$\to$MATH & 0.261 & 0.084 (0.41) & 0.116 (0.55) & 0.113 (0.55) \\
		R1-Qwen-1.5B & MATH$\to$GSM8K & 0.021 & 0.066 (0.86) & 0.079 (0.90) & 0.087 (0.93) \\
		Gemma-4-E4B & GSM8K$\to$MATH & 0.409 & 0.061 (0.19) & 0.060 (0.19) & 0.070 (0.22) \\
		Gemma-4-E4B & MATH$\to$GSM8K & 0.001 & 0.064 (0.91) & 0.075 (0.95) & 0.084 (0.98) \\
		\bottomrule
	\end{tabular}
\end{table}
 
\subsection{Calibration-Set Size on Hard Configurations}
\label{app:cr-ncal}
 
Tables~\ref{tab:cr-ncal} and~\ref{tab:cr-ncal-floor} extend Appendix~\ref{app:abla-ncal} to the hardest configurations. CRC is feasible only when $n_\mathrm{cal} \ge 1/\alpha - 1$; below this size it abstains on every question. Above it, yield on most configurations reaches nearly its $n_\mathrm{cal} = 200$ value by about 20 calibration prompts, while on MATH-Hard at $\alpha = 0.05$ it keeps increasing up to 200. MATH $\times$ Gemma-4-E4B at $\alpha = 0.05$ abstains at every size, because its calibration sets contain more unanimous wrong answers than the error budget allows.
 
\begin{table}[h]
	\centering
	\caption{Small calibration sets on the hardest configurations (Reward-L8B, 20 splits). $k=\lfloor (n_\mathrm{cal}+1)\alpha\rfloor$; CRC is feasible only when $k\ge1$. Degenerate: share of splits abstaining on every question.}
	\label{tab:cr-ncal}
	\scriptsize
	\setlength{\tabcolsep}{4pt}
	\begin{tabular}{llccccccccc}
		\toprule
		Dataset & Model & $\alpha$ & $n_\mathrm{cal}$ & $k$ & Err. & Viol. & $A_c$ & $Y$ & std($\hat\lambda$) & Degenerate \\
		\midrule
		MATH-Hard & Qwen3-8B & 0.05 & 10 & 0 & 0.000 & 0.00 & --- & 0.000 & 0.000 & 1.00 \\
		MATH-Hard & Qwen3-8B & 0.05 & 15 & 0 & 0.000 & 0.00 & --- & 0.000 & 0.000 & 1.00 \\
		MATH-Hard & Qwen3-8B & 0.05 & 20 & 1 & 0.027 & 0.20 & 0.889 & 0.222 & 0.191 & 0.50 \\
		MATH-Hard & Qwen3-8B & 0.05 & 25 & 1 & 0.027 & 0.20 & 0.888 & 0.223 & 0.191 & 0.50 \\
		MATH-Hard & Qwen3-8B & 0.05 & 50 & 2 & 0.038 & 0.40 & 0.892 & 0.332 & 0.192 & 0.25 \\
		MATH-Hard & Qwen3-8B & 0.05 & 100 & 5 & 0.042 & 0.40 & 0.896 & 0.393 & 0.145 & 0.15 \\
		MATH-Hard & Qwen3-8B & 0.05 & 200 & 10 & 0.046 & 0.45 & 0.902 & 0.454 & 0.082 & 0.00 \\
		MATH-Hard & Qwen3-8B & 0.10 & 10 & 1 & 0.065 & 0.25 & 0.845 & 0.380 & 0.243 & 0.25 \\
		MATH-Hard & Qwen3-8B & 0.10 & 20 & 2 & 0.073 & 0.25 & 0.863 & 0.501 & 0.146 & 0.00 \\
		MATH-Hard & Qwen3-8B & 0.10 & 50 & 5 & 0.094 & 0.45 & 0.835 & 0.545 & 0.101 & 0.00 \\
		MATH-Hard & Qwen3-8B & 0.10 & 200 & 20 & 0.094 & 0.40 & 0.832 & 0.553 & 0.046 & 0.00 \\
		\midrule
		MATH & R1-Qwen-1.5B & 0.05 & 15 & 0 & 0.000 & 0.00 & --- & 0.000 & 0.000 & 1.00 \\
		MATH & R1-Qwen-1.5B & 0.05 & 20 & 1 & 0.065 & 0.45 & 0.815 & 0.321 & 0.149 & 0.10 \\
		MATH & R1-Qwen-1.5B & 0.05 & 200 & 10 & 0.047 & 0.30 & 0.832 & 0.276 & 0.019 & 0.00 \\
		MATH & R1-Qwen-1.5B & 0.10 & 10 & 1 & 0.112 & 0.45 & 0.790 & 0.479 & 0.225 & 0.05 \\
		MATH & R1-Qwen-1.5B & 0.10 & 200 & 20 & 0.098 & 0.50 & 0.808 & 0.506 & 0.045 & 0.00 \\
		\midrule
		MATH & Gemma-4-E4B & 0.05 & 200 & 10 & 0.000 & 0.00 & --- & 0.000 & 0.000 & 1.00 \\
		MATH & Gemma-4-E4B & 0.10 & 20 & 2 & 0.086 & 0.45 & 0.685 & 0.271 & 0.043 & 0.35 \\
		MATH & Gemma-4-E4B & 0.10 & 200 & 20 & 0.081 & 0.40 & 0.685 & 0.256 & 0.015 & 0.25 \\
		\bottomrule
	\end{tabular}
\end{table}
 
\begin{table}[h]
	\centering
	\caption{Calibration-set size needed on the hardest configurations. Cliff: largest swept $n_\mathrm{cal}$ with $k=0$ (---: none). Practical floor: smallest swept $n_\mathrm{cal}$ whose yield reaches 90\% of its value at $n_\mathrm{cal}=200$.}
	\label{tab:cr-ncal-floor}
	\footnotesize
	\begin{tabular}{llcccc}
		\toprule
		Dataset & Model & $\alpha$ & Cliff & Practical floor & Usable at $n_\mathrm{cal}=200$ \\
		\midrule
		MATH & Gemma-4-E4B & 0.05 & 15 & --- & no (abstains on every split) \\
		MATH & Gemma-4-E4B & 0.10 & --- & 20 & yes ($Y=0.256$) \\
		MATH & R1-Qwen-1.5B & 0.05 & 15 & 20 & yes ($Y=0.276$) \\
		MATH & R1-Qwen-1.5B & 0.10 & --- & 10 & yes ($Y=0.506$) \\
		MATH-Hard & Qwen3-8B & 0.05 & 15 & 200 & yes ($Y=0.454$) \\
		MATH-Hard & Qwen3-8B & 0.10 & --- & 20 & yes ($Y=0.553$) \\
		\bottomrule
	\end{tabular}
\end{table}
 
\subsection{Confidence Granularity}
\label{app:cr-granularity}
 
Under majority vote, $\nu$ takes at most $m$ distinct values; under score weighting it takes hundreds even at $m = 4$ (Table~\ref{tab:cr-granularity}). On the coarse grid, MV + CRC often stops below the error budget. At $m = 4$, the weighted rule answers far more questions on GSM8K $\times$ R1-Qwen-1.5B and somewhat more on MATH-Hard, with a realized error closer to the target (Table~\ref{tab:cr-granularity-m4}). The across-split variability of $\hat\lambda$ and $\hat\Delta$ is comparable under the two rules (Table~\ref{tab:cr-granularity-std}).
 
\begin{table}[h]
	\centering
	\caption{Number of distinct values of $\nu$ over the 1300 questions of a configuration, by $m$.}
	\label{tab:cr-granularity}
	\small
	\begin{tabular}{lllccccc}
		\toprule
		Dataset & Model & Vote & $m=4$ & $m=8$ & $m=12$ & $m=16$ & $m=20$ \\
		\midrule
		GSM8K & R1-Qwen-1.5B & MV & 4 & 8 & 12 & 15 & 19 \\
		GSM8K & R1-Qwen-1.5B & Weighted & 881 & 1060 & 1147 & 1191 & 1212 \\
		MATH-Hard & Qwen3-8B & MV & 4 & 8 & 12 & 16 & 20 \\
		MATH-Hard & Qwen3-8B & Weighted & 848 & 937 & 971 & 991 & 1003 \\
		\bottomrule
	\end{tabular}
\end{table}
 
\begin{table}[h]
	\centering
	\caption{Operating point at $m=4$ (Reward-L8B for the weighted vote, 20 splits).}
	\label{tab:cr-granularity-m4}
	\small
	\begin{tabular}{llclccc}
		\toprule
		Dataset & Model & $\alpha$ & Vote & Err. & $A_c$ & $Y$ \\
		\midrule
		GSM8K & R1-Qwen-1.5B & 0.05 & MV & 0.019 & 0.946 & 0.340 \\
		GSM8K & R1-Qwen-1.5B & 0.05 & Weighted & 0.041 & 0.936 & 0.623 \\
		GSM8K & R1-Qwen-1.5B & 0.10 & MV & 0.059 & 0.907 & 0.626 \\
		GSM8K & R1-Qwen-1.5B & 0.10 & Weighted & 0.089 & 0.892 & 0.814 \\
		MATH-Hard & Qwen3-8B & 0.05 & MV & 0.021 & 0.894 & 0.200 \\
		MATH-Hard & Qwen3-8B & 0.05 & Weighted & 0.029 & 0.874 & 0.228 \\
		MATH-Hard & Qwen3-8B & 0.10 & MV & 0.069 & 0.845 & 0.421 \\
		MATH-Hard & Qwen3-8B & 0.10 & Weighted & 0.091 & 0.812 & 0.474 \\
		\bottomrule
	\end{tabular}
\end{table}
 
\begin{table}[h]
	\centering
	\caption{Across-split standard deviation of $\hat\lambda$ and $\hat\Delta(\hat\lambda)$, averaged over $m\in\{4,8,12,16,20\}$ and $\alpha\in\{0.10,0.05\}$.}
	\label{tab:cr-granularity-std}
	\small
	\begin{tabular}{lllcc}
		\toprule
		Dataset & Model & Vote & mean std($\hat\lambda$) & mean std($\hat\Delta$) \\
		\midrule
		GSM8K & R1-Qwen-1.5B & MV & 0.057 & 0.078 \\
		GSM8K & R1-Qwen-1.5B & Weighted & 0.083 & 0.076 \\
		MATH-Hard & Qwen3-8B & MV & 0.087 & 0.087 \\
		MATH-Hard & Qwen3-8B & Weighted & 0.060 & 0.083 \\
		\bottomrule
	\end{tabular}
\end{table}
 
\subsection{Validation of the Yield Identity}
\label{app:cr-yield}
 
Table~\ref{tab:cr-yield} checks the identity $Y(\lambda) = R(\lambda)/(1-p_v) + p_v\,\Delta(\lambda)$ of \Cref{sec:frontier} across samples. With $p_v$ and $\hat\Delta(\hat\lambda)$ estimated on calibration, it predicts test yield with an error comparable to that of the selective-accuracy predictor. Setting $R(\lambda)$ to the nominal $\alpha$ slightly over-predicts yield, because the realized error of CRC is below $\alpha$; using the realized error removes most of this gap. Read operationally, the yield at a fixed target is a score-independent floor $\alpha/(1-p_v)$ plus $p_v$ units of yield per unit of separability.
 
\begin{table}[h]
	\centering
	\caption{Yield identity across samples (Weighted + CRC, Reward-L8B, 20 splits). $p_v$ and $\hat\Delta(\hat\lambda)$ are estimated on calibration and the predictions compared with test values. MAE$_\mathrm{nom}$, MAE$_\mathrm{real}$: mean absolute error with $R(\lambda)$ set to $\alpha$ or to the realized test error. Last column: $\alpha$ minus the realized test error.}
	\label{tab:cr-yield}
	\scriptsize
	\setlength{\tabcolsep}{3pt}
	\begin{tabular}{llcccccccc}
		\toprule
		 & & & \multicolumn{3}{c}{Yield} & \multicolumn{3}{c}{Selective accuracy} & \\
		\cmidrule(lr){4-6}\cmidrule(lr){7-9}
		Dataset & Model & $\alpha$ & Obs. & MAE$_\mathrm{nom}$ & MAE$_\mathrm{real}$ & Obs. & MAE$_\mathrm{nom}$ & MAE$_\mathrm{real}$ & $\alpha-$Err. \\
		\midrule
		GSM8K & R1-Qwen-1.5B & 0.05 & $0.801\pm0.077$ & 0.041 & 0.023 & 0.949 & 0.019 & 0.018 & $+0.0081$ \\
		GSM8K & R1-Qwen-1.5B & 0.10 & $0.951\pm0.025$ & 0.035 & 0.014 & 0.901 & 0.015 & 0.015 & $+0.0056$ \\
		MATH & Qwen3-1.7B & 0.05 & $0.633\pm0.034$ & 0.026 & 0.022 & 0.923 & 0.015 & 0.015 & $+0.0006$ \\
		MATH & Qwen3-1.7B & 0.10 & $0.755\pm0.037$ & 0.027 & 0.028 & 0.872 & 0.022 & 0.021 & $+0.0025$ \\
		MATH-Hard & Qwen3-8B & 0.05 & $0.454\pm0.040$ & 0.032 & 0.028 & 0.902 & 0.030 & 0.029 & $+0.0044$ \\
		MATH-Hard & Qwen3-8B & 0.10 & $0.553\pm0.036$ & 0.031 & 0.029 & 0.832 & 0.037 & 0.036 & $+0.0062$ \\
		\bottomrule
	\end{tabular}
\end{table}

\end{document}